\documentclass{article}
\usepackage[utf8x]{inputenc}
\usepackage{scalerel}
 \usepackage{xcolor}

\usepackage[margin=1.1in]{geometry}

\usepackage{url}            

\usepackage[T1]{fontenc}    
\usepackage{url}            
\usepackage{booktabs}       
\usepackage{amsfonts}       
\usepackage{nicefrac}       
\usepackage{microtype}      
\usepackage{xcolor}         
\usepackage[shortlabels]{enumitem}
\usepackage{float}
\usepackage{wrapfig}
\usepackage{listings}

\usepackage{amsmath}
\usepackage{amssymb}
\usepackage{amsthm}
\usepackage{mathtools}
\usepackage{graphicx}
\usepackage{subcaption}

\usepackage{microtype}
\usepackage{graphicx}
\usepackage{dirtytalk}
\usepackage{subfiles}
\usepackage[numbers]{natbib}
\usepackage{hyperref}

\definecolor{darkpastelgreen}{rgb}{0.01, 0.75, 0.24}
	\definecolor{cadmiumgreen}{rgb}{0.0, 0.42, 0.24}
\definecolor{armygreen}{rgb}{0.29, 0.33, 0.13}
\hypersetup{colorlinks,linkcolor={blue},citecolor={darkpastelgreen},urlcolor={red}}  
\usepackage{enumitem}
\usepackage{xspace}
\usepackage{forloop}
\usepackage{multirow}
\usepackage{algorithm,algorithmic,refcount}

\title{FedAvg with Fine Tuning: Local Updates Lead to Representation Learning}

\author{Liam Collins\thanks{Department of Electrical and Computer Engineering, 
The University of Texas at Austin, Austin, TX,  USA. \qquad\qquad \qquad Email: \{liamc@utexas.edu, mokhtari@austin.utexas.edu, sanjay.shakkottai@utexas.edu\}.}, \quad Hamed Hassani\thanks{Department of Electrical and Systems Engineering, University of Pennsylvania, Philadelphia, PA, USA. \qquad\qquad\qquad Email: \{hassani@seas.upenn.edu\}.},\quad  Aryan Mokhtari$^*$, \quad Sanjay Shakkottai$^*$}

\begin{document}
\maketitle

\newcommand{\ones}{\mathbf{1}}
\newcommand{\integers}{{\mbox{\bf Z}}}
\newcommand{\symm}{{\mbox{\bf S}}}  

\newcommand{\nullspace}{{\mathcal N}}
\newcommand{\range}{{\mathcal R}}
\newcommand{\Rank}{\mathop{\bf Rank}}
\newcommand{\Tr}{\mathop{\bf Tr}}
\newcommand{\diag}{\mathop{\bf diag}}
\newcommand{\card}{\mathop{\bf card}}
\newcommand{\rank}{\mathop{\bf rank}}
\newcommand{\conv}{\mathop{\bf conv}}
\newcommand{\prox}{\mathbf{prox}}

\newcommand{\ind}{\mathds{1}}
\newcommand{\E}{\mathbb{E}}
\newcommand{\Prob}{\mathbb{P}}
\newcommand{\bigO}{\mathcal{O}}
\newcommand{\B}{\mathcal{B}}
\newcommand{\s}{\mathcal{S}}
\newcommand{\Ev}{\mathcal{E}}
\newcommand{\R}{\mathbb{R}}
\newcommand{\Co}{{\mathop {\bf Co}}} 
\newcommand{\dist}{\mathop{\bf dist{}}}
\newcommand{\argmin}{\mathop{\rm argmin}}
\newcommand{\argmax}{\mathop{\rm argmax}}
\newcommand{\epi}{\mathop{\bf epi}} 
\newcommand{\Vol}{\mathop{\bf vol}}
\newcommand{\dom}{\mathop{\bf dom}} 
\newcommand{\intr}{\mathop{\bf int}}
\newcommand{\sign}{\mathop{\bf sign}}

\newcommand{\cf}{{\it cf.}}
\newcommand{\eg}{{\it e.g.}}
\newcommand{\ie}{{\it i.e.}}
\newcommand{\etc}{{\it etc.}}

\newtheorem{innercustomthm}{Theorem}
\newenvironment{customthm}[1]
  {\renewcommand\theinnercustomthm{#1}\innercustomthm}
  {\endinnercustomthm}

\newtheorem{theorem}{Theorem}
\newtheorem{remark}{Remark}
\newtheorem{definition}{Definition}
\newtheorem{corollary}{Corollary}
\newtheorem{proposition}{Proposition}
\newtheorem{lemma}{Lemma}
\newtheorem{fact}{Fact}
\newtheorem{assumption}{Assumption}
\newtheorem{claim}{Claim}

\makeatletter
\newcommand*\rel@kern[1]{\kern#1\dimexpr\macc@kerna}
\newcommand*\widebar[1]{%
  \begingroup
  \def\mathaccent##1##2{%
    \rel@kern{0.8}%
    \overline{\rel@kern{-0.8}\macc@nucleus\rel@kern{0.2}}%
    \rel@kern{-0.2}%
  }%
  \macc@depth\@ne
  \let\math@bgroup\@empty \let\math@egroup\macc@set@skewchar
  \mathsurround\z@ \frozen@everymath{\mathgroup\macc@group\relax}%
  \macc@set@skewchar\relax
  \let\mathaccentV\macc@nested@a
  \macc@nested@a\relax111{#1}%
  \endgroup
}
\makeatother

\newcommand{\numberthis}{\addtocounter{equation}{1}\tag{\theequation}}

\newcommand{\simiid}{\overset{\text{i.i.d.}}{\sim}}
\def\dist{\operatorname{dist}}
\def\col{\operatorname{col}}
\def\sign{\operatorname{sign}}
\def\del{\boldsymbol{\Delta}}
\def\deld{\boldsymbol{\bar{\Delta}}}

\begin{abstract}
The Federated Averaging (FedAvg) algorithm, which  consists of  alternating between a few local stochastic gradient updates at client nodes, followed by a model averaging update at the server, is perhaps the most commonly used method in Federated Learning.  Notwithstanding its simplicity, several empirical studies have illustrated that the output model of FedAvg, after a few fine-tuning steps, leads to a model that generalizes well to new unseen tasks. This surprising performance of such a simple method, however, is not fully understood from a theoretical point of view. In this paper, we formally investigate this phenomenon  in the multi-task linear representation setting. We show that the reason behind generalizability of the FedAvg's output is its power in learning the common data representation among the clients' tasks, by leveraging the diversity among client data distributions via local updates. We formally establish the iteration complexity required by the clients for proving such result in the setting where the underlying shared representation is a linear map. To the best of our knowledge, this is the first such result for any setting. We also provide empirical evidence demonstrating FedAvg's representation learning ability in federated image classification with heterogeneous data.
\end{abstract}

\section{Introduction} \label{sec:intro}



Federated Learning (FL) \cite{mcmahan2017communication} provides a communication-efficient and privacy preserving means to learn from data distributed across clients such as cell phones, autonomous vehicles, and hospitals. 
FL aims for each client to benefit from collaborating in the learning process without sacrificing data privacy or paying a substantial communication cost. 
Federated Averaging (FedAvg) \cite{mcmahan2017communication} is the predominant FL algorithm. In FedAvg, also known as Local SGD \cite{stich2018local,stich2019error,wang2019cooperative}, the clients achieve communication efficiency by making multiple local updates of a shared global model before sending the result to the server, which averages the locally updated models to compute the next global model.

FedAvg is motivated by settings with {\em homogeneous} data across clients, since multiple local updates should improve model performance on all other clients' data when their data is similar. In contrast,  FedAvg faces two major challenges in more realistic {\em heterogeneous} data settings:  learning a single global model may not necessarily yield good performance for each individual client, and, multiple local updates may cause the FedAvg updates to drift away from  solutions of the global objective \cite{karimireddy2020scaffold,malinovskiy2020local,pathak2020fedsplit,charles2021convergence, wang2020tackling}. 
Despite these challenges, several empirical studies \cite{yu2020salvaging,reddi2020adaptive,li2020ditto} have observed that this shared global  model trained by  FedAvg {\em with several local updates per round} when further fine-tuned for individual clients is surprisingly effective in heterogeneous FL settings.
These studies motivate us to explore the impact of local updates on post-fine-tuning performance.

Meanwhile, a large number of recent works have shown that representation learning is a powerful paradigm for attaining high performance in multi-task settings, including FL. This is because the tasks' data often share a small set of features which are useful for downstream tasks, even if the datasets as a whole are heterogeneous. Consider, for example, heterogeneous federated image classification in which each client (task) may have images of different types of animals. It is safe to assume the images share a small number of features, such as body shape and color, which admit a simple  and accurate mapping  from feature space to label space. Since the number of important features is much smaller than the dimension of the data, knowing these features greatly simplifies each client's task. 

\begin{figure*}[t]
\centering
\begin{minipage}{.48\linewidth}
    \centering
    \includegraphics[width=0.94\linewidth]{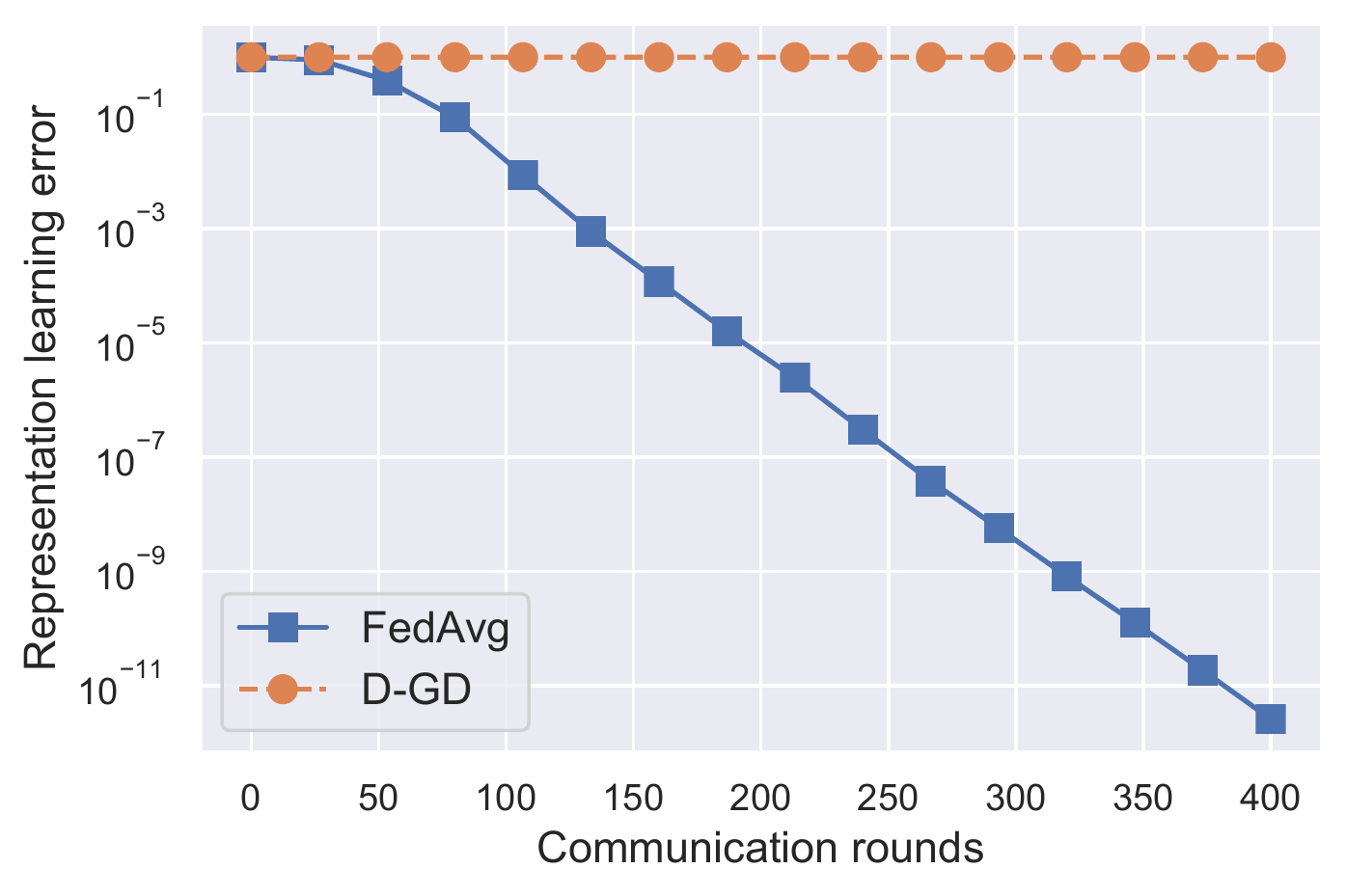}
    \vspace{-.2cm}
    \caption{
    In multi-task linear regression with population losses, FedAvg linearly converges to the ground-truth representation, while D-GD (FedAvg with one local update) fails to learn it.}
    \label{fig:1}
\end{minipage}%
\hspace{4mm}
\begin{minipage}{.48\linewidth}
    \centering
    \includegraphics[width=0.9\linewidth]{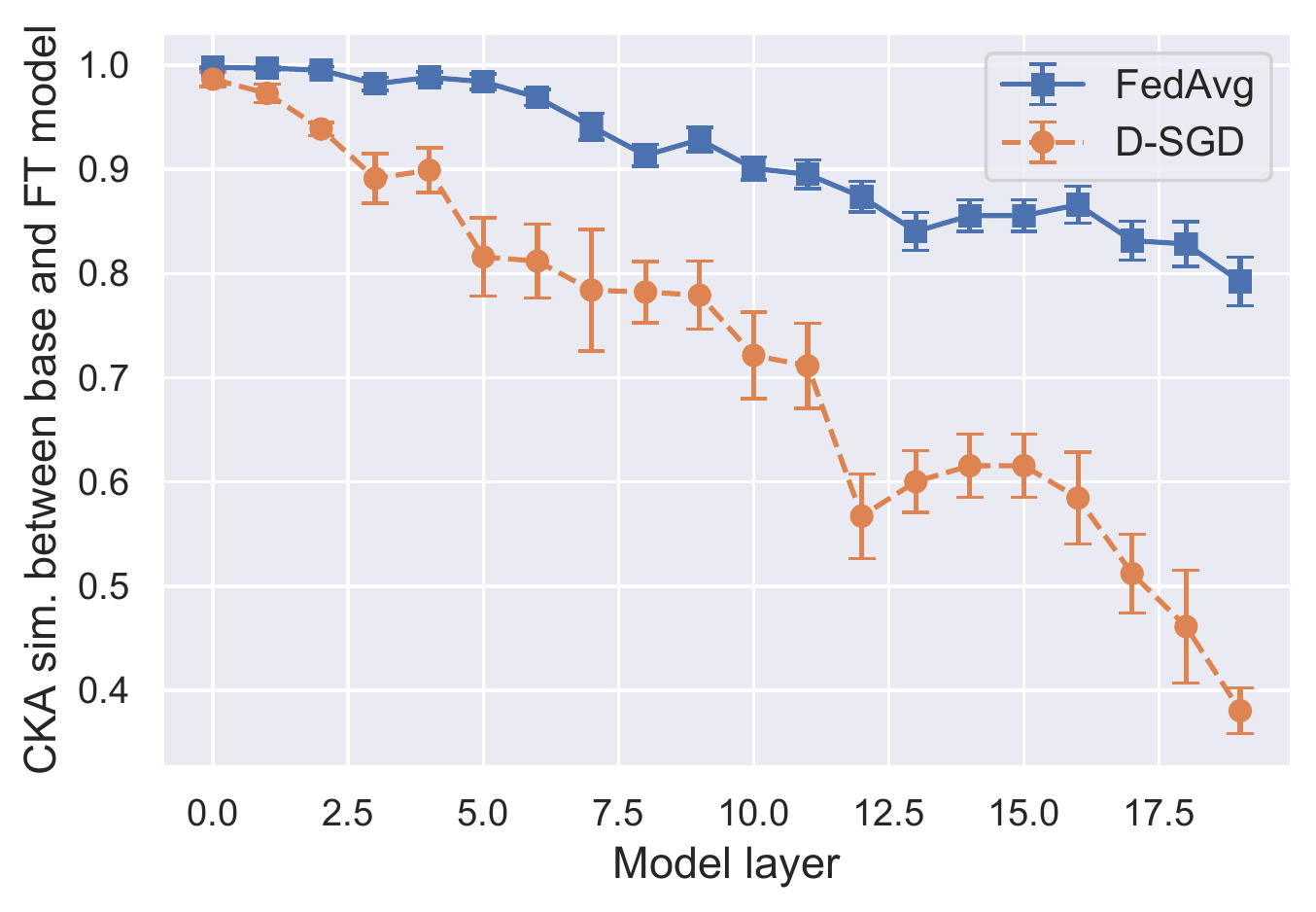}
     \vspace{-.2cm}
    \caption{The NN representation learned by FedAvg on CIFAR-100 with 5 classes/client does not change significantly when fine-tuned on a new dataset (CIFAR-10), unlike D-SGD.
    }
    \label{fig:2}
\end{minipage}
\end{figure*}

To explore the connection between local updates and representation learning, we 
first study multi-task linear regression
sharing a common ground-truth representation (Figure \ref{fig:1}).
We observe that FedAvg converges (exponentially fast) to the ground-truth representation in principal angle distance, while Distributed-GD (D-GD), which is effectively FedAvg with one local gradient update, fails to learn the shared representation. A similar concept can be shown in the nonlinear setting. We study a multi-layer CNN on a heterogeneous partition of CIFAR-100 (Figure \ref{fig:2}). Since there is not necessarily a ground-truth model here, we evaluate representation learning as follows. We first train the models with FedAvg and Distributed-SGD (D-SGD) then fine-tune the pre-trained models on clients from a new dataset, CIFAR-10.
Finally we evaluate the quality of the learned representation by measuring the amount that each model layer changes during fine-tuning using CKA similarity \cite{kornblith2019similarity}. Observe that the early layers of FedAvg's pre-trained model (corresponding to the representation) change much less than those of D-SGD. 
More details for both experiments are in Section \ref{sec:experiments} and Appendix \ref{app:experiments}.
These observations suggest that FedAvg learns a shared representation that generalizes to new clients, even when trained in a heterogeneous setting.  Hence, a natural question that arises is:

\vspace{-1.5mm}
\begin{quote}
\centering
{\em   {Does FedAvg provably learn effective representations of heterogeneous data?}}
\vspace{-1.5mm}
\end{quote}

We answer ``yes'' to this question by proving that FedAvg recovers the ground-truth representation in the case of multi-task linear regression. 
Critically, we show that FedAvg's local updates leverage the diversity among client data distributions to learn their common representation. This is surprising because FedAvg is a general-purpose algorithm not designed for representation learning.
Our analysis thus yields new insights on how FedAvg finds generalizable models. Our contributions are:

\vspace{-2mm}
\begin{itemize}[leftmargin=25pt, labelindent=1.5pt]
\item \textbf{Representation learning guarantees.} We study the behavior of FedAvg in multi-task linear regression with common representation. Here, each client aims to solve a $d$-dimensional regression with ground-truth solution that belongs to a shared $k$-dimensional subspace of $\mathbb{R}^d$, where $k\ll d$.
Our results show that FedAvg with $\tau\geq 2$ local updates learns the representation at a linear rate when each client accesses population gradients.
{\em To the best of our knowledge, this  is the first result showing that FedAvg learns an effective representation in any setting.}
\item \textbf{Insights on the importance of local updates.} Our analysis reveals that executing  more than one local update between communication rounds \textit{exploits the diversity} of the clients' ground truth regressors to improve the learned representation in all $k$ directions in the linear setting. 
In contrast, 
we prove that D-GD (FedAvg with one local update) fails to learn the representation.
    \item \textbf{Empirical evidence of representation learning.} We provide experimental results showing Fedavg learns a generalizable representation when we use deep neural networks on real-world data sets. This confirms the main massage of our theoretical results in the bilinear setting and suggest that our results can generalize to more complex scenarios.
\end{itemize}

\subsection{Related work}
Recently there has been a surge of interest motivated by FL in analyzing FedAvg/Local SGD in heterogeneous settings. Multiple works have shown that FedAvg converges to a global optimum (resp. stationary point) of the global  objective in convex (resp. nonconvex) settings but with decaying learning rate \cite{khaled2020tighter,koloskova2020unified,li2019convergence,qu2020federated,karimireddy2020scaffold}, leading to sublinear rates and communication complexity sometimes dominated by Distributed-SGD \cite{woodworth2020minibatch}.
These results are tight in the sense that FedAvg with fixed learning rate may {\em not} converge to a stationary point of the global objective in the presence of data heterogeneity, as its multiple local updates cause it to optimize a distinct, unknown objective \cite{malinovskiy2020local,mitra2021linear,pathak2020fedsplit,charles2021convergence, wang2020tackling,wang2019federated,li2019convergence}. Several methods have tried to correct this objective inconsistency via gradient tracking \cite{karimireddy2020scaffold, mitra2021linear, haddadpour2020federated,liang2019variance,murata2021bias,gorbunov2021local}, local regularization \cite{li2018federated,wang2019federated,zhang2020federated,t2020personalized}, operator splitting \cite{pathak2020fedsplit}, and strategic client sampling \cite{ribero2020communication,chen2020optimal,cho2020client}.
We take an orthogonal approach by arguing that FedAvg's local updates actually {\em benefit} convergence  in heterogeneous settings  by resulting in more generalizable models. 

Several papers have also analyzed FedAvg from a generalization perspective. It was shown in \cite{chen2021theorem} that in a setting with strongly convex losses, either  local training or FedAvg with fine-tuning  (but not both) achieves minimax risk, depending on the level of data heterogeneity.
Similarly, \cite{cheng2021fine} argued that FedAvg with fine-tuning generalizes as well as more sophisticated methods, including model-agnostic meta-learning (MAML) \cite{finn2017model,fallah2020personalized}, in a strongly convex regularized linear regression setting. Additional work has studied the generalization of FedAvg in kernel regression, but for convex objectives that do not allow for representation learning  \cite{su2021achieving}, and the generalization of a variant of FedAvg, known as Reptile \cite{nichol2018reptile}, on wide two-layer ReLU networks with homogeneous data \cite{huang2021fl}. 
We focus on the multi-task linear representation learning setting \cite{maurer2016benefit}, which has become popular in recent years as it is an expressive but tractable nonconvex setting for studying the sample-complexity benefits of learning representations and the representation learning abilities of popular algorithms in data heterogeneous settings \cite{ collins2021exploiting,du2020fewshot, tripuraneni2020provable,kong2020robust, thekumparampil2021statistically,collins2022maml,sun2021towards}. Remarkably, our study of FedAvg reveals that it can learn an effective representation even though it was not designed for this goal, unlike a variety of personalized FL methods specifically tailored for representation learning \cite{collins2021exploiting,liang2020think,arivazhagan2019federated,oh2021fedbabu}.

\textbf{Notations.} We use $\mathcal{N}(\mathbf{u}, \mathbf{\Sigma})$ to signify the multivariate Gaussian distribution with mean $\mathbf{u}$ and covariance $\mathbf{\Sigma}$.  $\mathcal{O}^{d \times k}$ denotes the set of matrices in $\mathbb{R}^{d \times k}$ with orthonormal columns. The notation $\col(\mathbf{B})$ represents the column space of the matrix $\mathbf{B}$, and $\col(\mathbf{B})^\perp$ is the orthogonal complement to this space. The norm $\|\cdot\|$ is the spectral norm and  $\mathbf{I}_d$ is the identity matrix in $\mathbb{R}^{d \times d}$.  We use $[m]$ to indicate the set of natural numbers up to and including $m$.

\section{Problem Formulation} \label{sec:formulation}

Consider a federated setting with a central server and $M$ clients. Each client $i\in [M]$ has a training dataset $\mathcal{\hat{D}}_i$ of $n_i$ labeled samples drawn from a distribution $\mathcal{D}_i$ over $\mathcal{X}\times \mathcal{Y}$, where $\mathcal{X}$ is the input space and $\mathcal{Y}$ is the label space. The learning model is given by $h_{\boldsymbol{\theta}}:\mathcal{X}\rightarrow \mathcal{Y}$ for model parameters $\boldsymbol{\theta}\in \mathbb{R}^D$. The loss of the model on a sample 
$(\mathbf{x}, \mathbf{y}) \in \mathcal{X}\times \mathcal{Y}$ is given by $\ell( h_{\boldsymbol{\theta}}(\mathbf{x}), \mathbf{y} )$, which may be, for example, the squared or cross entropy loss.  
The loss of model parameters $\boldsymbol{\theta}$ on the $i$-th client is the average loss of the model $h_{\boldsymbol{\theta}}$ on the samples in $\mathcal{\hat{D}}_i$, namely $f_i(\boldsymbol{\theta}) \coloneqq \tfrac{1}{n_i}\sum_{j=1}^{n_i}\ell(h_{\boldsymbol{\theta}}( \mathbf{x}_{i,j}),  \mathbf{y}_{i,j})$, where $(\mathbf{x}_{i,j}, \mathbf{y}_{i,j})$ is the $j$-th sample in $\mathcal{\hat{D}}_i$.
The server aims to  leverage all of the data across clients to find models that achieve small loss $f_i(\boldsymbol{\theta})$ for each client. To do so, the standard approach is to find a single model $\boldsymbol{\theta}$ that minimizes  the average of the client losses weighted by number of samples:
\begin{equation}
    \min_{\boldsymbol{\theta}} \frac{1}{N}\sum_{i=1}^M n_i f_i(\boldsymbol{\theta}) \label{global}\ = \
    \frac{1}{N}\sum_{i=1}^M 
    \sum_{j\in \mathcal{\hat{D}}_i}\ell(h_{\boldsymbol{\theta}}( \mathbf{x}_{i,j}),\mathbf{y}_{i,j})),
\end{equation}
where $N= \sum_{i=1}^M n_i$. Due to communication and privacy constraints, the clients cannot share their local data $\mathcal{\hat{D}}_i$, so \eqref{global} must be solved in a federated manner. 


\textbf{FedAvg.} The most common FL method is FedAvg. On each round $t$ of FedAvg, the server uniformly samples a set $\mathcal{I}_t$ of $m\leq M$ clients.
Each selected client receives the current global parameters $\boldsymbol{\theta}_t$, executes multiple SGD steps on its local data  starting from $\boldsymbol{\theta}_t$, then sends the result back to the server. The server then computes $\boldsymbol{\theta}_{t+1}$ as the weighted average of the updates.  Specifically, upon receiving the global model $\boldsymbol{\theta}_t$, client $i$ computes
\begin{align}
    \boldsymbol{\theta}_{t,i,s+1} &= \boldsymbol{\theta}_{t,i,s} - \alpha \mathbf{g}_{t,i,s}(\boldsymbol{\theta}_{t,i,s}), \label{eq_local_updates}
\end{align}
for $s=1,\dots,\tau-1$,
where $\tau$ is the number of local steps, $\boldsymbol{\theta}_{t,i,0} = \boldsymbol{\theta}_t$ and $\mathbf{g}_{t,i,s}(\boldsymbol{\theta}_{t,i,s})$ is a stochastic gradient of $f_i$ evaluated at $\boldsymbol{\theta}_{t,i,s}$ using $b$ samples from $\mathcal{\hat{D}}_i$. The client then sends $\boldsymbol{\theta}_{t,i,\tau}$ back to the server, which computes the next global iterate as:
\begin{equation}
    \boldsymbol{\theta}_{t+1} = \frac{1}{N_t}\sum_{i \in \mathcal{I}_t} n_i \boldsymbol{\theta}_{t,i,\tau},
\end{equation} where $N_t \!\coloneqq\! \sum_{i \in \mathcal{I}_t} n_i$.
Note that $\tau\!=\!1$ corresponds to D-SGD, also known as mini-batch SGD whose convergence properties are well-understood \cite{woodworth2020minibatch,nguyen2018sgd,shamir2014distributed,gower2019sgd}. FedAvg improves the communication efficiency of D-SGD by making  $\tau\!\geq\! 2$ local updates between communication rounds. 


\textbf{Fine-tuning.} After training for $T$ communication rounds, the global parameters $\boldsymbol{\theta}_T$ learned by FedAvg are typically fine-tuned on each client before testing. In particular, starting from $\boldsymbol{\theta}_T$, client $i$ executes $\tau'$ steps of SGD on its local data as follows:
\begin{align}
   \boldsymbol{\theta}_{T,i,s+1} = \boldsymbol{\theta}_{T,i,s} - \alpha \mathbf{g}_{T,i,s}(\boldsymbol{\theta}_{t,i,s})
\end{align}
for $s\in [\tau'-1]$. The fine-tuned model ultimately used for testing is $\boldsymbol{\theta}_{T,i,\tau'}$. Note that a new client, indexed by $M+1$, entering the system after FedAvg training has completed can also fine-tune $\boldsymbol{\theta}_{T}$ using the same procedure to obtain a personalized solution $\boldsymbol{\theta}_{T,M+1,\tau'}$. 



\textbf{Representation learning.} We aim to answer why the fine-tuned models $\{\boldsymbol{\theta}_{T,i,\tau'}\}_{i=1}^{M+1}$ perform well in practice by taking a representation learning perspective. We show that the output of FedAvg, i.e., $\mathbf{\theta}_T$, has learned the common data representation among clients. To formalize this result, we consider a class of models that can be written as the composition of a representation $h^{\text{rep}}$ and a prediction module, i.e. head, denoted as $h^{\text{head}}$. Let the model parameters be split as $\boldsymbol{\theta} := [\boldsymbol{\phi}, \boldsymbol{\psi}]$, where $\boldsymbol{\phi}$ contains the representation parameters and $\boldsymbol{\psi}$ contains the head  parameters. Then, for any $\mathbf{x}\in \mathcal{X}$, the prediction of the learning model is $h_{\boldsymbol{\theta}}(\mathbf{x}) = (h_{\boldsymbol{\psi}}^{\text{head}} \circ h_{\boldsymbol{\phi}}^{\text{rep}})(\mathbf{x}) = h_{\boldsymbol{\psi}}^{\text{head}} (h_{\boldsymbol{\phi}}^{\text{rep}}(\mathbf{x}))$. 
For instance, if $h_{\boldsymbol{\theta}}$ is a neural network with weights $\boldsymbol{\theta}$, then $h_{\boldsymbol{\phi}}^{\text{rep}}$ is the first many layers of the network with weights $\boldsymbol{\phi}$, and $h_{\boldsymbol{\psi}}^{\text{head}}$ is the network last few layers with weights $\boldsymbol{\psi}$. 
A standard assumption in multi-task settings is existence of a common representation $h_{\boldsymbol{\phi}_\ast}^{\text{rep}}$ that admits an easily learnable head $h_{\boldsymbol{\psi}_{\ast,i}}^{\text{rep}}$ such that $h_{\boldsymbol{\psi}_{\ast,i}}^{\text{head}} \circ h_{\boldsymbol{\phi_*}}^{\text{rep}}$ performs well for task $i$. It is thus of interest to all the clients to learn $h_{\boldsymbol{\phi}_\ast}^{\text{rep}}$.

\section{Main Results} \label{sec:linear}

We employ the standard setting used for algorithmic representation learning analysis: multi-task linear regression \cite{tripuraneni2020provable,thekumparampil2021sample,collins2021exploiting,chua2021fine}. In this setting, samples $(\mathbf{x}_{i,j},y_{i,j})$ for each client $i$ are drawn independently from a distribution $\mathcal{D}_i$ on $\mathbb{R}^d\times \mathbb{R}$  such that
\begin{align}
    \mathbf{x}_{i,j} \stackrel{\text{i.i.d.}}{\sim} p_{\mathbf{x}}, \;\; y_{i,j} = \langle \boldsymbol{\beta}_{\ast,i}, \mathbf{x}_{i,j} \rangle + \zeta_{i,j} \; \text{ where } \; \zeta_{i,j} \stackrel{\text{i.i.d.}}{\sim} p_{\zeta} \nonumber
\end{align}
for an unobserved ground-truth regressor $\boldsymbol{\beta}_{\ast,i}\!\in \!\mathbb{R}^d$ and label noise $\zeta_{i,j}$. We assume the distributions $p_{\mathbf{x}}$ and $p_{\zeta}$ are such that $\mathbb{E}[\mathbf{x}_{i,j}]\! =\! \mathbf{0}, \mathbb{E}[\mathbf{x}_{i,j}\mathbf{x}_{i,j}^\top]\! =\! \mathbf{I}_d$ and $\mathbb{E}[\zeta_{i,j}]\! =\! 0$.

To incentivize representation learning, each $\boldsymbol{\beta}_{\ast,i}$ belongs to the same $k$-dimensional subspace of $\mathbb{R}^d$, where $k\ll d$. Let $\mathbf{B}_\ast \in \mathcal{O}^{d\times k}$ have columns that form an orthogonal basis for the shared subspace, so that $\boldsymbol{\beta}_{\ast,i} = \mathbf{B}_\ast \mathbf{w}_{\ast,i}$ for some  $\mathbf{w}_{\ast,i}\in\mathbb{R}^k$ for each $i$. In other words, there  exists a low-dimensional set of parameters known as the ``head'' that can specify the ground-truth model for client $i$ once the shared representation, i.e., $\col(\mathbf{B}_\ast)$, is known. It is advantageous to learn $\col(\mathbf{B}_\ast)$ because once it is known, all clients (including potentially new clients entering the system) have sample complexity $O(k)\ll d$ as they only need to learn the parameters of their head \cite{tripuraneni2020provable,du2020fewshot}. 

Each client $i$ ultimately aims to learn a model $\hat{\boldsymbol{\beta}}_i$ that approximates $\boldsymbol{\beta}_{\ast,i}$ in order to achieve good generalization on its local distribution. 
To eventually achieve this for each client, FedAvg with fine-tuning first aims to learn a global model consisting of a representation $\mathbf{B}\in\mathbb{R}^{d \times k}$ and a head $\mathbf{w}\in\mathbb{R}^k$ that minimizes the average loss across clients. The loss for client $i$ is $f_i(\mathbf{B}, \mathbf{w}) \coloneqq \frac{1}{2n_i}\sum_{j=1}^{n_i} (y_{i,j} - \langle \mathbf{Bw}, \mathbf{x}_{i,j} \rangle)^2$, i.e. the average squared loss on the local data, so
FedAvg tries to learn a global model that solves the nonconvex problem:
\begin{align}
 \min_{\mathbf{B}\in\mathbb{R}^{d\times k}, \mathbf{w}\in \mathbb{R}^{k}} \frac{1}{N} \sum_{i=1}^M n_i\bigg\{ f_{i}(\mathbf{B},\mathbf{w}) \coloneqq \frac{1}{2n_i}\sum_{j=1}^{n_i} (y_{i,j} - \langle \mathbf{Bw}, \mathbf{x}_{i,j} \rangle)^2\bigg\}. \label{global_linear}
\end{align}
To solve \eqref{global_linear} in a distributed manner, FedAvg dictates that each client makes a series of local updates of the current global model before returning the models to the server for averaging, as discussed in Section \ref{sec:formulation}. 
We aim to show that the FedAvg training procedure learns the column space of $\mathbf{B}_\ast$.
The first step is to  make standard diversity and normalization assumptions on the ground-truth heads.
\begin{assumption}[Client normalization] \label{assump:n}
There exists $L_{\max}<\infty$ s.t. $\forall i \in [M]$, $\|\mathbf{w}_{\ast,i}\|_2 \leq L_{\max}$.
\end{assumption}
\begin{assumption}[Client diversity] \label{assump:td}
There exists $\mu \!>\! 0$ s.t. $\sigma_{\min}(\tfrac{1}{M}\sum_{i=1}^M  (\mathbf{w}_{\ast,i}-\mathbf{\bar{w}}_{\ast})(\mathbf{w}_{\ast,i}-\mathbf{\bar{w}}_{\ast})^\top) \geq \mu^2 $, where $\mathbf{\bar{w}}_{\ast} \coloneqq \tfrac{1}{M} \sum_{i=1}^M \mathbf{w}_{\ast,i}$. Define $\kappa_{\max}\coloneqq \nicefrac{L_{\max}}{\mu}$. 
\end{assumption}
Assumption \ref{assump:td} is very similar to typical task diversity assumptions except that it quantifies the diversity of the centered rather than un-centered tasks \cite{du2020fewshot,tripuraneni2020provable}. Intuitively, task diversity is required so that all of the directions in $\col(\mathbf{B}_\ast)$ are observed.
Next, to obtain convergence results we must define the variance of the ground-truth heads and the principal angle distance between representations.

\begin{definition}[Client variance] \label{def:var}
For $\gamma > 0$, define: $\gamma^2 \coloneqq \frac{1}{M}\sum_{i=1}^M \|\mathbf{w}_{\ast,i} - \mathbf{\bar{w}}_{\ast}\|^2$, where $\mathbf{\bar{w}}_{\ast}$ is defined in Assumption \ref{assump:td}. For $H > 0$, define $H^4  \coloneqq \frac{1}{M}\sum_{i=1}^M \| \mathbf{w}_{\ast,i}\mathbf{w}_{\ast,i}^\top - \tfrac{1}{M}\sum_{i'=1}^M\mathbf{\bar{w}}_{\ast,i}\mathbf{\bar{w}}_{\ast,i}^\top\|^2$.
\end{definition}
\begin{definition}[Principal angle distance]
For two matrices $\mathbf{B}_1, \mathbf{B}_2 \in \mathbb{R}^{d \times k}$, the {\em principal angle distance} between $\mathbf{B}_1$ and $\mathbf{B}_2$ is defined as $\dist(\mathbf{B}_1, \mathbf{B}_2) \coloneqq \|\mathbf{\bar{B}}_{1,\perp}^\top \mathbf{\bar{B}}_{2}\|_2$, 
where the columns of $\mathbf{\bar{B}}_{1,\perp} \in \mathcal{O}^{d \times d-k}$ and $\mathbf{\bar{B}}_{2} \in \mathcal{O}^{d \times k}$ form orthonormal bases for $\col(\mathbf{B}_1)^\perp$ and $\col(\mathbf{B}_2)$, respectively.
\end{definition}
Intuitively, the principal angle distance between  $\mathbf{B}_1$ and $\mathbf{B}_2$ is the sine of the largest angle between the subspaces spanned by their columns.
Now we are ready to state our main result. We consider the case that each client has access to gradients of the population loss on its local data distribution. 


\begin{theorem}[FedAvg Representation Learning] \label{thm:main_pop} Consider the case that each client takes gradient steps with respect to their population loss $f_i(\mathbf{B},\mathbf{w})\coloneqq \tfrac{1}{2}\|\mathbf{Bw}-\mathbf{B}_{\ast}\mathbf{w}_{\ast,i}\|^2$ and all losses are weighted equally in the global objective. 
Suppose Assumptions  \ref{assump:n} and \ref{assump:td}  hold, the number of clients participating each round satisfies $m \geq \min(M, 20((\nicefrac{\gamma}{L_{\max}})^2 + (\nicefrac{H}{L_{\max}})^4)(\alpha L_{\max})^{-4}\log(kT))$, and the initial parameters satisfy (i)  $ \delta_0 \coloneqq \dist(\mathbf{B}_0, \mathbf{B}_\ast)\leq \sqrt{1\! -\! E_0}$ for any $E_0 \in (0,1]$, (ii) $\|\mathbf{I}- \alpha \mathbf{B}_0^\top \mathbf{B}_0\|_2 = O(\alpha^2 \tau L_{\max}^2 \kappa_{\max}^2)$ and (iii) $\|\mathbf{w}_0\|_2 = O(\alpha^{2.5} \tau L_{\max}^3)$.
Choose  step size $\alpha = O(\tfrac{1-\delta_0}{ \sqrt{\tau} L_{\max} \kappa_{\max}^2})$. 
    Then for any $\epsilon \in (0, 1)$, the distance of the representation learned by FedAvg with $\tau\geq 2$ local updates satisfies  $\dist(\mathbf{B}_T, \mathbf{B}_\ast) <\epsilon$ after at most
\begin{align}
T = O \big(\tfrac{1}{\alpha^2 \tau \mu^2 E_0}\log(\nicefrac{1}{\epsilon})\big)
\end{align}
communication rounds with probability at least $1 - 4 (kT)^{-99}$. 
\end{theorem}
Theorem \ref{thm:main_pop} shows that FedAvg converges exponentially fast to the ground-truth representation when executed on the clients' population losses.
We provide intuition for the proof in Section \ref{sec:sketch} and the full proof in Appendix \ref{app:proof}. First, some comments are in order.

\textbf{Mild initial conditions.} 
Theorem \ref{thm:main_pop} holds under benign initial conditions. In particular, condition ($i$) requires that the initial distance is only a constant away from 1. Condition ($ii$) ensures that the initial representation is well-conditioned with appropriate scaling, and ($iii$) guarantees the initial head is not too large. The last two conditions can be easily achieved by normalizing the inputs. 


\textbf{Generalization without convergence in terms of the global loss.} When each client accesses its population loss as in Theorem \ref{thm:main_pop}, the global objective is:
\begin{align}
     \min_{\mathbf{B}\in\mathbb{R}^{d\times k}, \mathbf{w}\in \mathbb{R}^{k}} \frac{1}{M} \sum_{i=1}^M \|\mathbf{Bw}-\mathbf{B}_\ast\mathbf{w}_{\ast,i}\|^2 \label{glob_pop}
\end{align}
However, Theorem \ref{thm:main_pop} does not imply that FedAvg solves \eqref{glob_pop}. In fact, our simulations in Section \ref{sec:experiments} show that it does not even reach a stationary point of \eqref{glob_pop}. This is consistent with prior works that have noticed the ``objective inconsistency'' phenomenon of FedAvg: it solves an unknown objective distinct from the global objective due to the fact that after multiple local updates, local gradients are no longer unbiased estimates of gradients of \eqref{glob_pop} \cite{wang2020tackling}. 
Nevertheless, our results show that FedAvg is able to learn a generalizable model  {\em even when it does not optimize the global loss in data heterogeneous settings}.

\textbf{Multiple local updates critically harness diversity, whereas Distributed GD (D-GD) does not learn the representation.} Key to the proof of Theorem \ref{thm:main_pop} is that the locally-updated heads become {\em diverse}, meaning that they cover all directions in $\mathbb{R}^k$, with greater diversity corresponding to more evenly covering in all directions. We will show in Section \ref{sec:sketch} that the locally-updated heads become roughly as diverse as the ground-truth heads, and this causes the representation to move towards the ground-truth at rate depending on the  diversity level. Theorem \ref{thm:main_pop} reflects this: the convergence rate improves with the diversity metric $\nicefrac{\mu}{L_{\max}}$. In this way FedAvg {\em exploits} data heterogeneity to learn the representation, as more diverse $\{\mathbf{w}_{\ast,i}\}_{i\in [M]}$ implies more heterogeneous data.
Importantly, head diversity only benefits the global representation update if $\tau\geq 2$. We formally prove that D-GD (equivalent to FedAvg with $\tau\!=\!1$ and $m\!=\!M$) cannot recover $\col(\mathbf{B}_\ast)$ in the following result.
\begin{proposition}[Distributed GD lower bound] \label{prop:dgd}
Suppose we are in the setting described in Section \ref{sec:linear} and $d\! >\! k \!>\! 1$. Then for any set of ground-truth heads $\{\mathbf{w}_{\ast,i}\}_{i=1}^M$, full-rank initialization $\mathbf{B}_0\in\mathbb{R}^{d\times k}$, initial distance $\delta_0 \in (0,1/2]$, step size $\alpha > 0$, and number of rounds $T$, there exists $\mathbf{B}_\ast\in \mathcal{O}^{d \times k}$ such that $\dist(\mathbf{B}_0, \mathbf{B}_\ast) = \delta_0$, $\mathbf{B}_\ast\mathbf{\bar{w}}_{\ast}\in \col(\mathbf{B}_0)$ and $\dist(\mathbf{B}_T^{\text{D-GD}}, \mathbf{B}_\ast) \geq 0.7\delta_0$, where $\mathbf{B}_T^{\text{D-GD}}\equiv \mathbf{B}_T^{\text{D-GD}}(\mathbf{B}_0,\mathbf{B}_\ast,\{\mathbf{w}_{\ast,i}\}_{i=1}^M, \alpha)$ is the result of D-GD with step size $\alpha$ and initialization $\mathbf{B}_0$ on the system with ground-truth representation $\mathbf{B}_\ast$ and ground-truth heads $\{\mathbf{w}_{\ast,i}\}_{i=1}^M$.
\end{proposition}
Proposition \ref{prop:dgd} shows that for any choice of $\delta_0 \in (0,1/2]$, non-degenerate initialization $\mathbf{B}_0$, and ground-truth heads, there exists a $\mathbf{B}_\ast$ whose column space is $\delta_0$-close to $\col(\mathbf{B}_0)$, yet is at least $0.7\delta_0$-far from the representation learned by D-GD in the setting with $\mathbf{B}_\ast$ as ground-truth.
Therefore, even allowing for a strong initialization, D-GD cannot guarantee to recover the ground-truth representation. This negative result 
combined with our previous results suggest even if we had an infinite communication budget, it would still be advantageous to execute multiple local updates between communication rounds in order to achieve better generalization through representation learning.

\section{Intuitions and Proof Sketch} \label{sec:sketch}
Next we highlight the key ideas behind the importance of local updates and why FedAvg learns $\col(\mathbf{B}_\ast)$, while D-GD fails to achieve this goal.

\textbf{Global update $\mathbf{B}_{t+1}$.} To build intuition for why FedAvg can learn $\col{(\mathbf{B}_\ast)}$, we examine the global update of the representation in the full participation case ($m=M$): 
 \begin{align}
    \mathbf{B}_{t+1} &= \mathbf{{B}}_{t}\bigg(\underbrace{ \tfrac{1}{M} \sum_{i=1}^M \prod_{s=0}^{\tau-1}(\mathbf{I}_k - \alpha \mathbf{w}_{t,i,s}\mathbf{w}_{t,i,s}^\top}_{\text{prior weight}})\bigg)  + \mathbf{{B}}_\ast  \bigg(\underbrace{\tfrac{\alpha}{M}\sum_{i=1}^M \mathbf{w}_{\ast,i}  \sum_{s=0}^{\tau-1}\mathbf{w}_{t,i,s}^\top \prod_{r=s+1}^{\tau-1}(\mathbf{I}_k - \alpha \mathbf{w}_{t,i,r}\mathbf{w}_{t,i,r}^\top}_{\text{signal weight}}) \bigg)\nonumber 
\end{align}
Notice that $\mathbf{B}_{t+1}$ is a mixture of $\mathbf{B}_t$ and $\mathbf{B}_\ast$ with weight matrices in $\mathbb{R}^{k\times k}$. We aim to show that 
\vspace{-1mm}
\begin{enumerate}[(I), leftmargin=25pt, labelindent=1.5pt]
    \item the `prior weight'  on $\mathbf{B}_t$  has spectral norm strictly less than 1, and
    \item the `signal weight' on $\mathbf{B}_\ast$ adds energy from $\col(\mathbf{B}_\ast)$ to $\mathbf{B}_{t+1}$ so that $\sigma_{\min}(\mathbf{B}_{t+1}) \approx \sigma_{\min}(\mathbf{B}_t)$.
\end{enumerate}
\vspace{-1mm}
These properties ensure that the contribution from $\col(\mathbf{B}_t)$ in $\col(\mathbf{B}_{t+1})$ contracts, while  energy from $\col(\mathbf{B}_\ast)$ replaces the lost energy from $\col(\mathbf{B}_t)$. Consequently, $\col(\mathbf{B}_{t+1})$ moves closer to $\col(\mathbf{B}_\ast)$ in all $k$ directions.



\textbf{The role of head diversity and multiple local updates.} To show {(I)} and {(II)}, it is imperative to use the diversity of the locally-updated heads when $\tau\geq 2$. First consider {(I)}. Notice that for each $i$, $\prod_{s=0}^{\tau-1}({\mathbf{I}_k \!-\! \alpha \mathbf{w}_{t,i,s}\mathbf{w}_{t,i,s}^\top})$ has singular values at most 1, and strictly less than 1 corresponding to directions spanned by $\{ \mathbf{w}_{t,i,s}\}_{s\in[\tau-1]}$. Thus, the maximum singular value of the average of these matrices should be strictly less than 1 as long as $\{ \mathbf{w}_{t,i,s}\}_{s\in[\tau-1],i \in [M]}$ spans $\mathbb{R}^k$, i.e. the locally-updated heads are diverse. Similarly,  the signal weight is rank-$k$ if the locally-updated heads span $\mathbb{R}^k$, which leads to (II) as  discussed below. 
In contrast, if $\tau=1$, then  the global update of the representation does {\em not} leverage head diversity, as it is only a function of the global head and the average ground-truth head: $\mathbf{B}_{t+1} = \mathbf{B}_t({\mathbf{I}_k - \alpha  \mathbf{w}_t \mathbf{w}_t^\top}) +  \alpha\mathbf{B}_{\ast}{\mathbf{\bar{w}}_{\ast} \mathbf{w}_t^\top}$ in this case. As a result, $\col(\mathbf{B}_{t+1})$ can only improve in one direction,  so D-GD ultimately fails to learn $\col(\mathbf{B}_\ast)$ (Prop. \ref{prop:dgd}).

\textbf{Achieving head diversity: the necessity of controlling $\mathbf{I}_k\!- \alpha\mathbf{B}_{t}^\top \mathbf{B}_t$.} We have discussed the intuition for why head diversity implies (I) and (II) for FedAvg. Next, we investigate why the heads become diverse. 
Let us examine client $i$'s first local update for the head at round $t$:
\begin{align}
    \mathbf{w}_{t,i,1} = (\mathbf{I}_k-\alpha \mathbf{B}_{t}^\top \mathbf{B}_{t}) \mathbf{w}_{t} + \alpha \mathbf{B}_{t}^\top \mathbf{B}_{\ast}\mathbf{w}_{\ast,i}  \nonumber
\end{align}
{From this equation we see that if $\del_t \coloneqq \mathbf{I}_k - \alpha \mathbf{B}_t^\top \mathbf{B}_t \approx \mathbf{0}$ and $\|\mathbf{w}_t\|$ is bounded, then $\mathbf{w}_{t,i,1} \approx {\alpha} \mathbf{B}_t^\top \mathbf{B}_\ast \mathbf{w}_{\ast,i}$. If this approximation holds, then  $\{\mathbf{w}_{t,i,1}\}_{i \in [M]}$ inherits the diversity of $\{\mathbf{w}_{\ast,i}\}_{i \in [M]}$, which is indeed diverse due to Assumption \ref{assump:td}, meaning that the local heads are diverse after just one local update. 
Moreover, it can be shown that if $\del_t \approx \mathbf{0}$  and the heads become diverse after one local update, then they remain diverse for all local updates due to the observation that each $\mathbf{B}_{t,i,s}$ changes slowly over $s$.
Note that in addition to implying local head diversity, $\del_t \approx \mathbf{0}$ for all $t$ implies $\sigma_{\min}(\mathbf{B}_t) \approx \sigma_{\min}(\mathbf{B}_{t+1}) \approx \tfrac{1}{\sqrt{\alpha}}$, which  directly ensures {(II)}. Thus we aim to show $\del_t \approx \mathbf 0$ for all communication rounds, i.e. $\mathbf{B}_t$ remains close to a scaled orthonormal matrix.  

However, it is surprising why $\|\del_t\|$ remains small:
$\mathbf{B}_{t+1}$ is the average of nonlinearly locally-updated representations, and the local updates could `overfit' by adding more energy to some columns than others, and/or lead to cancellation when summed, so it is not intuitive why $\sqrt{\alpha}\mathbf{B}_t$ remains almost orthonormal. Nor does the expression above for $\mathbf{B}_{t+1}$ provide any clarity on this.} Nevertheless, through a careful induction we show that $\del_t$ indeed stays close to zero since the local heads converge quickly and the projection of the local representation gradient onto $\col(\mathbf{B}_t)$ is exponentially decaying.

\textbf{Inductive argument.} While the above intuitions seem to simplify the behavior of FedAvg, showing that they all hold simultaneously is not at all obvious. To study this, we are inspired by recent work \cite{collins2022maml} that developed an inductive argument for representation learning in the context of gradient based-meta-learning.
To formalize our intuition discussed previously, in our proof we need to show that (i) the learned representation does not overfit to each client's loss despite {\em many local updates} and simultaneously the heads quickly become diverse, and (ii) the update at the global server preserves the learned representation despite {\em averaging} many nonlinearly perturbed representations gathered from clients after local updates. To address these challenges, we construct a pair of intertwined induction hypotheses over time, one for tracking the effect of local updates, and another for tracking the global averaging. Each induction hypothesis (local and global) itself consists of several hypotheses (in effect, a nested induction) that evolve both within and across communication rounds.

\textit{Local induction.} The proof leverages the following local inductive hypotheses for every $t,i$:
\begin{enumerate}
    \item $A_{1,t,i}(s) \coloneqq \{\|\mathbf{w}_{t,i,s'}- \alpha \mathbf{B}_{t,i,s'-1}^\top \mathbf{B}_{\ast}\mathbf{w}_{\ast,i}\|_2 = c_1\alpha^{2.5} \tau  L_{\max}^3  \kappa_{\max}^2 E_0^{-1} \quad \forall s'\in \{1,\dots,s\}\}$
      \item $A_{2,t,i}(s) \coloneqq \{\|\mathbf{w}_{t,i,s'}\|_2 \leq c_2 \sqrt{\alpha} L_{\max} \quad \forall s'\in \{1,\dots,s\}  \}$ 
     \item $A_{3,t,i}(s) \coloneqq \{\|\mathbf{I}_k - \alpha \mathbf{B}_{t,i,s'}^\top \mathbf{B}_{t,i,s'}\|_2 = c_3\alpha^2 L_{\max}^2 \kappa_{\max}^2 E_0^{-1} \quad \forall s'\in \{1,\dots,s\}  \}$ 
        \item $A_{4,t,i}(s)\coloneqq \{ \dist(\mathbf{B}_{t,i,s'}, \mathbf{B}_\ast) \leq c_4 \dist(\mathbf{B}_{t}, \mathbf{B}_\ast)\quad \forall s'\in \{1,\dots,s\}\}$
    \end{enumerate}
    
    
The local induction tracks the effect of updates at each client node: At the end of $\tau$ local updates, $A_{1,t,i}(\tau)$ captures the diversity of the local heads, $A_{2,t,i}(\tau)$ ensures that the heads remain uniformly bounded, 
$A_{3,t,i}(\tau)$ shows that the locally adapted representations stay close to a scaled orthonormal matrix, and $A_{4,t,i}(\tau)$ shows that the locally adapted representations do not diverge too quickly from  the ground-truth.
The second set of inductions below controls the global behavior.

\textit{Global induction.} The global induction utilizes a similar set of inductive hypotheses.

\begin{enumerate}
\item $A_1(t) \coloneqq \{ \|\mathbf{w}_{t'} - \alpha(\mathbf{I}_k + \del_{t'}) \mathbf{B}_{t'}^\top \mathbf{B}_\ast\mathbf{\bar{w}}_{\ast,t'}\|_2 = c_1'\alpha^{2.5}\tau L_{\max}^3  \quad \forall t'\in \{1,\dots, t\}\}$ 
\item $A_2(t) \coloneqq \{ \|\mathbf{w}_{t'}\|_2 \leq  c_2'\sqrt{\alpha} L_{\max}  \quad \forall t'\in \{1,\dots, t\}\}$
    \item $A_3(t) \coloneqq \{ \|\del_{t'}\|_2 = c_3'\alpha^2 \tau  L_{\max}^2  \kappa_{\max}^2 E_0^{-1} \quad \forall t'\in \{1,\dots, t\}\}$
    \item $A_4(t) \coloneqq \{ \|\mathbf{B}_{\ast,\perp}^\top \mathbf{B}_{t'}\|_2 \leq  (1 - c_4' \alpha^2 \tau \mu^2 E_0)\|\mathbf{B}_{\ast,\perp}^\top \mathbf{B}_{t'-1} \|_2 \quad \forall t'\in \{1,\dots, t\} \}$
    \item $A_5(t) \coloneqq \{ \dist(\mathbf{B}_{t},\mathbf{B}_\ast) \leq  (1 - c_5'\alpha^2 \tau \mu^2 E_0)^{t-1} \quad \forall t'\in \{1,\dots, t\}  \}$
\end{enumerate}
Hypotheses $A_1(t)$, $A_2(t)$ and $A_3(t)$ are analogous to $A_{1,t,i}(s)$, $A_{2,t,i}(s)$ and $A_{3,t,i}(s)$, respectively. $A_4(t)$ shows that the energy of $\col(\mathbf{B}_t)$ that is orthogonal to the ground-truth subspace is contracting, and $A_5(t)$ finally shows that the principal angle distance between the learned and ground-truth representations is exponentially decreasing. 
Our main claim follows from $A_5(T)$. 
However, proving this result requires showing that all the above local and global hypotheses hold for all times $t\geq1$, as these hypotheses are heavily coupled.
As mentioned previously, the most difficult challenge is controlling $\|\del_t\|$ ($A_3(t)$) despite many local updates, and doing so requires leveraging both local and global properties.
The details of this local-global induction argument are in Appendix \ref{app:proof}.

\section{Experiments} \label{sec:experiments}


In this section we conduct experiments to (I) verify our theoretical results in the linear setting and (II) determine whether our established insights generalize to deep neural networks. Notably, demonstrating the competitive performance of FedAvg plus fine-tuning for personalized FL is {\em not} a goal of this section, as this is evident from prior experiments \cite{yu2020salvaging,li2020ditto,collins2021exploiting,cheng2021fine}. Rather, to achieve (II) we test whether FedAvg learns effective representations when trained with neural networks in heterogeneous data settings via three popular benchmarks for evaluating the quality of learned representations. Since our main claim is that local updates are key to representation learning, we use distributed SGD as our baseline in all experiments.  

\begin{figure}
     \centering
     \begin{subfigure}[b]{0.43\textwidth}
         \centering
         \includegraphics[width=\textwidth]{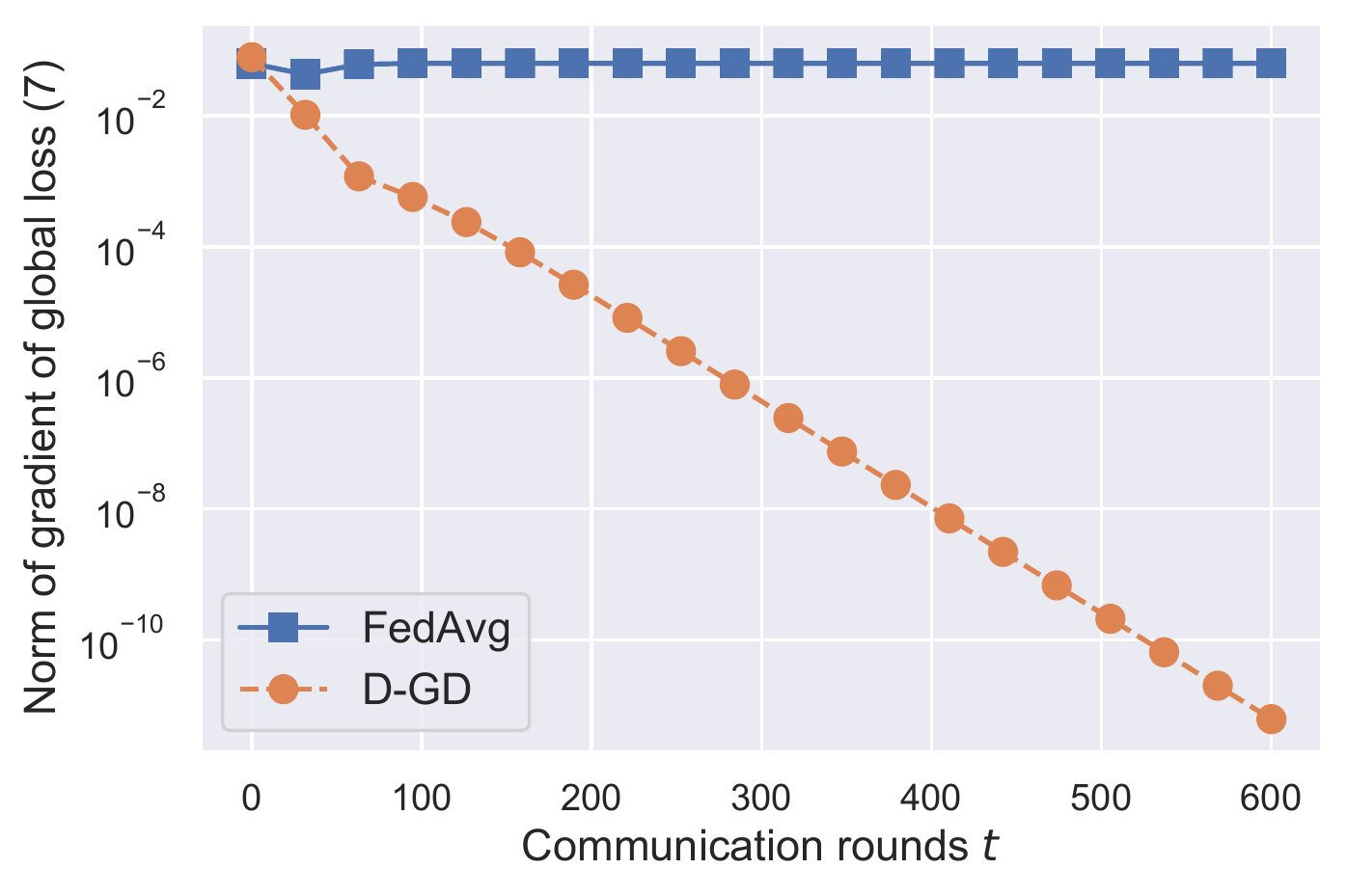}
     \end{subfigure}
     \qquad
     \begin{subfigure}[b]{0.42\textwidth}
         \centering
         \includegraphics[width=\textwidth]{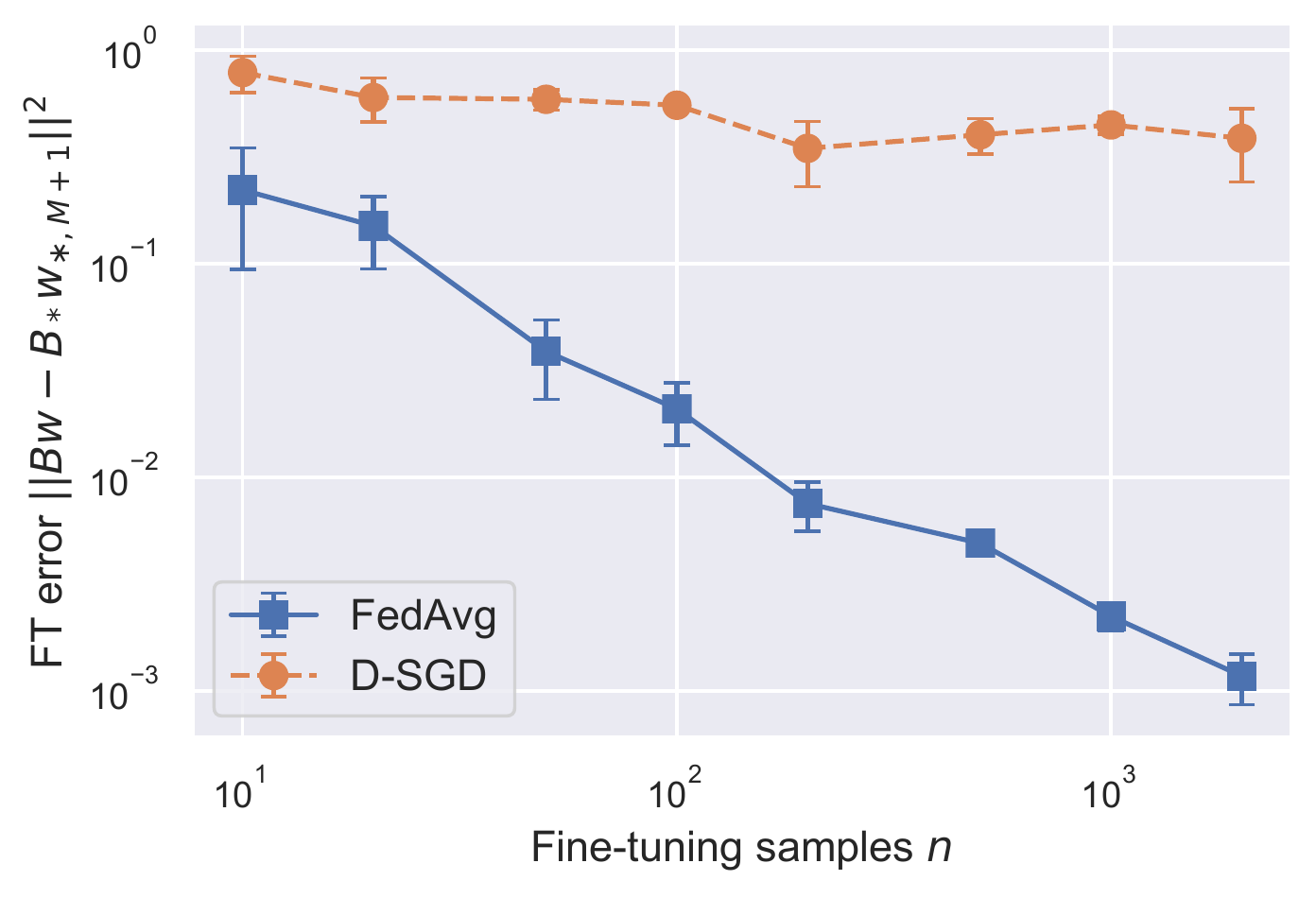}
     \end{subfigure}
        \caption{(Left) D-GD converges to a stationary point of the global objective \eqref{glob_pop}, unlike FedAvg, yet (Right) FedAvg achieves smaller error after fine-tuning with various numbers of samples.  
        } 
        \label{fig:linear}
\end{figure}

\subsection{Multi-task linear regression} We first experiment with the regression setting from our theory. We randomly generate $\mathbf{B}_\ast \in \mathbb{R}^{d \times k}$ and $\{\mathbf{w}_{\ast,i}\}_{i\in [M]}$ by sampling each element i.i.d. from the normal distribution, where $d\!=\!100, k\!=\!5$ and $M\!=\!40$, and then orthogonalizing $\mathbf{B}_\ast$. Then we run FedAvg with $\tau\!=\!2$ local updates and D-GD, both sampling $m\!=\!M$ clients per round. We have seen in Figure \ref{fig:1} that the principal angle distance between the representation learned by FedAvg and the ground-truth representation linearly converges to zero, whereas D-GD does not learn the ground-truth representation. Conversely, Figure \ref{fig:linear} (left) tracks the gradient of the  global loss \eqref{glob_pop} and shows that D-GD linearly converges to stationary point of \eqref{glob_pop}, while FedAvg does not converge to one at all. Although D-GD optimizes the global loss, it does not generalize as well as FedAvg to new clients as demonstrated by Figure \ref{fig:linear} (right). Here, we fine-tune the models learned by FedAvg and D-GD on a new client with $n$ samples generated by $\mathbf{x}_{M+1}\!\sim\! \mathcal{N}(\mathbf{0},\mathbf{I}_d)$, $\zeta_{M+1}\!\sim\! \mathcal{N}(0,0.01)$, and $y_{M+1}\! =\! \langle \mathbf{B}_{\ast}\mathbf{w}_{\ast,M+1} ,\mathbf{x}_{M+1}\rangle + \zeta_{M+1}$. We fine-tune using GD for $\tau'\!=\!200$ iterations with batch size $b\!=\!n$, and plot the final error $\|\mathbf{B}_{T,M+1,\tau'}\mathbf{w}_{T,M+1,\tau'}\! -\! \mathbf{B}_\ast\mathbf{w}_{\ast,M+1}\|^2$.
Both plots are generated by averaging 10 runs.

\subsection{Image classification with neural networks} Next we evaluate  FedAvg's representation learning ability on nonlinear neural networks. For fair comparison, in every experiment all methods make the same total amount of local updates during the course of training (e.g. D-SGD is trained for $50\times$ more rounds than FedAvg with $\tau=50$). 

\textbf{Datasets and models.} We use the image classification datasets CIFAR-10 and CIFAR-100 \cite{krizhevsky2009learning}, which consist of 10 and 100 classes of RGB images, respectively.
We use a convolutional neural network (CNN) with three convolutional blocks followed by a three-layer multi-layer perceptron, with each convolutional block consisting of two convolutional layers and a max pooling layer. 

\begin{figure}[t] 
\begin{center}
\vspace{-2mm}
\centerline{\includegraphics[width=\columnwidth]{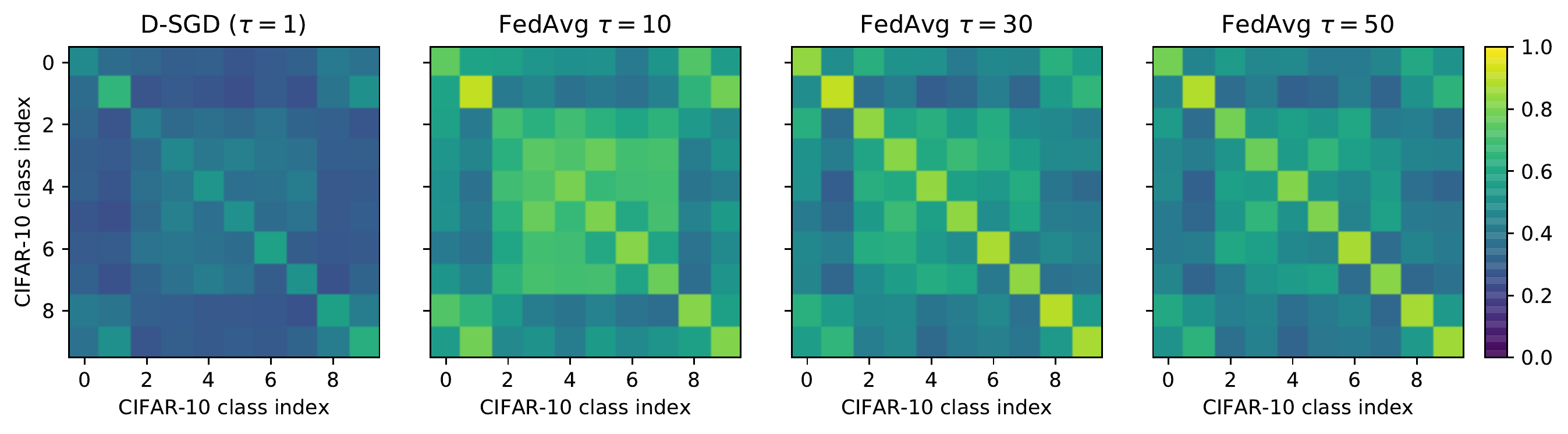}}
\vspace{-0mm}
\caption{Average cosine similarity for features learned by D-SGD and FedAvg with varying numbers of local updates on a heterogeneous partition of CIFAR-10. }
\label{fig:sim}
\end{center}
\vspace{-6mm}
\end{figure}

\textbf{Cosine similarity of features.}  A desirable property of representations for downstream classification tasks is that features of examples from the same class are similar to each other, while features of examples from different classes are dissimilar \cite{hadsell2006dimensionality}. In Figure \ref{fig:sim} we examine whether the representations learned by FedAvg satisfy this property. Here we have trained FedAvg with varying $\tau$ and D-SGD  (FedAvg with $\tau\! =\! 1$) on CIFAR-10. Image classes are heterogeneously allocated to  $M=100$ clients according to the Dirichlet distribution with parameter 0.6 as in \cite{acar2021federated}. Each subplot is a 10x10 matrix whose $(i,j)$-th element gives the average cosine similarity between features of images from the $i$-th and $j$-th classes learned by the corresponding model. Ideally, diagonal elements are close to 1 (high similarity) and off-diagonal elements are close to 0 (low similarity). 
Figure \ref{fig:sim} shows that FedAvg indeed learns features with high intra-class similarity and low inter-class similarity, with representation quality improving with more local updates between communications.
Meanwhile, D-SGD does not learn such features. The leftmost subplot shows that all of the features learned by D-SGD are dissimilar, regardless of whether two images belong to the same class.

\begin{figure}
     \centering
     \begin{subfigure}[b]{0.43\textwidth}
         \centering
         \includegraphics[width=\textwidth]{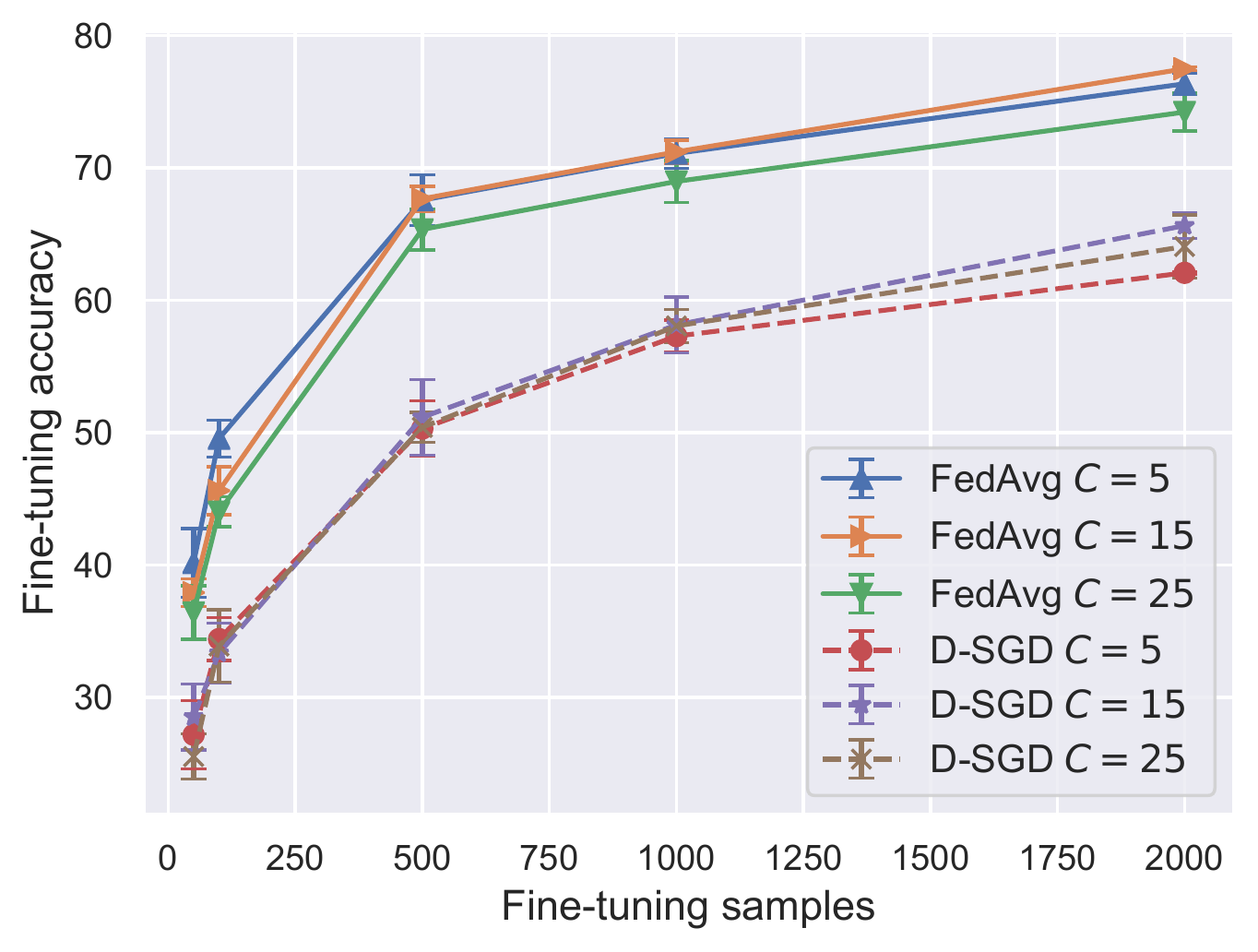}
     \end{subfigure}
     \qquad
     \begin{subfigure}[b]{0.43\textwidth}
         \centering
         \includegraphics[width=\textwidth]{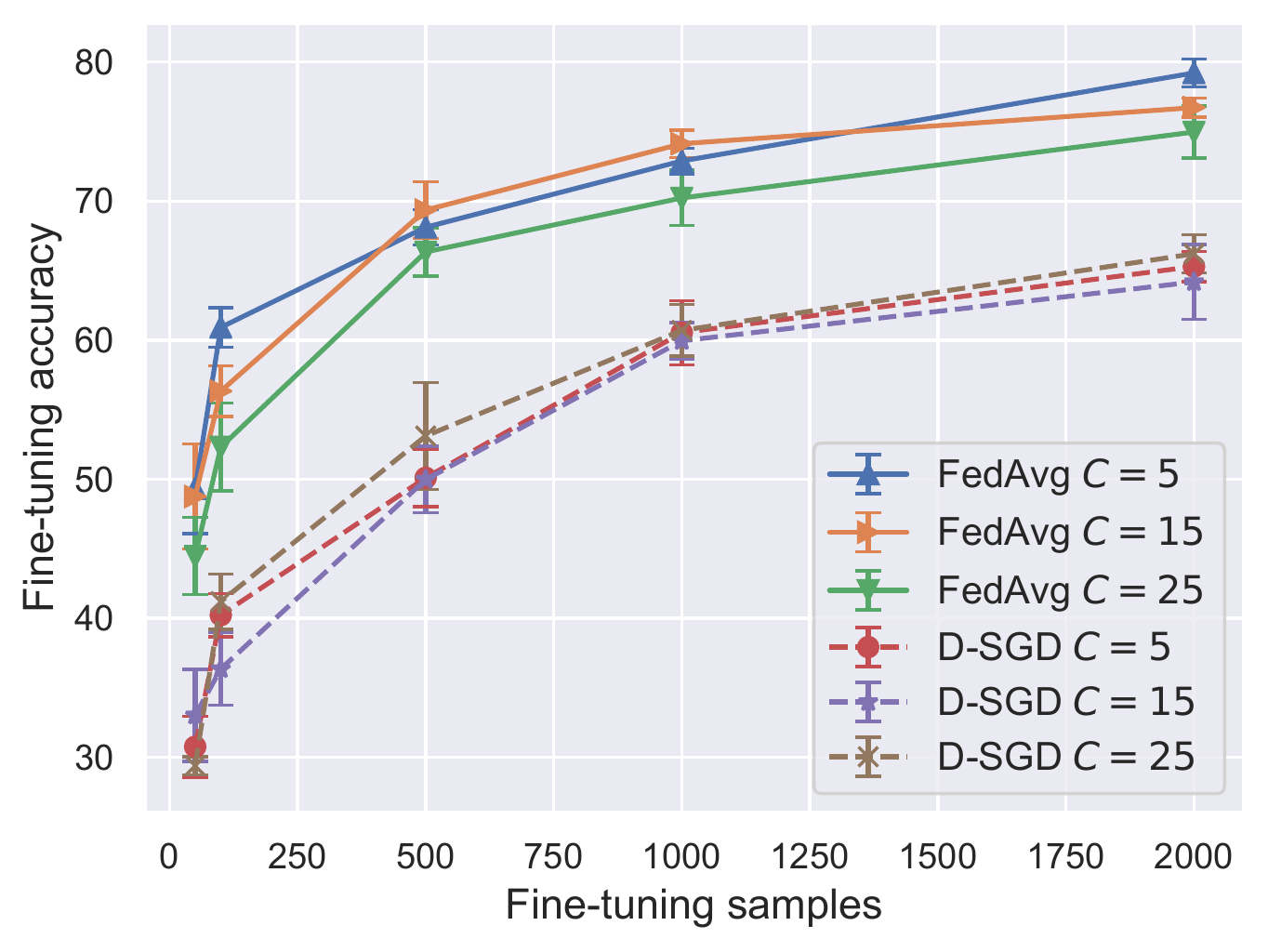}
     \end{subfigure}
        \caption{Average fine-tuning accuracies on new clients for models trained by FedAvg and D-SGD. (Left) Models trained on 80 classes from CIFAR-100 (with $C$ classes/client) and fine-tuned on new clients from 20 new classes from CIFAR-100. (Right) Models trained on CIFAR-100 with $C$ classes/client and fine-tuned on new clients from  CIFAR-10 (10 classes/client). For FedAvg, $\tau=50$ in all cases, and error bars give standard deviations over five trials with five new clients tested/trial.  
        } 
        \label{fig:CIFAR-FT}
\end{figure}

\textbf{Fine-tuning performance.}  We evaluate the generalization ability of the representations learned by FedAvg to new classes and also new datasets. An effective representation identifies universally important features, so it should generalize to new data, with perhaps a small amount of fine-tuning needed to learn a new mapping from feature space to label space. The transfer learning performance of fine-tuned models is a popular metric for evaluating the quality of learned representations \cite{whitney2020evaluating,chen2020simple}.

We first study how models trained by FedAvg and D-SGD generalize to unseen classes from the same dataset. To do so, we train models on heterogeneous partitions of CIFAR-100 using both FedAvg with $\tau = 50$ as well as D-SGD. In the left plot of Figure~\ref{fig:CIFAR-FT}, we illustrate the case that models are trained on 80 clients each with 500 total images from $C$ classes sub-selected from 80 classes of CIFAR-100, and tested on new clients with images from the remaining 20 classes of CIFAR-100. We fine-tune the trained models on the new clients with 10 epochs of SGD, with varying numbers of samples per epoch as listed on the x-axis, before testing. 
 
Next, we investigate how well models trained by FedAvg and D-SGD generalize to an unseen dataset. In the right plot of Figure~\ref{fig:CIFAR-FT}, we train models with $C$ classes/client from CIFAR-100, then test on new clients with samples drawn from CIFAR-10 (a different dataset, but with presumably  similar ``basic'' features). Specifically, for these new clients, we fine-tune for 10 epochs as previously, then test the post-fine-tuned models on the test data for each client. 
In both left and right plots, we observe that FedAvg significantly outperforms D-SGD, indicating that FedAvg has learned a representation that generalizes better to new classes. Additional details are provided in Appendix \ref{app:experiments}.

\section*{Acknowledgements}

This research of Collins, Mokhtari, and Shakkottai is supported in part by NSF Grants 2019844 and 2112471, ARO Grant W911NF2110226, ONR Grant N00014-19-1-2566, the Machine Learning Lab (MLL) at UT Austin, and the Wireless Networking and Communications Group (WNCG) Industrial Affiliates Program. The research of Hassani is supported by NSF Grants 1837253, 1943064, AFOSR Grant FA9550-20-1-0111, DCIST-CRA, and the AI Institute for Learning-Enabled Optimization at Scale (TILOS).

\newpage

\bibliography{refs}
\bibliographystyle{plain}

\newpage
\appendix

\section{Additional Related Work} \label{app:rw}


In this section we provide further discussion of the related works.

\textbf{Convergence of FedAvg.} The convergence of FedAvg, also known as Local SGD, has been the subject of intense  study in recent years, due to the algorithm's effectiveness combined with the difficulties of analyzing it. In homogeneous data settings, local updates are easier to reconcile with solving the global objective, allowing much progress to be made in understanding convergence rates in this case  \cite{stich2018local,haddadpour2019local,stich2019error,yu2019parallel,wang2019cooperative,haddadpour2019local,spiridonoff2021communication,woodworth2020local,zhou2017convergence}.
In the heterogeneous case multiple works have shown that FedAvg with fixed learning rate may not solve the global objective because the local updates induce a non-vanishing bias by drifting towards local solutions, 
even with full gradient steps and and strongly convex objectives  \cite{li2019convergence,malinovskiy2020local,pathak2020fedsplit,charles2021convergence, wang2020tackling,karimireddy2020scaffold,charles2020outsized,wang2019federated,wang2021local}. As a remedy, several papers have analyzed FedAvg with learning rate that decays over communication rounds, and have shown that this approach indeed reaches a stationary point of the global objective, but at sublinear rates \cite{khaled2020tighter,koloskova2020unified,li2019convergence,qu2020federated,karimireddy2020scaffold} that can be strictly slower than the convergence rates of D-SGD \cite{woodworth2020minibatch,karimireddy2020scaffold}.
Nevertheless, in overparameterized settings with strongly convex losses, multiple works have shown that FedAvg achieves linear convergence to a global optimum \cite{koloskova2020unified,qu2020federated}.
Here overparameterized means that the model class contains a single model that achieves zero loss for all clients. Our setting is not overparameterized in this sense, nor is it convex.

\textbf{Multi-task representation learning.} Multiple works have studied the multi-task linear representation learning setting \cite{maurer2016benefit} in recent years. 
\cite{tripuraneni2020provable} and \cite{kong2020robust} give statistical rates for a method-of-moments estimator for learning the representation and  \cite{sun2021towards} analyze a projection and eigen-weighting based algorithm designed for the case in which ground-truth representation is unknown. 
Other works have studied alternating minimization procedures for learning $\col(\mathbf{B}_\ast)$ in the context of meta-learning  \cite{thekumparampil2021statistically} and federated learning \cite{collins2021exploiting}. However, these methods require a unique head for each client, which greatly simplifies the analysis since head diversity is guaranteed prior to local updates and is not applicable to some cross-device FL settings which cannot tolerate stateful clients.
Outside of the multi-task linear regression setting, \cite{tripuraneni2020theory} and  \cite{xu2021representation} have demonstrated the necessity of task diversity to learning generalizable representations from a learning theoretic perspective, and \cite{du2020fewshot} considered the statistical rates of representation learning by solving an ERM with unique heads per task.
%




\newpage
\section{Proof of Theorem Main Results}\label{app:proof}



\subsection{Proof of Theorem \ref{thm:main_pop}}
In this section we provide the proof of Theorem \ref{thm:main_pop}. We make use of the notations in Table \ref{table:notations}. 
\begin{table}[H]
  \caption{Notations.}
  \label{table:notations}
  \centering
  \begin{tabular}{lll}
    \toprule
    Notation     & Definition   \\
    \midrule
$\mu$                                                                                      & $ \sigma_{\min}^{0.5}\left(\tfrac{1}{M}\sum_{i=1}^M (\mathbf{w}_{\ast,i} - \mathbf{\bar{w}}_{\ast})(\mathbf{w}_{\ast,i} - \mathbf{\bar{w}}_{\ast})^\top\right)$  \\
$L_{\max}$                                                                                      & $\max_{i\in[M]}\|\mathbf{w}_{\ast,i}\|_2 $ \\ 
$\kappa_{\max}$                                                                                      & $\nicefrac{L_{\max}}{\mu}$  \\ 
$\mathbf{\bar{w}}_{\ast} $ & $\tfrac{1}{M}\sum_{i=1}^M \mathbf{w}_{\ast,i}$ \\
$\mathbf{\bar{w}}_{\ast,t}$ & $ \tfrac{1}{m}\sum_{i\in \mathcal{I}_t} \mathbf{w}_{\ast,i}$ \\
$ \mathbf{B}_{t,i,s},\mathbf{w}_{t,i,s}$ & the results of $s$ local updates of the global model at round $t$ by the $i$-th client\\
$ \mathbf{B}_{t,i,0},\mathbf{w}_{t,i,0}$ & $ \mathbf{B}_{t},\mathbf{w}_{t}$, respectively  \\
$\mathbf{e}_{t,i,s}$ & $\mathbf{B}_{t,i,s}\mathbf{w}_{t,i,s} - \mathbf{B}_{\ast}\mathbf{w}_{\ast,i} $, i.e. product error for $s$-th local update for task $i$, round $t$ \\
$ \mathbf{G}_{t,i,s}$ & $(\mathbf{B}_{t,i,s+1} - \mathbf{B}_{t,i,s})/\alpha$, such that $\mathbf{B}_{t,i,s+1} = \mathbf{B}_{t,i,s} -\alpha \mathbf{G}_{t,i,s}$ \\
$ \mathbf{G}_{t}$ & $(\mathbf{B}_{t+1} - \mathbf{B}_{t})/\alpha$, such that $\mathbf{B}_{t+1} = \mathbf{B}_{t} -\alpha \mathbf{G}_{t}$ \\
$\del_t$                                                    & $\mathbf{I}_k - \alpha \mathbf{B}_t^\top \mathbf{B}_t$                        \\ 
$\deld_t$                & $\mathbf{I}_d - \alpha \mathbf{B}_t \mathbf{B}_t^\top$ \\
$\dist_t$ & $\dist(\mathbf{B}_t,\mathbf{B}_\ast)$ \\
$\delta_0$ & $\dist_0$ \\
$E_0$ & $1-\dist_0^2$ \\
$\col(\mathbf{B}), \col(\mathbf{B})^\perp$ & column space of $\mathbf{B}$, orthogonal complement to column space of $\mathbf{B}$, respectively \\
    \bottomrule
  \end{tabular}
\end{table}
Here the local updates are given by
\begin{align}
    \mathbf{B}_{t,i,s+1} &= \mathbf{B}_{t,i,s} - \alpha (\mathbf{B}_{t,i,s}\mathbf{w}_{t,i,s} - \mathbf{B}_\ast \mathbf{w}_{\ast,i})\mathbf{w}_{t,i,s}^\top \nonumber \\
    \mathbf{w}_{t,i,s+1} &= \mathbf{w}_{t,i,s} - \alpha \mathbf{B}_{t,i,s}^\top (\mathbf{B}_{t,i,s} \mathbf{w}_{t,i,s} - \mathbf{B}_{\ast} \mathbf{w}_{\ast,i}) \nonumber \\
    &= \del_{t,i,s}\mathbf{w}_{t,i,s} + \alpha \mathbf{B}_{t,i,s}^\top \mathbf{B}_{\ast} \mathbf{w}_{\ast,i} \nonumber
\end{align}
and the global updates are given by 
\begin{align}
    \mathbf{B}_{t+1} &= \tfrac{1}{m}\sum_{i \in \mathcal{I}_t} \mathbf{B}_{t,i,\tau} \nonumber \\
    \mathbf{w}_{t+1} &= \tfrac{1}{m}\sum_{i \in \mathcal{I}_t} \mathbf{w}_{t,i,\tau}. \nonumber
\end{align}

First we control the ground-truth heads sampled on each round.
\begin{lemma}\label{lem:concen}
Suppose $m\geq \min(M, 20((\nicefrac{\gamma}{L_{\max}})^2 + (\nicefrac{H}{L_{\max}})^4)(\alpha L_{\max})^{-4}\log(kT))$. Then the event
\begin{align}
  A_0\coloneqq \Bigg\{&  \bigg\| \tfrac{1}{m}\sum_{i\in \mathcal{I}_t} \mathbf{w}_{\ast,i} - \mathbf{\bar{w}}_{\ast} \bigg\| \leq 4 \alpha^2 L_{\max}^3,\nonumber \\
  & \bigg\| \tfrac{1}{m}\sum_{i\in \mathcal{I}_t} \mathbf{w}_{\ast,i}\mathbf{w}_{\ast,i}^\top - \tfrac{1}{M}\sum_{i'=1}^M \mathbf{w}_{\ast,i'}\mathbf{w}_{\ast,i'}^\top  \bigg\| \leq 4\alpha^2 L_{\max}^4 \quad \forall t \in [T]\Bigg\} \nonumber
\end{align}
with probability at least $ 1- 4 (kT)^{-99}$.
\end{lemma}

\begin{proof}
If $m=M$ then $A_0$ holds almost surely. Otherwise, first let $\mathbf{W}_{\ast,i} := \text{diag}(\mathbf{w}_{\ast,i}) \in \mathbf{R}^{k\times k}$, and let $\mathbf{\bar{W}}_{\ast} := \tfrac{1}{M}\sum_{i=1}^M\text{diag}(\mathbf{w}_{\ast,i})$. For any $t\in [T]$, $\{\mathbf{W}_{\ast,i}\}_{i \in\mathcal{I}_t}$ is a set of Hermitian matrices sampled uniformly without replacement from $\{\mathbf{W}_{\ast,i}\}_{i \in[M]}$, $\|\mathbf{W}_{\ast,i}\|\leq L_{\max}$ almost surely, and $\left\|\tfrac{1}{M}\sum_{i=1}^M (\mathbf{W}_{\ast,i}-\mathbf{\bar{W}}_{\ast})^2 \right\| \leq \tfrac{1}{M}\sum_{i=1}^M \|\mathbf{w}_{\ast,i}-\mathbf{\bar{w}}_{\ast}\|^2 =\gamma^2$ by the triangle and Cauchy-Schwarz inequalities and  Definition \ref{def:var}. Thus, we can apply Theorem 1 in \cite{gross2010note} to obtain
\begin{align}
 \mathbb{P}\left( \left\|\sum_{i \in \mathcal{I}_t} \mathbf{W}_{\ast,i} - \mathbf{\bar{W}}_{\ast}\right\|> t \right) &\leq 2 k \exp( \tfrac{-t^2}{4m\gamma^2} )
\end{align}
as long as $t\leq 2 m \gamma^2/L_{\max}$. Choose $t = 4 m \alpha^2 L_{\max}^3$. Note that indeed $t \leq 2 m \gamma^2/L_{\max}$ since $\gamma^2 = \sum_{l=1}^k\tfrac{1}{M}\sum_{i=1}^M \mathbf{u}_l^\top (\mathbf{w}_{\ast,i}- \mathbf{\bar{w}}_{\ast})(\mathbf{w}_{\ast,i}- \mathbf{\bar{w}}_{\ast})^\top \mathbf{u}_l\geq k\mu^2$, where $\mathbf{u}_l$ is the $l$-th standard basis vector, and $\alpha \leq \mu^2/(2L_{\max}^4)$. Thus we obtain
\begin{align}
    \mathbb{P}\left( \left\|\sum_{i \in \mathcal{I}_t} \mathbf{W}_{\ast,i} - \mathbf{\bar{W}}_{\ast}\right\|> 4m \alpha^2 L_{\max}^3 \right) &\leq 2 k \exp\left( \tfrac{- 4m \alpha^4 L_{\max}^6 }{ \gamma^2} \right) \nonumber \\
    \implies  \mathbb{P}\left( \left\|\tfrac{1}{m}\sum_{i \in \mathcal{I}_t} \mathbf{W}_{\ast,i} - \mathbf{\bar{W}}_{\ast}\right\|>  4\alpha^2 L_{\max}^3 \right) &\leq 2 k \exp\left( - 100 \log(kT) \right) \nonumber 
\end{align}
since $m \geq 20 (\tfrac{L_{\max}^2}{\gamma^2})\alpha^4 L_{\max}^4 \log(kT)$.
An analogous argument, without needing to lift the matrices to higher dimensions, yields
\begin{align}
\mathbb{P}\left( \left\|\tfrac{1}{m}\sum_{i \in \mathcal{I}_t}\left( \mathbf{w}_{\ast,i}\mathbf{w}_{\ast,i}^\top - \tfrac{1}{M}\sum_{i'=1}^M\mathbf{{w}}_{\ast,i'}\mathbf{{w}}_{\ast,i'}^\top\right) \right\|> 4 \alpha^2 L_{\max}^4 \right) &\leq 2 k \exp\left( - 100 \log(kT) \right) \nonumber 
\end{align}
Union bounding, we obtain that $ \left\|\tfrac{1}{m}\sum_{i \in \mathcal{I}_t} \mathbf{W}_{\ast,i} - \mathbf{\bar{W}}_{\ast}\right\|\leq  4\alpha^2 L_{\max}^3 $ and \\
$ \left\|\tfrac{1}{m}\sum_{i \in \mathcal{I}_t}\left( \mathbf{w}_{\ast,i}\mathbf{w}_{\ast,i}^\top - \tfrac{1}{M}\sum_{i'=1}^M\mathbf{{w}}_{\ast,i'}\mathbf{{w}}_{\ast,i'}^\top\right) \right\|\leq 4 \alpha^2 L_{\max}^3 $ with probability at least \\
$1 - 4k\exp(-100\log(kT)) = 1 - 4k^{-99}T^{-100}$.
Union bounding over all $t\in [T]$ completes the proof. 
\end{proof}

Next we state and prove the version of Theorem \ref{thm:main_pop} with explicit constants. Note that the constants are not optimized.

\begin{theorem}[FedAvg Representation Learning]
\label{thm:app_pop} Consider the case that each client takes gradient steps with respect to their population loss $f_i(\mathbf{B},\mathbf{w})\coloneqq \tfrac{1}{2}\|\mathbf{Bw}-\mathbf{B}_{\ast}\mathbf{w}_{\ast,i}\|^2$ and all losses are weighted equally in the global objective. 
Suppose Assumptions  \ref{assump:n} and \ref{assump:td}  hold, the number of clients participating each round satisfies $m \geq \min(M, 20((\nicefrac{\gamma}{L_{\max}})^2 + (\nicefrac{H}{L_{\max}})^4)(\alpha L_{\max})^{-4}\log(kT))$, and the initial parameters satisfy (i)  $ \delta_0 \coloneqq \dist(\mathbf{B}_0, \mathbf{B}_\ast)\leq \sqrt{1\! -\! E_0}$ for any $E_0 \in (0,1]$, (ii) $\|\mathbf{I}- \alpha \mathbf{B}_0^\top \mathbf{B}_0\|_2 \leq \alpha^2 \tau L_{\max}^2 \kappa_{\max}^2$ and (iii) $\|\mathbf{w}_0\|_2 \leq \alpha^{2.5} \tau L_{\max}^3$.
Choose  step size $\alpha \leq  \tfrac{1-\delta_0}{4800{  \sqrt{\tau}L_{\max}\kappa_{\max}^{2} }}$. 
    Then for any $\epsilon \in (0, 1)$, the distance of the representation learned by FedAvg with $\tau\geq 2$ local updates satisfies  $\dist(\mathbf{B}_T, \mathbf{B}_\ast) <\epsilon$ after at most
\begin{align}
T \leq \tfrac{25}{\alpha^2 \tau \mu^2 E_0}\log(\nicefrac{1}{\epsilon}) \nonumber
\end{align}
communication rounds with probability at least $1 - 4 (kT)^{-99}$. 
\end{theorem}



\begin{proof}
First we condition on the event $A_0$, which occurs with probability at least $1 - 4(kT)^{-99}$ by Lemma \ref{lem:concen}. Conditioned on this event, we will show that that the following two sets of  inductive hypotheses hold for all $s\in [\tau]$, $i\in \mathcal{I}_t$, and $t\in[T]$. The first set of inductive hypotheses controls  local behavior. We apply the below local induction  in parallel for each client $i \in [M]$ at every communication round $t \geq 0$, starting from the base case $s=1$.
\begin{enumerate}
    \item $A_{1,t,i}(s) \coloneqq \{\|\mathbf{w}_{t,i,s'}- \alpha \mathbf{B}_{t,i,s'-1}^\top \mathbf{B}_{\ast}\mathbf{w}_{\ast,i}\|_2 \leq 4 c_3 \alpha^{2.5} \tau  L_{\max}^3  \kappa_{\max}^2 E_0^{-1} \quad \forall s'\in \{1,\dots,s\}\}$
    
    \item $A_{2,t,i}(s) \coloneqq \{\|\mathbf{w}_{t,i,s'}\|_2 \leq 2 \alpha^{0.5} L_{\max} \quad \forall s'\in \{1,\dots,s\}  \}$ 
    
     \item $A_{3,t,i}(s) \coloneqq \{\|\del_{t,i,s'}\|_2 \leq 2c_3 \alpha^2 \tau L_{\max}^2 \kappa_{\max}^2 E_0^{-1} \quad \forall s'\in \{1,\dots,s\}  \}$ 
    

        
        \item $A_{4,t,i}(s)\coloneqq \{ \dist(\mathbf{B}_{t,i,s'}, \mathbf{B}_\ast) \leq 1.1  \dist(\mathbf{B}_{t}, \mathbf{B}_\ast)\quad \forall s'\in \{1,\dots,s\}\}$
        
    \end{enumerate}
The second set of inductions controls the global behavior, starting from $t=1$ as the base case:
\begin{enumerate}
\item $A_1(t) \coloneqq \{ \|\mathbf{w}_{t'} - \alpha(\mathbf{I}_k + \del_{t'}) \mathbf{B}_{t'}^\top \mathbf{B}_\ast\mathbf{\bar{w}}_{\ast,t'}\|_2 \leq   91\alpha^{2.5}\tau L_{\max}^3  
\quad \forall t'\in \{1,\dots, t\}\}$ 
    \item $A_2(t) \coloneqq \{ \|\mathbf{w}_{t'}\|_2 \leq  2\alpha^{0.5} L_{\max}  \quad \forall t'\in \{1,\dots, t\}\}$
    \item $A_3(t) \coloneqq \{ \|\del_{t'} \|_2 \leq  c_3 \alpha^2 \tau  L_{\max}^2  \kappa_{\max}^2 E_0^{-1} \quad \forall t'\in \{1,\dots, t\}\}$
    \item $A_4(t) \coloneqq \{ \|\mathbf{B}_{\ast,\perp}^\top \mathbf{B}_{t'}\|_2 \leq  (1 - 0.04 \alpha^2 \tau \mu^2 E_0)\|\mathbf{B}_{\ast,\perp}^\top \mathbf{B}_{t'-1} \|_2 \quad \forall t'\in \{1,\dots, t\} \}$
    \item $A_5(t) \coloneqq \{ \dist_{t} \leq  (1 - 0.04 \alpha^2 \tau \mu^2 E_0)^{t-1} \quad \forall t'\in \{1,\dots, t\}  \}$
\end{enumerate}
where  $c_3 = 4800$.
Without loss of generality let $\alpha \leq \frac{1-\delta_0}{{c_3} \sqrt{\tau}L_{\max}\kappa_{\max}^2}$.  For ease of presentation we refer to $c_3$ symbolically rather than by its value throughout the proof.

The above inductions are applied in the following manner. First, the global initialization at $t=0$ implies that the local inductive hypotheses hold after one local update (the base case). Then, by the local inductive argument, these conditions continue to hold for all subsequent local updates. This in turn implies that the global inductive hypotheses hold after the first global averaging step, i.e. $A_1(1), A_2(1)$ and $A_3(1)$ hold.  Next, the global hypotheses holding at $t=1$ implies that the local inductions hold in their base case at $t=1$ (after one local update, i.e. $s=1$), which implies they continue to hold for all subsequent local updates. Again, this implies the global hypotheses hold at $t=2$, which implies the base case for the local inductions at $t=2$, and so on. In summary, the ordering of the inductions is:
\begin{align}
\text{Initialization at $t\!=\!0$}\! \implies\! & \text{Local inductions at $t\!=\!0$} \! \implies \! \text{Global inductions at $t\!=\!1$} \! \nonumber \\
&\implies\!\text{Local inductions at $t\!=\!1$}\!  \implies\!  \dots \nonumber
\end{align}

We start by showing that the base case $s=1$ holds for the local inductions. The proof is identical for all $i\in [M]$ and $t\geq 0$.

\begin{itemize}
    \item \textbf{If $t=0$: initial conditions $\implies$ $A_{1,t,i}(1)$, else $A_2(t) \cap A_3(t) \implies A_{1,t,i}(1)$.} 

Note that at initialization, $\|\del_0\mathbf{w}_0\|\leq \|\del_0\|\|\mathbf{w}_0\| \leq \alpha^{2.5}\tau L_{\max}^3 \kappa_{\max}^2 \leq 4c_3 \alpha^{2.5}\tau L_{\max}^3 \kappa_{\max}^2 E_0^{-1}$. Likewise, at arbitrary $t$, $\|\del_t\mathbf{w}_t\|\leq 4c_3 \alpha^{2.5}\tau L_{\max}^3 \kappa_{\max}^2 E_0^{-1} $ due to $A_2(t)$ and $A_3(t)$. Thus, since 
$\mathbf{w}_{t,i,1} = \del_t \mathbf{w}_t + \alpha \mathbf{B}_t^\top \mathbf{B}_\ast \mathbf{w}_{\ast,i}$, we have
    \begin{align}
       \| \mathbf{w}_{t,i,1} - \alpha \mathbf{B}_t^\top \mathbf{B}_\ast \mathbf{w}_{\ast,i} \| &= \|\del_t \mathbf{w}_t \|
        \leq \|\del_t \| \| \mathbf{w}_t \| \leq  4 c_3 \alpha^{2.5} \tau L_{\max}^3 \kappa_{\max}^2 E_0^{-1} \nonumber
    \end{align}
    as desired (recall that $\mathbf{B}_{t,i,0}\equiv \mathbf{B}_t$).
    
    \item \textbf{If $t=0$: initial conditions $ \cap \; A_{1,t,i}(1) \implies$ $A_{2,t,i}(1)$, else $ A_3(t) \cap A_{1,t,i}(1) \implies A_{2,t,i}(1)$.} 
    
    For any $t \geq 0$, we have $\|\del_t\| \leq c_3 \alpha^2 \tau L_{\max}^2 \kappa_{\max}^2 E_0^{-1}$ due to either the initialization $(t=0)$ or $A_3(t)$ $(t> 0)$. This implies that $\|\mathbf{B}_t\| \leq \sqrt{\tfrac{1 + c_3 \alpha^2 \tau L_{\max}^2 \kappa_{\max}^2 E_0^{-1}}{\alpha}} \leq \tfrac{1.1}{\sqrt{\alpha}}$ since $\alpha$ is sufficiently small (noting that $\tfrac{(1-\delta)^2}{E_0} = \tfrac{(1-\delta)^2}{1-\delta^2}\leq 1$ . Now we use $A_{1,t,i}(1)$ and the triangle inequality to obtain:
    \begin{align}
        \|\mathbf{w}_{t,i,1}\| &\leq \| \mathbf{w}_{t,i,1} - \alpha \mathbf{B}_t^\top \mathbf{B}_\ast \mathbf{w}_{\ast,i} \| + \| \alpha \mathbf{B}_t^\top \mathbf{B}_\ast \mathbf{w}_{\ast,i} \|\nonumber \\
        &\leq 4c_3 \alpha^{2.5} \tau L_{\max}^3 \kappa_{\max}^2 E_0^{-1} + \alpha \|\mathbf{B}_t\| \|\mathbf{w}_{\ast,i}\|\nonumber \\
        &\leq {2\sqrt{\alpha}}L_{\max}. \nonumber
    \end{align}
    as desired.  
    
\item \textbf{If $t=0$: initial conditions $\implies$ $A_{3,t,i}(1)$, else $A_2(t) \cap A_3(t) \implies A_{3,t,i}(1)$.}  

We have
\begin{align}
     \del_{t,i,1}  &=  \mathbf{I}_k - \alpha \mathbf{B}_{t,i,1}^\top \mathbf{B}_{t,i,1} \nonumber \\
    &= \del_t + \alpha^2 \mathbf{B}_t^\top( \mathbf{B}_{t} \mathbf{w}_t - \mathbf{B}_{\ast} \mathbf{w}_{\ast,i}) \mathbf{w}_t^\top \nonumber \\
    &\quad + \alpha^2 \mathbf{B}_t^\top( \mathbf{B}_{t} \mathbf{w}_t - \mathbf{B}_{\ast} \mathbf{w}_{\ast,i}) \mathbf{w}_t^\top - \alpha^3 \mathbf{w}_t  \mathbf{w}_t^\top \|\mathbf{B}_{t} \mathbf{w}_t - \mathbf{B}_{\ast} \mathbf{w}_{\ast,i}\|_2^2 \label{12}
\end{align}
By the initial conditions and by inductive hypotheses $A_1(t)$ and $A_2(t)$, for any $t\geq 0$  we have $\|\mathbf{w}_t\|_2\leq 2\sqrt{\alpha}L_{\max}$, $\|\del_t\|\leq c_3 \alpha^2 \tau L_{\max}^2 \kappa_{\max}^2E_0^{-1}$, and $ \|\mathbf{B}_t\|_2\leq \tfrac{1.1}{\sqrt{\alpha}}$. This implies
$\|\mathbf{B}_t \mathbf{w}_{t} - \mathbf{B}_\ast \mathbf{w}_{\ast,i}\|_2\leq \|\mathbf{B}_t\| \|\mathbf{w}_{t}\| +\| \mathbf{B}_\ast \mathbf{w}_{\ast,i}\|\leq 3.2 L_{\max}$.
Therefore using \eqref{12}, we obtain
\begin{align}
     \|\del_{t,i,1}\| _2
    &\leq \|\del_t\|_2 + 2 \alpha^2 \|\mathbf{B}_t^\top( \mathbf{B}_{t} \mathbf{w}_t - \mathbf{B}_{\ast} \mathbf{w}_{\ast,i}) \mathbf{w}_t^\top\|_2 + \alpha^3 \|\mathbf{w}_t \|^2_2 \|\mathbf{B}_{t} \mathbf{w}_t - \mathbf{B}_{\ast} \mathbf{w}_{\ast,i}\|_2^2  \nonumber \\
    &\leq \|\del_t\|_2 + 15 \alpha^2 L_{\max}^2 + 41\alpha^4 L_{\max}^4  \nonumber \\
    &\leq 2 c_3 \alpha^2 \tau L_{\max}^2 \kappa_{\max}^2 E_0^{-1}
\end{align}
as desired.

\item \textbf{If $t=0$: initial conditions $\implies$ $A_{4,t,i}(1)$, else $A_{3,t,i}(1)\cap A_2(t) \cap A_3(t) \implies A_{4,t,i}(1)$.}  

Note that 
\begin{align}
   \|\mathbf{B}_{\ast,\perp}^\top \mathbf{B}_{t,i,1}\|_2\! &=  \|\mathbf{B}_{\ast,\perp}^\top\mathbf{B}_{t} (\mathbf{I}_k\! -\! \alpha \mathbf{w}_{t}\mathbf{w}_{t}^\top) \|_2 \leq\|\mathbf{B}_{\ast,\perp}^\top\mathbf{B}_{t}\| \|\mathbf{I}_k\! -\! \alpha \mathbf{w}_{t}\mathbf{w}_{t}^\top \|_2 \leq \|\mathbf{B}_{\ast,\perp}^\top\mathbf{B}_{t} \|_2 \nonumber 
\end{align}
as $\alpha\|\mathbf{w}_t\|^2 \leq 1$ by either the initialization (if $t=0$) or $A_2(t)$ (if $t\geq 1$) and the choice of $\alpha$ sufficiently small.
Thus, letting $\mathbf{\hat{B}}_{t,i,1}\mathbf{R}_{t,i,1} = \mathbf{B}_{t,i,1}$ denote the QR-decomposition of $\mathbf{B}_{t,i,1}$, we have
\begin{align}
    \dist(\mathbf{B}_{t,i,1},\mathbf{B}_\ast) &= \|\mathbf{B}_{\ast,\perp}^\top \mathbf{\hat{B}}_{t,i,1}\|_2 \nonumber \\
                                            &\leq \tfrac{1}{\sigma_{\min}(\mathbf{B}_{t,i,1})}\|\mathbf{B}_{\ast,\perp}^\top \mathbf{{B}}_{t,i,1}\|_2 \nonumber \\
                                            &\leq \tfrac{1}{\sigma_{\min}(\mathbf{B}_{t,i,1})}\|\mathbf{B}_{\ast,\perp}^\top \mathbf{{B}}_{t}\|_2 \nonumber \\
                                            &\leq \tfrac{\|\mathbf{B}_t\|}{\sigma_{\min}(\mathbf{B}_{t,i,1})}\dist(\mathbf{B}_t, \mathbf{B}_{\ast}) \nonumber \\
                                            &\leq \sqrt{\tfrac{1+ c_3 \alpha^2 \tau L_{\max}^2 \kappa_{\max}^2 E_0^{-1}}{1 - 2c_3 \alpha^2 \tau L_{\max}^2 \kappa_{\max}^2 E_0^{-1}}} \dist_{t}  \label{jgj} \\
   &\leq 1.1 \dist(\mathbf{B}_t, \mathbf{B}_{\ast})\label{gbg}
\end{align}
where the \eqref{jgj} follows by  $A_{3,t,i}(1)$ and either the initial condition on $\|\del_t\|$ (if $t=0$) or $A_3(t)$ (if $t>0$), and \eqref{gbg} follows as $\alpha$ is sufficiently small.
\end{itemize}

Now we show that the global inductions hold at global round $t=1$ following the local updates at round $t=0$.
\begin{itemize}
    \item \textbf{Initialization $\cap  \big(\cap_{i\in\mathcal{I}_0}A_{1,0,i}(\tau)\cap A_{2,0,i}(\tau)\cap A_{3,0,i}(\tau)\big) \implies A_1(1)\cap A_2(1)\cap A_4(1)\cap A_5(1)$.}
    
    To show each of these hypotheses hold we can apply the proofs of Lemmas \ref{lem:a1}, \ref{lem:1.5} \ref{lem:a2} and \ref{lem:a4} respectively, since they only rely on inductive hypotheses $ A_{1,0,i}(\tau)$, $ A_{2,0,i}(\tau)$  and $ A_{3,0,i}(\tau)$ and appropriate scaling of $\|\mathbf{B}_0\|$ and $\|\mathbf{w}_0\|$, which is guaranteed by the initialization. In particular, the proof of these inductive hypotheses is identical for all $t\geq 1$.
    \item \textbf{Initialization  $ \cap  \big(\cap_{i\in\mathcal{I}_0} A_{2,0,i}(\tau)\cap A_{3,0,i}(\tau)\big) \implies A_2(1)$.} 
    
    In the proof of $A_2(t)$ for $t\geq 2$ (Lemma \ref{lem:a2}) we leverage the fact that $\mathbf{w}_{t-1}$ is close to a matrix times the average of the $\mathbf{\bar{w}}_{\ast,t-1}$. Our initialization cannot guarantee that this holds for $\mathbf{w}_0$. Instead, we show that $\|\del_1\|$ may increase from $\|\del_0\|$  at a large rate that would cause $\|\del_t\|$ to blow up if continued indefinitely, but since it only grows at this rate for the first round, this is ok. In particular, let $\mathbf{G}_0 = \tfrac{1}{\alpha}(\mathbf{B}_0 - \mathbf{B}_1)$ such that $\mathbf{B}_1 = \mathbf{B}_0 - \alpha \mathbf{G}_0$. Then
    \begin{align}
        \del_{1} &= \del_0 + \alpha^2 \mathbf{B}_0^\top \mathbf{G}_0 + \alpha^2 \mathbf{G}_0^\top \mathbf{B}_0 - \alpha^3 \mathbf{G}_0^\top \mathbf{G}_0  \nonumber 
    \end{align}
    Moreover,
    \begin{align}
     \|   \mathbf{G}_0 \| &= \left\|\frac{1}{m}\sum_{i\in\mathcal{I}_0} \sum_{s=0}^{\tau-1} ( \mathbf{B}_{0,i,s} \mathbf{w}_{0,i,s} - \mathbf{B}_{\ast}\mathbf{w}_{\ast,i})\mathbf{w}_{0,i,s}^\top \right\|
     \leq 7 \sqrt{\alpha} \tau L_{\max}^2 \nonumber 
    \end{align}
    by the initialization and $A_{2,0,i}(\tau)$ and $A_{3,0,i}(\tau)$, thus
    \begin{align}
        \|  \mathbf{B}_0^\top \mathbf{G}_0 \| &\leq 8 \tau L_{\max}^2, \quad \quad \|  \mathbf{G}_0^\top \mathbf{G}_0 \| \leq 49 \alpha \tau^2 L_{\max}^4 \nonumber
    \end{align}
    which implies that $\|\del_1\|\leq \|\del_0\| + 10 \alpha^2 \tau L_{\max}^2 \leq c_3 \alpha^2 \tau L_{\max}^2 \kappa_{\max}^2 E_0^{-1}$, as desired.
\end{itemize}

Assume that the inductive hypotheses hold up to time $t$ and local round $s\geq 1$. We first show that the local inductive hypotheses hold for local round $s+1$. Then, we show that the global inductions hold at time $t+1$. This is achieved by the following lemmas.

{\textbf{Local inductions.}}
\begin{itemize}
    \item $A_{2,t,i}(s) \cap A_{3,t,i}(s) \implies A_{1,t,i}(s+1).  $ This is Lemma \ref{lem:a1s}.
    \item $A_{1,t,i}(s+1) \cap A_{2,t,i}(s)  \cap A_{3,t,i}(s) \implies A_{2,t,i}(s+1).  $ This is Lemma \ref{lem:a1.5s}.
    \item $ A_{2,t,i}(s) \cap A_{3,t,i}(s) \implies A_{3,t,i}(s+1). $  This is Lemma \ref{lem:a2s}. 
    \item $A_{2,t,i}(s) \cap A_{3,t,i}(s+1) \cap A_{3}(t) \implies A_{4,t,i}(s+1)$. This is Lemma \ref{lem:a3s}. 
\end{itemize}
{\textbf{Global inductions.}}
\begin{itemize}
    \item $ \cap_{i\in \mathcal{I}_t} \big(A_{1,t,i}(\tau-1) \cap A_{3,t,i}(\tau-1)\big) \cap A_1(t) \cap A_2(t)\cap A_3(t) \implies A_{1}(t+1).  $ This is Lemma \ref{lem:a1}.
    \item $ \cap_{i\in \mathcal{I}_t}  A_{2,t,i}(\tau) \implies A_{2}(t+1).  $ This is Lemma \ref{lem:1.5}.
    \item $ \cap_{i\in \mathcal{I}_t} \big(\cap_{h=1}^4 A_{h,t,i}(\tau) \big)  \cap A_1(t) \cap A_2(t) \cap A_3(t) \cap A_5(t)\implies A_3(t+1).$  This is Lemma \ref{lem:a2}. 
    \item $ \cap_{i\in \mathcal{I}_t} \big( A_{1,t,i}(\tau) \cap A_{2,t,i}(\tau) \cap A_{3,t,i}(\tau) \big) \cap A_2(t) \cap A_3(t)  \implies A_{4}(t+1). $  This is Lemma \ref{lem:a3}. 
    \item $ A_3(t+1) \cap A_4(t+1) \cap A_5(t)  \implies A_{5}(t+1).$  This is Lemma \ref{lem:a4}. 
\end{itemize}
These inductions complete the proof.
\end{proof}




\begin{lemma}$A_{2,t,i}(s) \cap A_{3,t,i}(s) \implies A_{1,t,i}(s+1)$. \label{lem:a1s} 
\end{lemma}

\begin{proof}
Since $\mathbf{w}_{t,i,s+1} = \del_{t,i,s}\mathbf{w}_{t,i,s} + \alpha \mathbf{B}_{t,i,s}^\top\mathbf{B}_\ast\mathbf{w}_{\ast,i}$,  we have
 \begin{align}
    \| \mathbf{w}_{t,i,s+1} -  \alpha \mathbf{B}_{t,i,s}^\top\mathbf{B}_\ast\mathbf{w}_{\ast,i} \| &=\|\del_{t,i,s}\mathbf{w}_{t,i,s} \|_2 \nonumber \\
    &\leq \|\del_{t,i,s}\|\|\mathbf{w}_{t,i,s} \|_2 \nonumber \\
     &\leq  4 c_3 \alpha^{2.5} \tau L_{\max}^3 \kappa_{\max}^2  
 \end{align}
where the last inequality follows by  $A_{2,t,i}(s)$ and $A_{3,t,i}(s)$.
\end{proof}

\begin{lemma}$A_{1,t,i}(s+1) \cap A_{3,t,i}(s) \implies A_{2,t,i}(s+1)$. \label{lem:a1.5s} 

\end{lemma}

\begin{proof}
Note that by the triangle inequality,
\begin{align}
    \| \mathbf{w}_{t,i,s+1}\| &\leq  \| \mathbf{w}_{t,i,s+1}-  \alpha \mathbf{B}_{t,i,s}^\top\mathbf{B}_\ast\mathbf{w}_{\ast,i} \| + \|\alpha \mathbf{B}_{t,i,s}^\top\mathbf{B}_\ast\mathbf{w}_{\ast,i} \| \nonumber\\
    &\leq  4 c_3 \alpha^{2.5} \tau L_{\max}^3 \kappa_{\max}^2 E_0^{-1} + \|\alpha \mathbf{B}_{t,i,s}^\top\mathbf{B}_\ast\mathbf{w}_{\ast,i} \| \label{vf} \\
    &\leq  4 c_3 \alpha^{2.5} \tau L_{\max}^3 \kappa_{\max}^2 E_0^{-1} + 1.1 \sqrt{\alpha} L_{\max} \label{vf1} \\
     &\leq  2 \sqrt{\alpha} L_{\max} \nonumber
\end{align}
where \eqref{vf} follows by $A_{1,t,i}(s)$ and \eqref{vf1} follows by the fact that $\|\mathbf{B}_{t,i,s}\|\leq \tfrac{1.1}{\sqrt{\alpha}}$ by $A_{3,t,i}(s)$, and choice of $\alpha \leq (1-\delta_0)(c_3\sqrt{ \tau } L_{\max}\kappa_{\max}^{2})^{-1}$.
\end{proof}

\begin{lemma}$ A_{2,t,i}(s) \cap A_{3,t,i}(s) \implies A_{3,t,i}(s+1)$. \label{lem:a2s}

\end{lemma}

\begin{proof}
    Let $\mathbf{e}_{t,i,s}\coloneqq \mathbf{B}_{t,i,s}\mathbf{w}_{t,i,s} - \mathbf{B}_\ast \mathbf{w}_{\ast,i}$ and  $\mathbf{G}_{t,i,s} \coloneqq \mathbf{e}_{t,i,s} \mathbf{w}_{t,i,s}$. 
    We have 
    \begin{align}
        \del_{t,i,s+1} &= \del_{t,i,s}  + \alpha^2 \mathbf{B}_{t,i,s}^\top \mathbf{G}_{t,i,s} +\alpha^2 \mathbf{G}_{t,i,s}^\top \mathbf{B}_{t,i,s} - \alpha^3 \mathbf{G}_{t,i,s}^\top \mathbf{G}_{t,i,s} \nonumber 
    \end{align}
We use $A_{2,t,i}(s)$ and $A_{3,t,i}(s)$ throughout the proof. Recall that $A_{3,t,i}(s)$ directly implies $\|\mathbf{B}_{t,i,s}\|\leq \tfrac{1.1}{\sqrt{\alpha}}$.
This bound as well as the bound on $\|\mathbf{w}_{t,i,s}\|$ from $A_{2,t,i}(s)$ and the Cauchy Schwarz inequality implies $\|\mathbf{e}_{t,i,s}\|\leq 3.2 L_{\max}$ and $\|\mathbf{G}_{t,i,s}\|_2 \leq  7\sqrt{\alpha}L_{\max}^2$, thus $\|\alpha^3 \mathbf{G}_{t,i,s}^\top \mathbf{G}_{t,i,s} \|\leq 49\alpha^{4} L_{\max}^4$. Next, 
    \begin{align}
        \mathbf{B}_{t,i,s}^\top \mathbf{G}_{t,i,s} &= \mathbf{B}_{t,i,s}^\top \mathbf{e}_{t,i,s}\mathbf{w}_{t,i,s}^\top \nonumber \\
        &= \alpha\mathbf{B}_{t,i,s}^\top \mathbf{e}_{t,i,s}\mathbf{w}_{\ast,i}^\top \mathbf{B}_\ast^\top \mathbf{B}_{t,i,s-1} +  \mathbf{B}_{t,i,s}^\top \mathbf{e}_{t,i,s}\mathbf{w}_{t,i,s-1}^\top \del_{t,i,s-1} \nonumber \\
        &= \alpha\mathbf{B}_{t,i,s}^\top \mathbf{e}_{t,i,s}\mathbf{w}_{\ast,i}^\top \mathbf{B}_\ast^\top \mathbf{B}_{t,i,s} + \alpha\mathbf{B}_{t,i,s}^\top \mathbf{e}_{t,i,s}\mathbf{w}_{\ast,i}^\top \mathbf{B}_\ast^\top (\mathbf{B}_{t,i,s-1} -\mathbf{B}_{t,i,s}) \nonumber \\
    &\quad +  \mathbf{B}_{t,i,s}^\top \mathbf{e}_{t,i,s}\mathbf{w}_{t,i,s-1}^\top \del_{t,i,s-1} \nonumber 
    \end{align}
where, by the Cauchy-Schwarz inequality and $A_{2,t,i}(s)$ and $A_{3,t,i}(s)$
\begin{align}
    \|\alpha\mathbf{B}_{t,i,s}^\top \mathbf{e}_{t,i,s}\mathbf{w}_{\ast,i}^\top \mathbf{B}_\ast^\top (\mathbf{B}_{t,i,s-1} -\mathbf{B}_{t,i,s})\| &= \alpha^2 \|\mathbf{B}_{t,i,s}^\top \mathbf{e}_{t,i,s}\mathbf{w}_{\ast,i}^\top \mathbf{B}_\ast^\top \mathbf{e}_{t,i,s-1}\mathbf{w}_{t,i,s-1}^\top\| \leq 23 \alpha^2   L_{\max}^4, \nonumber \\
    \| \mathbf{B}_{t,i,s}^\top \mathbf{e}_{t,i,s}\mathbf{w}_{t,i,s-1}^\top \del_{t,i,s-1} \| &\leq 15 c_3 \alpha^2 \tau  L_{\max}^4 \kappa_{\max}^2 E_0^{-1},
\end{align}
and
\begin{align}
 \|\alpha&\mathbf{B}_{t,i,s}^\top \mathbf{e}_{t,i,s}\mathbf{w}_{\ast,i}^\top \mathbf{B}_\ast^\top \mathbf{B}_{t,i,s} \|  \nonumber \\
 &= \| \alpha^2\mathbf{B}_{t,i,s}^\top \mathbf{B}_{t,i,s}\mathbf{B}_{t,i,s-1}^\top \mathbf{B}_\ast \mathbf{w}_{\ast,i} \mathbf{w}_{\ast,i}^\top \mathbf{B}_\ast^\top \mathbf{B}_{t,i,s}  - \alpha\mathbf{B}_{t,i,s}^\top \mathbf{B}_{\ast}\mathbf{w}_{\ast,i}\mathbf{w}_{\ast,i}^\top \mathbf{B}_\ast^\top \mathbf{B}_{t,i,s}  \nonumber \\
 &\quad + \alpha^2\mathbf{B}_{t,i,s}^\top \mathbf{B}_{t,i,s}\del_{t,i,s-1} \mathbf{w}_{t,i,s-1} \mathbf{w}_{\ast,i}^\top \mathbf{B}_\ast^\top \mathbf{B}_{t,i,s} \| \nonumber \\
 &= \| - \alpha \del_{t,i,s}\mathbf{B}_{t,i,s}^\top \mathbf{B}_{\ast}\mathbf{w}_{\ast,i}\mathbf{w}_{\ast,i}^\top \mathbf{B}_\ast^\top \mathbf{B}_{t,i,s} \nonumber \\
    &\quad + \alpha^2\mathbf{B}_{t,i,s}^\top \mathbf{B}_{t,i,s}(\mathbf{B}_{t,i,s-1} - \mathbf{B}_{t,i,s})^\top \mathbf{B}_\ast \mathbf{w}_{\ast,i} \mathbf{w}_{\ast,i}^\top \mathbf{B}_\ast^\top \mathbf{B}_{t,i,s}  \nonumber \\
 &\quad + \alpha^2\mathbf{B}_{t,i,s}^\top \mathbf{B}_{t,i,s}\del_{t,i,s-1} \mathbf{w}_{t,i,s-1} \mathbf{w}_{\ast,i}^\top \mathbf{B}_\ast^\top \mathbf{B}_{t,i,s} \| \nonumber \\
 &\leq   \alpha\| \del_{t,i,s}\mathbf{B}_{t,i,s}^\top \mathbf{B}_{\ast}\mathbf{w}_{\ast,i}\mathbf{w}_{\ast,i}^\top \mathbf{B}_\ast^\top \mathbf{B}_{t,i,s}\| \nonumber \\
    &\quad + \alpha^2\|\mathbf{B}_{t,i,s}^\top \mathbf{B}_{t,i,s}\mathbf{w}_{t,i,s-1}\mathbf{e}_{t,i,s-1}^\top \mathbf{B}_\ast \mathbf{w}_{\ast,i} \mathbf{w}_{\ast,i}^\top \mathbf{B}_\ast^\top \mathbf{B}_{t,i,s}\|   \nonumber \\
 &\quad + \alpha^2\| \mathbf{B}_{t,i,s}^\top \mathbf{B}_{t,i,s}\del_{t,i,s-1} \mathbf{w}_{t,i,s-1} \mathbf{w}_{\ast,i}^\top \mathbf{B}_\ast^\top \mathbf{B}_{t,i,s}  \|  \nonumber \\
 &= 7 c_3  \alpha^2 \tau  L_{\max}^4 \kappa_{\max}^2 E_0^{-1} + 9 \alpha^2  L_{\max}^4   .  
\end{align}
Thus, 
\begin{align}
   \| \del_{t,i,s+1}\|_2 &\leq  \|\del_{t,i,s} \|_2 + 2 \alpha^2 \| \mathbf{B}_{t,i,s}^\top \mathbf{G}_{t,i,s}\| + \alpha^3 \| \mathbf{G}_{t,i,s}^\top \mathbf{G}_{t,i,s}\| \nonumber \\
   &\leq  \|\del_{t,i,s} \|_2 + 46 c_3 \alpha^4 \tau  L_{\max}^4 \kappa_{\max}^2 E_0^{-1} + 81 \alpha^4  L_{\max}^4  \nonumber \\
   &\vdots  \nonumber \\
   &\leq  \|\del_{t} \|_2  + 46 c_3 \alpha^4 \tau^2  L_{\max}^4 \kappa_{\max}^2 E_0^{-1} + 81 \alpha^4 \tau  L_{\max}^4   \nonumber \\
      &\leq  c_3 \alpha^2 \tau L_{\max}^2 \kappa_{\max}^2 E_0^{-1}  + 46 c_3 \alpha^4 \tau^2  L_{\max}^4 \kappa_{\max}^2 E_0^{-1} + 81 \alpha^4 \tau  L_{\max}^4   \nonumber \\
   &\leq 2 c_3 \alpha^2 \tau L_{\max}^2 \kappa_{\max}^2 E_0^{-1}
\end{align}
by choice of $c_3$ and $\alpha$ sufficiently small.
\end{proof}



 

\begin{lemma}$A_{2,t,i}(s) \cap A_{3,t,i}(s+1) \cap A_{3}(t) \implies A_{4,t,i}(s+1)$. \label{lem:a3s}
\end{lemma}

\begin{proof}
Note that
\begin{align}
   \|\mathbf{B}_{\ast,\perp}^\top \mathbf{B}_{t,i,s+1}\|_2 &=  \|\mathbf{B}_{\ast,\perp}^\top\mathbf{B}_{t,i,s} (\mathbf{I}_k - \alpha \mathbf{w}_{t,i,s}\mathbf{w}_{t,i,s}^\top) \|_2\nonumber \\
   &\leq \|\mathbf{B}_{\ast,\perp}^\top\mathbf{B}_{t,i,s} \|_2 \nonumber  \\
   &\quad \vdots \nonumber \\
   &\leq \|\mathbf{B}_{\ast,\perp}^\top \mathbf{B}_{t} \|_2
\end{align}
where the first inequality follows since $\|\mathbf{w}_{t,i,s}\|\leq 2 \sqrt{\alpha}L_{\max}$ (by $A_{2,t,i}(s)$) and $\alpha$ is sufficiently small, and the last inequality follows by recursively applying the first inequality for all local iterations leading up to $s$.
Thus
\begin{align}
    \dist_{t,i,s} \leq \frac{\|\mathbf{B}_{t}\|_2 }{\sigma_{\min}(\mathbf{B}_{t,i,s})} \dist_{t} \leq \sqrt{\tfrac{1+ c_3 \alpha^2 \tau L_{\max}^2 \kappa_{\max}^2 E_0^{-1}}{1 - 2c_3 \alpha^2 \tau L_{\max}^2 \kappa_{\max}^2 E_0^{-1}}} \dist_{t}  \leq 2 \dist_t. \nonumber
\end{align}
\end{proof}


\begin{lemma}$\cap_{i\in \mathcal{I}_t} \big(A_{2,t,i}(\tau-1) \cap A_{3,t,i}(\tau-1)\big) \cap A_2(t) \cap A_3(t) \implies A_{1}(t+1)$. \label{lem:a1}
\end{lemma}

\begin{proof}
Expanding $\mathbf{w}_{t+1}$ yields
\begin{align}
    \mathbf{w}_{t+1} &= \tfrac{1}{m}\sum_{i\in \mathcal{I}_t}\mathbf{w}_{t,i,\tau}  \nonumber \\
    &= \tfrac{1}{m}\sum_{i\in \mathcal{I}_t} \del_{t,i,\tau-1}\mathbf{w}_{t,i,\tau-1}  +\alpha \mathbf{B}_{t,i,\tau-1}^\top \mathbf{B}_\ast\mathbf{w}_{\ast,i} \nonumber \\
    &= \alpha \mathbf{B}_{t}^\top \mathbf{B}_\ast\mathbf{\bar{w}}_{\ast,t} + \tfrac{1}{m}\sum_{i\in \mathcal{I}_t} \del_{t,i,\tau-1}\mathbf{w}_{t,i,\tau-1}    +\alpha (\mathbf{B}_{t,i,\tau-1} - \mathbf{B}_{t})^\top \mathbf{B}_\ast\mathbf{w}_{\ast,i} \nonumber \\
    &=  \frac{1 }{n}\sum_{i\in \mathcal{I}_t}    \alpha(\mathbf{B}_{t,i,\tau-1} - \mathbf{B}_{t})^\top \mathbf{B}_\ast\mathbf{w}_{\ast,i} + \alpha \del_{t,i,\tau-1}\mathbf{B}_{t,i,\tau-2}^\top\mathbf{B}_\ast\mathbf{w}_{\ast,i}    \nonumber \\
    &\quad \quad \quad \quad\quad \quad  + \alpha \del_{t,i,\tau-1} \del_{t,i,\tau-2} \mathbf{w}_{t,i,\tau-2}  +\alpha \mathbf{B}_{t}^\top \mathbf{B}_\ast\mathbf{\bar{w}}_{\ast,t}  \nonumber \\
    &= \alpha \mathbf{B}_{t}^\top \mathbf{B}_\ast\mathbf{\bar{w}}_{\ast,t} + \alpha \del_{t} \mathbf{B}_{t}^\top \mathbf{B}_\ast\mathbf{\bar{w}}_{\ast,t} \nonumber \\
    &\quad + \tfrac{1}{m}\sum_{i\in \mathcal{I}_t}  \alpha(\mathbf{B}_{t,i,\tau-1} - \mathbf{B}_{t})^\top \mathbf{B}_\ast\mathbf{w}_{\ast,i} + \alpha ( \del_{t,i,\tau-1}\mathbf{B}_{t,i,\tau-2}^\top - \del_{t}\mathbf{B}_{t}^\top )\mathbf{B}_\ast\mathbf{w}_{\ast,i}  \nonumber \\
    &\quad \quad \quad \quad\quad \quad   +  \del_{t,i,\tau-1} \del_{t,i,\tau-2} \mathbf{w}_{t,i,\tau-2}  \nonumber \\
    &= \alpha \mathbf{B}_{t+1}^\top \mathbf{B}_\ast\mathbf{\bar{w}}_{\ast,t+1} + \alpha \del_{t+1} \mathbf{B}_{t+1}^\top \mathbf{B}_\ast\mathbf{\bar{w}}_{\ast,t+1} \nonumber \\
    &\quad + \alpha \mathbf{B}_{t+1}^\top \mathbf{B}_\ast(\mathbf{\bar{w}}_{\ast,t}-\mathbf{\bar{w}}_{\ast,t+1}) + \alpha \del_{t+1} \mathbf{B}_{t+1}^\top \mathbf{B}_\ast(\mathbf{\bar{w}}_{\ast,t} - \mathbf{\bar{w}}_{\ast,t+1})  \nonumber \\
    &\quad + \alpha (\mathbf{B}_t -\mathbf{B}_{t+1})^\top \mathbf{B}_\ast\mathbf{\bar{w}}_{\ast,t} + \alpha (\del_t\mathbf{B}_t -\del_{t+1}\mathbf{B}_{t+1})^\top \mathbf{B}_\ast\mathbf{\bar{w}}_{\ast,t}  \nonumber \\
    &\quad + \tfrac{1}{m}\sum_{i\in \mathcal{I}_t}  \alpha(\mathbf{B}_{t,i,\tau-1} - \mathbf{B}_{t})^\top \mathbf{B}_\ast\mathbf{w}_{\ast,i} + \alpha ( \del_{t,i,\tau-1}\mathbf{B}_{t,i,\tau-2}^\top - \del_{t}\mathbf{B}_{t}^\top )\mathbf{B}_\ast\mathbf{w}_{\ast,i}  \nonumber \\
    &\quad \quad \quad \quad\quad \quad   +  \del_{t,i,\tau-1} \del_{t,i,\tau-2} \mathbf{w}_{t,i,\tau-2}  \label{chara}
\end{align}
The remainder of the proof lies in bounding the error terms, which are all terms in the RHS of \eqref{chara} besides the terms in the first line.
First, by $A_0$ and the triangle inequality, we have
\begin{align}
    \|\mathbf{\bar{w}}_{\ast,t}-\mathbf{\bar{w}}_{\ast,t+1} \| &\leq \|\mathbf{\bar{w}}_{\ast,t}-\mathbf{\bar{w}}_{\ast}\| + \|\mathbf{\bar{w}}_{\ast,t+1}-\mathbf{\bar{w}}_{\ast}\| \leq 8\alpha^2 L_{\max}^3 \nonumber
\end{align}
Thus, by $A_3(t+1)$, we have
\begin{align}
    \|\alpha \mathbf{B}_{t+1}^\top \mathbf{B}_\ast(\mathbf{\bar{w}}_{\ast,t}-\mathbf{\bar{w}}_{\ast,t+1}) \| &\leq 1.1 \sqrt{\alpha}   \|\mathbf{\bar{w}}_{\ast,t}-\mathbf{\bar{w}}_{\ast,t+1} \| \leq 9 \alpha^{2.5} L_{\max}^3 \nonumber \\
    \|\alpha \del_{t+1} \mathbf{B}_{t+1}^\top \mathbf{B}_\ast(\mathbf{\bar{w}}_{\ast,t}-\mathbf{\bar{w}}_{\ast,t+1}) \| &\leq 1.1 \sqrt{\alpha} \|\del_{t+1}\| \|\mathbf{\bar{w}}_{\ast,t}-\mathbf{\bar{w}}_{\ast,t+1} \| \leq 18 c_3 \alpha^{4.5} \tau L_{\max}^5 \kappa_{\max}^2 E_0^{-1} \nonumber
\end{align}
Next, we can bound the difference between the locally-updated representation and the global representation as follows, for any $s\in \{1,\dots,\tau\}$
\begin{align}
    \|\mathbf{B}_{t,i,s} - \mathbf{B}_{t}\|_2 &\leq \sum_{r=1}^{s} \|\mathbf{B}_{t,i,r}- \mathbf{B}_{t,i,r-1}\|_2 \leq   \alpha\sum_{r=1}^{s} \|\mathbf{e}_{t,i,r-1}\mathbf{w}_{t,i,r-1}^\top\|_2 \leq  7 \alpha^{1.5}s L_{\max}^2  \label{chara2}
\end{align}
using $A_2(t), A_3(t), A_{2,t,i}(\tau-1)$ and $A_{3,t,i}(\tau-1)$ to control the norms of $\mathbf{w}_{t,i,s-1}$ and $\mathbf{B}_{t,i,s-1}$.
From \eqref{chara2} it follows that 
\begin{align}
     \|\mathbf{B}_{t+1} - \mathbf{B}_{t}\|_2 &\leq \tfrac{1}{m}\sum_{i\in \mathcal{I}_t}\|\mathbf{B}_{t,i,\tau} - \mathbf{B}_{t}\|_2 \leq  7\alpha^{1.5}\tau L_{\max}^2  \nonumber \\
  \|\mathbf{B}_{t,i,\tau-2}\del_{t,i,\tau-1} - \mathbf{B}_{t}\del_{t}\|_2  &\leq \|\mathbf{B}_{t,i,\tau-2} - \mathbf{B}_{t}\|_2 + \alpha\|\mathbf{B}_{t,i,\tau-2}\mathbf{B}_{t,i,\tau-1}^\top\mathbf{B}_{t,i,\tau-1} - \mathbf{B}_{t}\mathbf{B}_{t}^\top\mathbf{B}_{t}\|_2  \nonumber \\
  &\leq \|\mathbf{B}_{t,i,\tau-2} - \mathbf{B}_{t}\|_2  +\alpha\|(\mathbf{B}_{t,i,\tau-2} - \mathbf{B}_{t})\mathbf{B}_{t,i,\tau-1}^\top\mathbf{B}_{t,i,\tau-1}\|_2  \nonumber \\
  &\quad + \alpha \| \mathbf{B}_t (\mathbf{B}_{t,i,\tau-1} - \mathbf{B}_t)^\top \mathbf{B}_{t,i,\tau-1} \|_2 \nonumber \\
    &\quad + \alpha \| \mathbf{B}_t \mathbf{B}_{t}^\top (\mathbf{B}_{t,i,\tau-1} - \mathbf{B}_t)  \|_2 \nonumber \\
  &\leq \|\mathbf{B}_{t,i,\tau-2} - \mathbf{B}_{t}\|_2   +\alpha\|\mathbf{B}_{t,i,\tau-2} - \mathbf{B}_{t}\|\|\mathbf{B}_{t,i,\tau-1}^\top\mathbf{B}_{t,i,\tau-1}\|_2  \nonumber \\
  &\quad + \alpha \| \mathbf{B}_t \|\|\mathbf{B}_{t,i,\tau-1} - \mathbf{B}_t\|\|\mathbf{B}_{t,i,\tau-1} \|_2\nonumber \\
    &\quad + \alpha \| \mathbf{B}_t \mathbf{B}_{t}^\top \|\|\mathbf{B}_{t,i,\tau-1} - \mathbf{B}_t  \| \nonumber \\
  &\leq 31 \alpha^{1.5}\tau L_{\max}^2  \nonumber \\
 \|\mathbf{B}_t\del_t - \mathbf{B}_{t+1}\del_{t+1}\|_2 &\leq 31 \alpha^{1.5}\tau L_{\max}^2 
\end{align}
Also, we have by $A_{2,t,i}(\tau-1)$ and $A_{3,t,i}(\tau-1)$,
\begin{align}
    \| \del_{t,i,\tau-1} \del_{t,i,\tau-2} \mathbf{w}_{t,i,\tau-2} \|_2 &\leq 8 c_3^2 \alpha^{4.5} \tau^2 L_{\max}^5 \kappa_{\max}^4 E_0^{-2}.
\end{align}
Thus, using these bounds with  \eqref{chara}, we obtain
\begin{align}
 \| \mathbf{w}_{t+1} -  \alpha (\mathbf{I}_k +\del_{t+1})\mathbf{B}_{t+1}^\top \mathbf{B}_\ast\mathbf{\bar{w}}_{\ast,t+1} \|_2 &\leq  82 \alpha^{2.5}\tau L_{\max}^3 + (8c_3^2 + 12c_3) \alpha^{4.5}\tau^2 L_{\max}^5 \kappa_{\max}^4 E_0^{-2}\nonumber \\
 &\leq  91 \alpha^{2.5}\tau L_{\max}^3 \nonumber
\end{align} to complete the proof, where we have used that $\alpha$ is sufficiently small in the last inequality.
\end{proof}

\begin{lemma}$\cap_{i\in \mathcal{I}_t} A_{2,t,i}(\tau)  \implies A_{2}(t+1)$ \label{lem:1.5}

\end{lemma}

\begin{proof}
By the triangle inequality and $\cap_{i\in \mathcal{I}_t} A_{2,t,i}(\tau)$, we have 
\begin{align}
    \|\mathbf{w}_{t+1}\| &= \bigg\|\tfrac{1}{m}\sum_{i\in\mathcal{I}_t} \mathbf{w}_{t,i,\tau} \bigg\|\leq \tfrac{1}{m}\sum_{i\in\mathcal{I}_t} \|\mathbf{w}_{t,i,\tau}\|  \leq 2 \sqrt{\alpha} L_{\max} \nonumber
\end{align}
as desired.
\end{proof}

 \begin{lemma} \label{lem:distlb}
 $A_{3}(t)\cap A_4(t) \implies \sigma_{\min}^2(\mathbf{B}_{t}^\top\mathbf{B}_\ast)\geq \tfrac{0.1}{\alpha} E_0$.
 \end{lemma}
 
 \begin{proof}
 First note that 
 \begin{align}
     \sigma_{\min}^2(\mathbf{B}_{t}^\top\mathbf{B}_\ast)&\geq \sigma_{\min}^2 (\mathbf{R}_t) \sigma_{\min}^2(\mathbf{\hat{B}}_{t}^\top\mathbf{B}_\ast) \nonumber \\
     &\geq \tfrac{0.9}{\alpha} \sigma_{\min}^2(\mathbf{\hat{B}}_{t}^\top\mathbf{B}_\ast) \label{rtrt} \\
     &= \tfrac{0.9}{\alpha}( 1 -  \|\mathbf{\hat{B}}_{t}^\top\mathbf{B}_{\ast,\perp}\|_2^2 ) \nonumber \\
     &= \tfrac{0.9}{\alpha} (1 - \dist_t^2) \label{whrr}
 \end{align}
 where $\mathbf{\hat{B}}_{t}\mathbf{R}_t = \mathbf{B}_t$ is the QR factorization of $\mathbf{B}_t$.
 Next, we would like to show the RHS is at most $(2+\delta_0)/3$.
 Using $A_3(t)$ and $A_4(t)$, we obtain
 \begin{align}
     \dist_{t}&= \|\mathbf{B}_{\ast,\perp}^\top \mathbf{\hat{B}}_{t}\|_2 \nonumber \\
    &\leq\tfrac{1}{ \sigma_{\min}(\mathbf{B}_{t})}\|\mathbf{B}_{\ast,\perp}^\top \mathbf{{B}}_{t}\|_2\nonumber \\
    &\quad \vdots \nonumber \\
    &\leq \tfrac{1}{ \sigma_{\min}(\mathbf{{B}}_{t+1})}(1-0.04\alpha^2 \tau E_0 \mu^2 )^t\|\mathbf{B}_{\ast,\perp}^\top \mathbf{{B}}_{0}\|_2\nonumber \\
    &\leq  \tfrac{\sigma_{\max}(\mathbf{B}_0)}{ \sigma_{\min}(\mathbf{{B}}_{t})}(1-0.04\alpha^2 \tau E_0 \mu^2 )^t\delta_0 \nonumber \\
    &\leq \tfrac{\sqrt{1 +\|\del_0\|_2}/\sqrt{\alpha}}{ \sqrt{1-\|\del_{t}\|_2}/\sqrt{\alpha}}\delta_0\nonumber 
    \end{align}
Next we use that $\|\del_0\|\leq\alpha^2\tau L_{\max}^2 \kappa_{\max}^2 \leq 0.1 (1 - \delta_0)^2$ by choice of initialization and choice of $\alpha$, and similarly $\|\del_t\|\leq \alpha^2\tau L_{\max}^2 \kappa_{\max}^2 E_0^{-1} \leq 0.1 (1 - \delta_0)^2/(1 - \delta_0^2)$. Let $c\coloneqq 0.1$. Then we have
    \begin{align}
   \tfrac{\sqrt{1 +\|\del_0\|_2}/\sqrt{\alpha}}{ \sqrt{1-\|\del_{t}\|_2}/\sqrt{\alpha}}\delta_0 &\leq \tfrac{\sqrt{1 + c(1-\delta_0)^2}}{ \sqrt{1-c(1-\delta_0)^2/(1 -\delta_0^2)}} \delta_0 \nonumber \\
    &= \tfrac{\sqrt{1 + c(1-\delta_0)^2}}{ \sqrt{1-c(1-\delta_0)/(1 +\delta_0)}} \delta_0 \nonumber \\
    &= \tfrac{\sqrt{1+\delta_0 + c(1-\delta_0)^2(1+\delta_0)}}{ \sqrt{1-c +(1+c)\delta_0}} \delta_0 \nonumber 
 \end{align}
 Now, observe that
 \begin{align}
     \tfrac{\sqrt{1+\delta_0 + c(1-\delta_0)^2(1+\delta_0)}}{ \sqrt{1-c +(1+c)\delta_0}} \delta_0 &\leq \tfrac{2+\delta_0}{3} \nonumber \\
     \iff \tfrac{{1+\delta_0 + c(1-\delta_0)^2(1+\delta_0)}}{ {1-c +(1+c)\delta_0}}\delta_0^2 &\leq \tfrac{4+4\delta_0+\delta_0^2}{9} \nonumber\\
     \iff (1+c)\delta_0^2 + \delta_0^3  -c\delta_0^4 +c\delta_0^5 &\leq 
     (4-4c+ 8\delta_0 +8\delta_0^2 +(1+c)\delta_0^3)/9
     \nonumber
     \end{align}
     \begin{align}
     \iff c\delta_0^5 - c\delta_0^4 + \tfrac{8-c}{9}\delta_0^3 + \tfrac{1+9c}{9} \delta_0^2 -\tfrac{8}{9}\delta_0 - \tfrac{4-4c}{9} &\leq 0 \label{where}
 \end{align}
 where \eqref{where} holds for all $\delta_0 \in [0,1)$ and $c=0.1$, therefore we have
 \begin{align}
   \dist_t  \leq \tfrac{\sqrt{1 +\|\del_0\|_2}/\sqrt{\alpha}}{ \sqrt{1-\|\del_{t}\|_2}/\sqrt{\alpha}}\delta_0 \leq \tfrac{2+\delta_0}{3}. \label{uhu}
 \end{align}
 Thus, using \eqref{whrr}, we obtain
 \begin{align}
     \sigma_{\min}^2(\mathbf{B}_{t}^\top\mathbf{B}_\ast)&\geq \tfrac{0.9}{\alpha} \left(1 - \tfrac{4+4\delta_0 +\delta_0^2}{9}\right) \nonumber \\
     &\geq \tfrac{0.9}{\alpha} \left(1 - \tfrac{8 +\delta_0^2}{9}\right) \nonumber \\
     &= \tfrac{0.9}{9\alpha} E_0\nonumber \\
     &= \tfrac{0.1}{\alpha} E_0 \nonumber
 \end{align}
 as desired.
 \end{proof}

\begin{lemma} $\cap_{i\in \mathcal{I}_t} \big(\cap_{h=1}^4 A_{h,t,i}(\tau) \big)  \cap A_1(t) \cap A_2(t) \cap A_3(t) \cap A_5(t) \implies A_{3}(t+1).$ \label{lem:a2}

\end{lemma}

\begin{proof}
We aim to write $\del_{t+1} = \tfrac{1}{2}(\mathbf{I}_k - \mathbf{P}_t)\del_t + \tfrac{1}{2}\del_t(\mathbf{I}_k - \mathbf{P}_t) + \mathbf{Z}_t$ for a positive definite matrix $\mathbf{P}_t$ and a perturbation matrix $\mathbf{Z}_t$. This will yield the inequality $\|\del_{t+1}\|_2 \leq (1 - \lambda_{\min}(\mathbf{P}_t))\|\del_{t}\|_2 + \|\mathbf{Z}_t\|_2$. Assuming $\lambda_{\min}(\mathbf{P}_t)$ and $\|\mathbf{Z}_t\|_2$ scale appropriately (defined later), this inequality combined with inductive hypothesis $A_5(t)$ will give the desired upper bound on $\|\del_{t+1}\|_2$ (this is because the upper bound on $\|\mathbf{Z}_t\|_2$ scales with $\dist_t$, so $A_5(t)$ contributes to controlling $\|\mathbf{Z}_t\|_2$). The proof therefore relies on showing the existence of appropriate $\mathbf{P}_t$ and $\mathbf{Z}_t$. 

First recall $\del_t \coloneqq \mathbf{I} - \alpha\mathbf{B}_t^\top \mathbf{B}_t$ and $\deld_t \coloneqq \mathbf{I}_d - \alpha \mathbf{B}_t \mathbf{B}_t^\top$. Let $\mathbf{G}_t \coloneqq \tfrac{1}{\alpha}(\mathbf{B}_t - \mathbf{B}_{t+1})$, i.e. $\mathbf{G}_t$ satisfies $\mathbf{B}_{t+1}= \mathbf{B}_t - \alpha \mathbf{G}_t$.
Then
\begin{align}
    \del_{t+1} &= \mathbf{I}_k - \alpha \mathbf{B}_{t+1}^\top \mathbf{B}_{t+1} =  \del_t + \alpha^2 \mathbf{B}_t^\top \mathbf{G}_t + \alpha^2 \mathbf{G}_t^\top \mathbf{B}_t - \alpha^3 \mathbf{G}_t^\top \mathbf{G}_t \label{del}
\end{align}
The key is showing that $\alpha^2 \mathbf{B}_t^\top \mathbf{G}_t = -\tfrac{1}{2}\del_t \mathbf{P}_t + \mathbf{Z}_t'$ for appropriate $\mathbf{P}_t$ and $\mathbf{Z}_t'$. Then, by \eqref{del}, we will have $\del_{t+1} =  \tfrac{1}{2}(\mathbf{I}_k - \mathbf{P}_t)\del_t + \tfrac{1}{2}\del_t(\mathbf{I}_k - \mathbf{P}_t) + \mathbf{Z}_t $ as desired, where $\mathbf{Z}_t = \mathbf{Z}_t' + (\mathbf{Z}_t')^\top - \alpha^3 \mathbf{G}_t^\top \mathbf{G}_t$.

Notice that $\mathbf{G}_t$ is the average across clients of the sum of their local gradients on every local update. In particular, we have 
\begin{align}
    \mathbf{G}_t = (\mathbf{B}_{t}\mathbf{w}_{t}- \mathbf{B}_\ast\mathbf{\bar{w}}_{\ast,t})\mathbf{w}_{t}^\top +\tfrac{1}{m}\sum_{i\in \mathcal{I}_t}\sum_{s = 1}^{\tau-1} (\mathbf{B}_{t,i,s}\mathbf{w}_{t,i,s}- \mathbf{B}_\ast\mathbf{w}_{\ast,i})\mathbf{w}_{t,i,s}^\top 
\end{align}
We will unroll the gradients for the first two local updates only, in order to obtain a negative term that will contribute to the contraction of $\|\del_t\|$ (i.e. $\mathbf{P}_t$ will be extracted from the gradients for the first two local updates).
The remaining terms will belong to $\mathbf{Z}_t$ and must be upper bounded (i.e. $\|\del_{t+1}\|$ can grow due to local updates beyond the second local update, but we will show that it can't grow too much). In particular, we have
\begin{align}
     \mathbf{G}_t &= (\mathbf{B}_{t}\mathbf{w}_{t}- \mathbf{B}_\ast\mathbf{\bar{w}}_{\ast,t})\mathbf{w}_{t}^\top +\tfrac{1}{m}\sum_{i\in \mathcal{I}_t}(\mathbf{B}_{t,i,1}\mathbf{w}_{t,i,1}- \mathbf{B}_\ast\mathbf{w}_{\ast,i})\mathbf{w}_{t,i,1}^\top \nonumber \\
    &\quad +\tfrac{1}{m}\sum_{i\in \mathcal{I}_t}\sum_{s = 2}^{\tau-1} (\mathbf{B}_{t,i,s}\mathbf{w}_{t,i,s}- \mathbf{B}_\ast\mathbf{w}_{\ast,i})\mathbf{w}_{t,i,s}^\top  \nonumber \\
    &= (\mathbf{B}_{t}\mathbf{w}_{t}- \mathbf{B}_\ast\mathbf{\bar{w}}_{\ast,t})\mathbf{w}_{t}^\top - \alpha \deld_t \mathbf{B}_\ast \tfrac{1}{m}\sum_{i\in \mathcal{I}_t}\mathbf{w}_{\ast,i}\mathbf{w}_{\ast,i}^\top\mathbf{B}_\ast^\top \mathbf{B}_t - \deld_t \mathbf{B}_\ast \mathbf{\bar{w}}_{\ast,t}\mathbf{w}_t^\top \del_t \nonumber \\
    &\quad + \tfrac{1}{m}\sum_{i\in \mathcal{I}_t}\alpha^2(\mathbf{B}_t\mathbf{w}_t - \mathbf{B}_\ast \mathbf{w}_{\ast,i})\mathbf{w}_t^\top \mathbf{B}_{t}^\top\mathbf{B}_\ast\mathbf{w}_{\ast,i}\mathbf{w}_{t,i,1}^\top
     \nonumber \\
    &\quad + \tfrac{1}{m}\sum_{i\in \mathcal{I}_t}(\mathbf{B}_t - \alpha(\mathbf{B}_t\mathbf{w}_t - \mathbf{B}_\ast \mathbf{w}_{\ast,i})\mathbf{w}_t^\top) \del_t \mathbf{w}_{t}\mathbf{w}_{t,i,1}^\top \nonumber \\
    &\quad+\tfrac{1}{m}\sum_{i\in \mathcal{I}_t}\sum_{s = 2}^{\tau-1} (\mathbf{B}_{t,i,s}\mathbf{w}_{t,i,s}- \mathbf{B}_\ast\mathbf{w}_{\ast,i})\mathbf{w}_{t,i,s}^\top  \nonumber 
\end{align}
Multiplying both sides by $\mathbf{B}_t^\top$, and using the fact that $\mathbf{B}_t^\top \deld_t = \del_t \mathbf{B}_t^\top$, we obtain
\begin{align}
    \mathbf{B}_t^\top \mathbf{G}_t &= \mathbf{B}_t^\top(\mathbf{B}_{t}\mathbf{w}_{t}- \mathbf{B}_\ast\mathbf{\bar{w}}_{\ast,t})\mathbf{w}_{t}^\top - \alpha \del_t \mathbf{B}_t^\top \mathbf{B}_\ast \tfrac{1}{m}\sum_{i\in \mathcal{I}_t}\mathbf{w}_{\ast,i}\mathbf{w}_{\ast,i}^\top\mathbf{B}_\ast^\top \mathbf{B}_t - \del_t \mathbf{B}_t^\top \mathbf{B}_\ast \mathbf{\bar{w}}_{\ast,t}\mathbf{w}_t^\top \del_t \nonumber \\
    &\quad - \mathbf{B}_t^\top\tfrac{1}{m}\sum_{i\in \mathcal{I}_t}\alpha^2(\mathbf{B}_t\mathbf{w}_t - \mathbf{B}_\ast \mathbf{w}_{\ast,i})\mathbf{w}_t^\top \mathbf{B}_{t}^\top\mathbf{B}_\ast\mathbf{w}_{\ast,i}\mathbf{w}_{t,i,1}^\top
     \nonumber \\
    &\quad + \mathbf{B}_t^\top\tfrac{1}{m}\sum_{i\in \mathcal{I}_t}(\mathbf{B}_t - \alpha(\mathbf{B}_t\mathbf{w}_t - \mathbf{B}_\ast \mathbf{w}_{\ast,i})\mathbf{w}_t^\top) \del_t \mathbf{w}_{t}\mathbf{w}_{t,i,1}^\top\nonumber \\
    &\quad +\mathbf{B}_t^\top\tfrac{1}{m}\sum_{i\in \mathcal{I}_t}\sum_{s = 2}^{\tau-1} (\mathbf{B}_{t,i,s}\mathbf{w}_{t,i,s}- \mathbf{B}_\ast\mathbf{w}_{\ast,i})\mathbf{w}_{t,i,s}^\top  \nonumber \\
    &= - \alpha \del_t \mathbf{B}_t^\top \mathbf{B}_\ast \tfrac{1}{m}\sum_{i\in \mathcal{I}_t}\mathbf{w}_{\ast,i}\mathbf{w}_{\ast,i}^\top\mathbf{B}_\ast^\top \mathbf{B}_t + \mathbf{N}_t \label{61}
\end{align}
where the first term is a negative term that helps $\|\del_{t+1}\|$ stay small, and the remaining terms are given by
\begin{align}
    \mathbf{N}_t &\coloneqq \mathbf{B}_t^\top(\mathbf{B}_{t}\mathbf{w}_{t}- \mathbf{B}_\ast\mathbf{\bar{w}}_{\ast,t})\mathbf{w}_{t}^\top - \del_t \mathbf{B}_t^\top \mathbf{B}_\ast \mathbf{\bar{w}}_{\ast,t}\mathbf{w}_t^\top \del_t \nonumber \\
    &\quad  - \mathbf{B}_t^\top\tfrac{1}{m}\sum_{i\in \mathcal{I}_t}\alpha^2(\mathbf{B}_t\mathbf{w}_t - \mathbf{B}_\ast \mathbf{w}_{\ast,i})\mathbf{w}_t^\top \mathbf{B}_{t}^\top\mathbf{B}_\ast\mathbf{w}_{\ast,i}\mathbf{w}_{t,i,1}^\top
     \nonumber \\
    &\quad + \mathbf{B}_t^\top\tfrac{1}{m}\sum_{i\in \mathcal{I}_t}(\mathbf{B}_t - \alpha(\mathbf{B}_t\mathbf{w}_t - \mathbf{B}_\ast \mathbf{w}_{\ast,i})\mathbf{w}_t^\top) \del_t \mathbf{w}_{t}\mathbf{w}_{t,i,1}^\top \nonumber \\
    &\quad+\mathbf{B}_t^\top\tfrac{1}{m}\sum_{i\in \mathcal{I}_t}\sum_{s = 2}^{\tau-1} (\mathbf{B}_{t,i,s}\mathbf{w}_{t,i,s}- \mathbf{B}_\ast\mathbf{w}_{\ast,i})\mathbf{w}_{t,i,s}^\top \nonumber \\
    &=\mathbf{B}_t^\top(\mathbf{B}_{t}\mathbf{w}_{t}- \mathbf{B}_\ast\mathbf{\bar{w}}_{\ast,t})\mathbf{w}_{t}^\top + \mathbf{B}_t^\top\mathbf{B}_t\del_t \mathbf{w}_{t}\tfrac{1}{m}\sum_{i\in \mathcal{I}_t}\mathbf{w}_{t,i,1}^\top 
     \nonumber \\
    &\quad  - \mathbf{B}_t^\top\tfrac{1}{m}\sum_{i\in \mathcal{I}_t}\alpha^2(\mathbf{B}_t\mathbf{w}_t - \mathbf{B}_\ast \mathbf{w}_{\ast,i})\mathbf{w}_t^\top \mathbf{B}_{t}^\top\mathbf{B}_\ast\mathbf{w}_{\ast,i}\mathbf{w}_{t,i,1}^\top\nonumber \\
    &\quad - \mathbf{B}_t^\top\tfrac{1}{m}\sum_{i\in \mathcal{I}_t}\alpha(\mathbf{B}_t\mathbf{w}_t - \mathbf{B}_\ast \mathbf{w}_{\ast,i})\mathbf{w}_t^\top \del_t \mathbf{w}_{t}\mathbf{w}_{t,i,1}^\top  \nonumber\\
    &\quad - \del_t \mathbf{B}_t^\top \mathbf{B}_\ast \mathbf{\bar{w}}_{\ast,t}\mathbf{w}_t^\top \del_t   +\mathbf{B}_t^\top\tfrac{1}{m}\sum_{i\in \mathcal{I}_t}\sum_{s = 2}^{\tau-1} (\mathbf{B}_{t,i,s}\mathbf{w}_{t,i,s}- \mathbf{B}_\ast\mathbf{w}_{\ast,i})\mathbf{w}_{t,i,s}^\top  \nonumber \\
    &= \underbrace{\mathbf{B}_t^\top(\mathbf{B}_{t}\mathbf{w}_{t}- \mathbf{B}_\ast\mathbf{\bar{w}}_{\ast,t})\mathbf{{w}}_{t}^\top + \alpha \mathbf{B}_t^\top\mathbf{B}_t\del_t \mathbf{w}_{t} \mathbf{\bar{w}}_{\ast,t}^\top \mathbf{B}_\ast^\top \mathbf{B}_t}_{=:\mathbf{E}_1}
     \nonumber \\
     &\quad -\underbrace{  \tfrac{1}{m}\sum_{i\in \mathcal{I}_t}\alpha \mathbf{B}_t^\top(\mathbf{B}_t\mathbf{w}_t - \mathbf{B}_\ast \mathbf{w}_{\ast,i})\mathbf{w}_t^\top \mathbf{w}_{t,i,1}\mathbf{w}_{t,i,1}^\top}_{=:\mathbf{E}_2}  \nonumber\\
    &\quad + \underbrace{\del_t \mathbf{B}_t^\top (\mathbf{B}_t \mathbf{w}_{t} - \mathbf{B}_\ast \mathbf{\bar{w}}_{\ast,t})\mathbf{w}_t^\top \del_t}_{=:\mathbf{E}_3} +\underbrace{\mathbf{B}_t^\top\tfrac{1}{m}\sum_{i\in \mathcal{I}_t}\sum_{s = 2}^{\tau-1} (\mathbf{B}_{t,i,s}\mathbf{w}_{t,i,s}- \mathbf{B}_\ast\mathbf{w}_{\ast,i})\mathbf{w}_{t,i,s}^\top}_{=:\mathbf{E}_4}  \label{split}
\end{align}
To get from the first to the second equation we expanded the fourth term in the first equation. We need to upper bound the spectral norm of each of these  terms. 
The matrices $\mathbf{E}_2$ and $\mathbf{E}_3$ are straightforward to control; we will take care of them shortly. For now we are concerned with $\mathbf{E}_1$. In order to control this matrix, we must use the fact that $\mathbf{w}_t$ is close to a matrix times $\mathbf{\bar{w}}_{\ast,t}$. This will allow us to subsume the dominant term from $\mathbf{E}_1$ into the negative term in \eqref{61}.
In particular, note that by $A_1(t)$, we have $\mathbf{w}_t = \alpha \mathbf{B}_{t}^\top \mathbf{B}_\ast \mathbf{\bar{w}}_{\ast,t} + \alpha\del_{t} \mathbf{B}_{t}^\top \mathbf{B}_\ast \mathbf{\bar{w}}_{\ast,t} + \mathbf{h}_t$, where $\|\mathbf{h}_t\|_2 \leq  91 \alpha^{2.5} \tau L_{\max}^3 $. 
This implies that
\begin{align}
    \mathbf{B}_t^\top(\mathbf{B}_{t}\mathbf{w}_{t}- \mathbf{B}_\ast\mathbf{\bar{w}}_{\ast,t}) &= \mathbf{B}_t^\top(\alpha \mathbf{B}_{t}\mathbf{B}_t^\top\mathbf{B}_\ast\mathbf{\bar{w}}_{\ast,t}  - \mathbf{B}_\ast\mathbf{\bar{w}}_{\ast,t}) + \alpha \del_t\mathbf{B}_t^\top \mathbf{B}_{t}\mathbf{B}_t^\top\mathbf{B}_\ast\mathbf{\bar{w}}_{\ast,t} + \mathbf{B}_t^\top \mathbf{B}_{t}\mathbf{h}_{t} \nonumber\\
    &= -\del_t \mathbf{B}_t^\top\mathbf{B}_\ast\mathbf{\bar{w}}_{\ast,t} + \alpha \del_t\mathbf{B}_t^\top \mathbf{B}_{t}\mathbf{B}_t^\top\mathbf{B}_\ast\mathbf{\bar{w}}_{\ast,t} + \mathbf{B}_t^\top \mathbf{B}_{t}\mathbf{h}_{t} \nonumber\\
    &= -\del_t^2 \mathbf{B}_t^\top\mathbf{B}_\ast\mathbf{\bar{w}}_{\ast,t} + \mathbf{B}_t^\top \mathbf{B}_{t}\mathbf{h}_{t} 
\end{align}
Making this substitution in $\mathbf{E}_1$, we obtain,
\begin{align}
    \mathbf{E}_1&=  -\del_t^2 \mathbf{B}_t^\top\mathbf{B}_\ast\mathbf{\bar{w}}_{\ast,t}\mathbf{{w}}_{t}^\top + \mathbf{B}_t^\top \mathbf{B}_{t}\mathbf{h}_{t} \mathbf{{w}}_{t}^\top + \alpha \mathbf{B}_t^\top\mathbf{B}_t\del_t \mathbf{w}_{t}  \mathbf{\bar{w}}_{\ast,t}^\top \mathbf{B}_\ast^\top \mathbf{B}_t \nonumber  \\
    &= -\del_t^2 \mathbf{B}_t^\top\mathbf{B}_\ast\mathbf{\bar{w}}_{\ast,t}\mathbf{{w}}_{t}^\top + \mathbf{B}_t^\top \mathbf{B}_{t}\mathbf{h}_{t} \mathbf{{w}}_{t}^\top  -\del_t^2 \mathbf{w}_{t}  \mathbf{\bar{w}}_{\ast,t}^\top \mathbf{B}_\ast^\top \mathbf{B}_t  + \del_t \mathbf{w}_{t}  \mathbf{\bar{w}}_{\ast,t}^\top \mathbf{B}_\ast^\top \mathbf{B}_t \nonumber \\
    &= -\del_t^2 \mathbf{B}_t^\top\mathbf{B}_\ast\mathbf{\bar{w}}_{\ast,t}\mathbf{{w}}_{t}^\top + \mathbf{B}_t^\top \mathbf{B}_{t}\mathbf{h}_{t} \mathbf{{w}}_{t}^\top  -\del_t^2 \mathbf{w}_{t}  \mathbf{\bar{w}}_{\ast,t}^\top \mathbf{B}_\ast^\top \mathbf{B}_t  \nonumber \\
    &\quad +  \del_t(\alpha\del_{t} \mathbf{B}_{t}^\top \mathbf{B}_\ast \mathbf{\bar{w}}_{\ast,t}+\mathbf{h}_t)  \mathbf{\bar{w}}_{\ast,t}^\top \mathbf{B}_\ast^\top \mathbf{B}_t + \alpha \del_t \mathbf{B}_{t}^\top \mathbf{B}_\ast \mathbf{\bar{w}}_{\ast,t}  \mathbf{\bar{w}}_{\ast,t}^\top \mathbf{B}_\ast^\top \mathbf{B}_t. \label{above}
\end{align} 
The dominant term in \eqref{above} is the last term. Specifically, we have, 
\begin{align}
    \|\mathbf{E}_1 - &\alpha \del_t \mathbf{B}_{t}^\top \mathbf{B}_\ast \mathbf{\bar{w}}_{\ast,t}  \mathbf{\bar{w}}_{\ast,t}^\top \mathbf{B}_\ast^\top \mathbf{B}_t \|\nonumber \\
    &\leq \| \del_t^2 \mathbf{B}_t^\top\mathbf{B}_\ast\mathbf{\bar{w}}_{\ast,t}\mathbf{{w}}_{t}^\top\| + \|\mathbf{B}_t^\top \mathbf{B}_{t}\mathbf{h}_{t} \mathbf{{w}}_{t}^\top\| + \|\del_t^2 \mathbf{w}_{t}  \mathbf{\bar{w}}_{\ast,t}^\top \mathbf{B}_\ast^\top \mathbf{B}_t \| \nonumber \\
    &\quad + \|\del_t(\alpha\del_{t} \mathbf{B}_{t}^\top \mathbf{B}_\ast \mathbf{\bar{w}}_{\ast,t}+\mathbf{h}_t)  \mathbf{\bar{w}}_{\ast,t}^\top \mathbf{B}_\ast^\top \mathbf{B}_t\| \nonumber \\
    &\leq \| \del_t^2 \mathbf{B}_t^\top\mathbf{B}_\ast\mathbf{\bar{w}}_{\ast,t}\mathbf{{w}}_{t}^\top\| + \|\del_t^2 \mathbf{w}_{t}  \mathbf{\bar{w}}_{\ast,t}^\top \mathbf{B}_\ast^\top \mathbf{B}_t \| +\alpha\|\del_t^2 \mathbf{B}_{t}^\top \mathbf{B}_\ast \mathbf{\bar{w}}_{\ast,t} \mathbf{\bar{w}}_{\ast,t}^\top \mathbf{B}_\ast^\top \mathbf{B}_t\| \nonumber \\
    &\quad +\|\mathbf{B}_t^\top \mathbf{B}_{t}\mathbf{h}_{t} \mathbf{{w}}_{t}^\top\| + \|\del_t\mathbf{h}_t \mathbf{\bar{w}}_{\ast,t}^\top \mathbf{B}_\ast^\top \mathbf{B}_t\| \nonumber \\
    &\leq 5.5 c_3^2 \alpha^4 \tau^2 L_{\max}^6 \kappa_{\max}^4 E_0^{-2} + 2.2\times 91 \alpha^2 \tau L_{\max}^4   + 1.1 \times 91 c_3\alpha^4 \tau^2 L_{\max}^6 \kappa_{\max}^2 E_0^{-1} \label{rrmm} \\
    &\leq  (206 + 101/c_3) \alpha^2 \tau L_{\max}^4 
    \label{rrm}
\end{align}
where \eqref{rrmm} follows by applying the Cauchy-Schwarz inequality to each of the terms in the previous inequality, and using $A_2(t), A_3(t)$, and our bound on $\mathbf{h}_t$ (from $A_1(t)$), and \eqref{rrm} follows as $\alpha$ is sufficiently small.
The last first term term can be subsumed by completing a square as follows. 
Combining \eqref{61}, \eqref{split} and \eqref{above} yields
\begin{align}
    \mathbf{B}_t^\top \mathbf{G}_t &= - \alpha \del_t \mathbf{B}_t^\top \mathbf{B}_\ast \tfrac{1}{m}\sum_{i\in\mathcal{I}_t}\mathbf{w}_{\ast,i}\mathbf{w}_{\ast,i}^\top\mathbf{B}_\ast^\top \mathbf{B}_t + \mathbf{E}_1 -  \mathbf{E}_2 +\mathbf{E}_3 + \mathbf{E}_4 \nonumber \\
    &= - \alpha \del_t \mathbf{B}_t^\top \mathbf{B}_\ast \tfrac{1}{m}\sum_{i\in\mathcal{I}_t}\mathbf{w}_{\ast,i}\mathbf{w}_{\ast,i}^\top\mathbf{B}_\ast^\top \mathbf{B}_t  + \alpha \del_t \mathbf{B}_{t}^\top \mathbf{B}_\ast \mathbf{\bar{w}}_{\ast,t}  \mathbf{\bar{w}}_{\ast,t}^\top \mathbf{B}_\ast^\top \mathbf{B}_t  \nonumber \\
    &\quad + (\mathbf{E}_1 - \alpha \del_t \mathbf{B}_{t}^\top \mathbf{B}_\ast \mathbf{\bar{w}}_{\ast,t}  \mathbf{\bar{w}}_{\ast,t}^\top \mathbf{B}_\ast^\top \mathbf{B}_t) -  \mathbf{E}_2 +\mathbf{E}_3 + \mathbf{E}_4 
    \nonumber \\
    &= - \alpha \del_t \mathbf{B}_t^\top \mathbf{B}_\ast \tfrac{1}{m}\sum_{i\in\mathcal{I}_t} (\mathbf{w}_{\ast,i} - \mathbf{\bar{w}}_{\ast,t})(\mathbf{w}_{\ast,i} - \mathbf{\bar{w}}_{\ast,t})^\top\mathbf{B}_\ast^\top \mathbf{B}_t  \nonumber \\
    &\quad + (\mathbf{E}_1 - \alpha \del_t \mathbf{B}_{t}^\top \mathbf{B}_\ast \mathbf{\bar{w}}_{\ast,t}  \mathbf{\bar{w}}_{\ast,t}^\top \mathbf{B}_\ast^\top \mathbf{B}_t) -  \mathbf{E}_2 +\mathbf{E}_3 + \mathbf{E}_4  
    \nonumber \\
    &=  - \tfrac{1}{2\alpha^2} \del_t \mathbf{P}_t + \tfrac{1}{\alpha^2}\mathbf{Z}'_t
\end{align}
where $\mathbf{P}_t = 2\alpha^3\mathbf{B}_t^\top \mathbf{B}_\ast \tfrac{1}{m}\sum_{i\in\mathcal{I}_t} (\mathbf{w}_{\ast,i} - \mathbf{\bar{w}}_{\ast,t})(\mathbf{w}_{\ast,i} - \mathbf{\bar{w}}_{\ast,t})^\top\mathbf{B}_\ast^\top \mathbf{B}_t$ and 
\begin{align}
   \mathbf{Z}'_t &:= \alpha^2(\mathbf{E}_1 - \alpha \del_t \mathbf{B}_{t}^\top \mathbf{B}_\ast \mathbf{\bar{w}}_{\ast,t}  \mathbf{\bar{w}}_{\ast,t}^\top \mathbf{B}_\ast^\top \mathbf{B}_t)  -  \alpha^2\mathbf{E}_2 +\alpha^2\mathbf{E}_3 + \alpha^2\mathbf{E}_4 \nonumber 
\end{align}
we have performed the desired decomposition; it remains to show that $\lambda_{\min}(\mathbf{P}_t)$ and $\|\mathbf{Z}'_t\|_2$ scale appropriately.
 First we lower bound $\lambda_{\min}(\mathbf{P}_t)$.
 \begin{align}
     \lambda_{\min}(\mathbf{P}_t) &= \lambda_{\min}\left(2\alpha^3\mathbf{B}_t^\top \mathbf{B}_\ast \tfrac{1}{m}\sum_{i\in\mathcal{I}_t} (\mathbf{w}_{\ast,i} - \mathbf{\bar{w}}_{\ast,t})(\mathbf{w}_{\ast,i} - \mathbf{\bar{w}}_{\ast,t})^\top\mathbf{B}_\ast^\top \mathbf{B}_t\right) \nonumber \\
     &\geq 2 \alpha^3 \sigma_{\min}(\mathbf{B}_t^\top \mathbf{B}_\ast)^2 \lambda_{\min}\left(\tfrac{1}{m}\sum_{i\in\mathcal{I}_t} (\mathbf{w}_{\ast,i} - \mathbf{\bar{w}}_{\ast,t})(\mathbf{w}_{\ast,i} - \mathbf{\bar{w}}_{\ast,t})^\top\right) \nonumber \\
     &\geq 0.2 \alpha^2 E_0  \lambda_{\min}\left(\tfrac{1}{m}\sum_{i\in\mathcal{I}_t} (\mathbf{w}_{\ast,i} - \mathbf{\bar{w}}_{\ast,t})(\mathbf{w}_{\ast,i} - \mathbf{\bar{w}}_{\ast,t})^\top\right) \label{e0}  \\
     &\geq  0.2 \alpha^2 E_0 \lambda_{\min}\left(\tfrac{1}{M }\sum_{i=1}^M (\mathbf{w}_{\ast,i} - \mathbf{\bar{w}}_{\ast})(\mathbf{w}_{\ast,i} - \mathbf{\bar{w}}_{\ast})^\top\right) \nonumber \\
     &\quad - 0.2 \alpha^2 E_0  \bigg\|\tfrac{1}{m}\sum_{i\in\mathcal{I}_t} \mathbf{w}_{\ast,i} \mathbf{w}_{\ast,i}^\top - \tfrac{1}{M}\sum_{i'=1}^M \mathbf{w}_{\ast,i'} \mathbf{w}_{\ast,i'}^\top \bigg\| \nonumber \\
     &\quad -  0.2 \alpha^2 E_0 \big\| \mathbf{w}_{\ast,t} \mathbf{w}_{\ast,t}^\top -  \mathbf{\bar{w}}_{\ast} \mathbf{\bar{w}}_{\ast}^\top \big\| \nonumber \\
     &\geq  0.15 \alpha^2 E_0 \mu^2 - 6\alpha^4 E_0 L_{\max}^4 \label{e00}  \\
     &\geq 0.15 \alpha^2 E_0 \mu^2 \label{e000}
 \end{align} 
 where \eqref{e0} follows by Lemma \ref{lem:distlb}, \eqref{e00} follows by Assumption \ref{assump:td} and $A_0$, and \eqref{e000} follows as $\alpha^2\leq \frac{1}{120\kappa_\ast^2 L_{\max}^2}$. 
 Now we upper bound $\|\mathbf{Z}'_t\|_2$. We have already upper bounded $\|\mathbf{E}_1 - \alpha \del_t \mathbf{B}_{t}^\top \mathbf{B}_\ast \mathbf{\bar{w}}_{\ast,t}  \mathbf{\bar{w}}_{\ast,t}^\top \mathbf{B}_\ast^\top \mathbf{B}_t\|$ in \eqref{rrm}. We next upper bound $\|\mathbf{E}_2\|_2$ and $\|\mathbf{E}_3\|_2$ by $A_2(t), A_{3}(t),$ and $A_{2,i,t}(1)$. We have
 \begin{align}
     \|\alpha^2\mathbf{E}_2\|_2 &\leq 32
     \alpha^4  L_{\max}^4 \nonumber \\ 
     \|\alpha^2\mathbf{E}_3\|_2 &\leq 28c_3^2 \alpha^6 \tau^2 L_{\max}^6\kappa_{\max}^2 E_0^{-2} \nonumber  
 \end{align}
using the triangle and Cauchy-Schwarz inequalities. Now we turn to $\|\alpha^2\mathbf{E}_4\|_2$. Recall that $\mathbf{E}_4$ is the sum of local gradients across all clients and all local updates beyond the first local update. We show that these gradients are sufficiently small such that $\|\del_{t+1}\|_2$ cannot grow beyond the desired threshold.
Recall that $\mathbf{E}_4 = \sum_{s=2}^{\tau-1} \mathbf{B}_t^\top \mathbf{e}_{t,i,s}\mathbf{w}_{t,i,s}^\top$. To bound this sum it is critical to control the evolution of $\mathbf{e}_{t,i,s}$. The idea is to split $\mathbf{e}_{t,i,s}$ into its projection onto $\col(\mathbf{B}_{t,i,s-1})\approx \col(\mathbf{B}_t)$ and its projection onto $\col(\mathbf{B}_{t,i,s-1})^\perp\approx \col(\mathbf{B}_t)^\perp.$ Then, we can show that the magnitude of the projection onto $\col(\mathbf{B}_{t,i,s-1})$ is going to zero very fast (the head is quickly learned, meaning it fits the product as much as it can with what it has to work with, i.e. $\col(\mathbf{B}_{t,i,s-1})$). On the other hand, the magnitude of the projection onto $\col(\mathbf{B}_{t,i,s-1})^\perp$ is slowly going to zero, since this reducing this error requires changing the representation and the representation changes slower than the head. The saving grace is that this error is proportional to $\dist(\mathbf{B}_{t,i,s-1},\mathbf{B}_\ast)$, which for all $s$ is linearly converging to zero with $t$. 

To show this, pick any $i\in\mathcal{I}_t$ and let 
$\mathbf{\hat{B}}_{t,i,s},\mathbf{R}_{t,i,s}$ denote the QR-factorization of $\mathbf{B}_{t,i,s}$. Define $\tilde{\del}_{t,i,s-1} \coloneqq \mathbf{\hat{B}}_{t,i,s-1}\mathbf{\hat{B}}_{t,i,s-1}^\top - \alpha \mathbf{B}_{t,i,s-1}\mathbf{B}_{t,i,s-1}^\top$ and $\omega_{t,i,s-1} \coloneqq \alpha \mathbf{w}_{t,i,s-1}^\top \del_{t,i,s-1}\mathbf{w}_{t,i,s} + \alpha^2 \mathbf{w}_{t,i,s-1}^\top \mathbf{B}_{t,i,s-1}^\top \mathbf{B}_\ast \mathbf{w}_{\ast,i}$. By expanding $\mathbf{e}_{t,i,s}$, we find
\begin{align}
    &\mathbf{e}_{t,i,s}\nonumber \\
    &= (\mathbf{I}_d - \alpha \mathbf{B}_{t,i,s-1}\mathbf{B}_{t,i,s-1}^\top - \alpha \mathbf{w}_{t,i,s-1}^\top \del_{t,i,s-1}\mathbf{w}_{t,i,s}\mathbf{I}_d - \alpha^2 \mathbf{w}_{t,i,s-1}^\top \mathbf{B}_{t,i,s-1}^\top \mathbf{B}_\ast \mathbf{w}_{\ast,i}\mathbf{I}_d)\mathbf{e}_{t,i,s-1}  \nonumber \\
    &= (\mathbf{I}_d - \mathbf{\hat{B}}_{t,i,s-1}\mathbf{\hat{B}}_{t,i,s-1}^\top)\mathbf{e}_{t,i,s-1}  + \big( \tilde{\del}_{t,i,s-1} - \omega_{t,i,s-1}\mathbf{I}_d\big)\mathbf{e}_{t,i,s-1}  \nonumber \\
    &= (\mathbf{I}_d - \mathbf{\hat{B}}_{t,i,s-1}\mathbf{\hat{B}}_{t,i,s-1}^\top)\mathbf{e}_{t,i,s-1}  + \big(\tilde{\del}_{t,i,s-1}  - \omega_{t,i,s-1}\mathbf{I}_d\big) (\mathbf{I}_d - \mathbf{\hat{B}}_{t,i,s-2}\mathbf{\hat{B}}_{t,i,s-2}^\top)\mathbf{e}_{t,i,s-2}  \nonumber \\
    &\quad + \big(\tilde{\del}_{t,i,s-1} - \omega_{t,i,s-1}\mathbf{I}_d\big) \big(\tilde{\del}_{t,i,s-2}  - \omega_{t,i,s-2}\mathbf{I}_d\big)\mathbf{e}_{t,i,s-2} 
\end{align}
Therefore, 
\begin{align}
    \mathbf{B}_t^\top \mathbf{e}_{t,i,s}
    &= \mathbf{B}_t^\top (\mathbf{I}_d - \mathbf{\hat{B}}_{t,i,s-1}\mathbf{\hat{B}}_{t,i,s-1}^\top)\mathbf{e}_{t,i,s-1} \nonumber \\
   & \quad + \mathbf{B}_t^\top\big(\tilde{\del}_{t,i,s-1}  - \omega_{t,i,s-1}\mathbf{I}_d\big) (\mathbf{I}_d - \mathbf{\hat{B}}_{t,i,s-2}\mathbf{\hat{B}}_{t,i,s-2}^\top)\mathbf{e}_{t,i,s-2}  \nonumber \\
    &\quad + \mathbf{B}_t^\top\big(\tilde{\del}_{t,i,s-1} - \omega_{t,i,s-1}\mathbf{I}_d\big) \big(\tilde{\del}_{t,i,s-2}  - \omega_{t,i,s-2}\mathbf{I}_d\big)\mathbf{e}_{t,i,s-2}  \label{3terms}
\end{align}
For the first term, we have
\begin{align}
    \|\mathbf{B}_t^\top (\mathbf{I}_d -& \mathbf{\hat{B}}_{t,i,s-1}\mathbf{\hat{B}}_{t,i,s-1}^\top)\mathbf{e}_{t,i,s-1}\|_2\nonumber \\
    &\leq \|\mathbf{B}_{t,i,s-1}^\top (\mathbf{I}_d - \mathbf{\hat{B}}_{t,i,s-1}\mathbf{\hat{B}}_{t,i,s-1}^\top)\mathbf{e}_{t,i,s-1}\|_2\nonumber \\
    &\quad + \|(\mathbf{B}_t - \mathbf{B}_{t,i,s-1})^\top (\mathbf{I}_d - \mathbf{\hat{B}}_{t,i,s-1}\mathbf{\hat{B}}_{t,i,s-1}^\top)\mathbf{e}_{t,i,s-1}\|_2 \nonumber \\
    &= \|(\mathbf{B}_t - \mathbf{B}_{t,i,s-1})^\top (\mathbf{I}_d - \mathbf{\hat{B}}_{t,i,s-1}\mathbf{\hat{B}}_{t,i,s-1}^\top)\mathbf{e}_{t,i,s-1}\|_2  \label{eq}\\
    &\leq \| (\mathbf{I}_d - \mathbf{\hat{B}}_{t,i,s-1}\mathbf{\hat{B}}_{t,i,s-1}^\top)\mathbf{e}_{t,i,s-1}\|_2  \sum_{r=1}^{s-1}\|\mathbf{B}_{t,i,r} - \mathbf{B}_{t,i,r-1} \|_2   \label{suum} \\
    &\leq 7 \| (\mathbf{I}_d - \mathbf{\hat{B}}_{t,i,s-1}\mathbf{\hat{B}}_{t,i,s-1}^\top)\mathbf{B}_\ast \mathbf{w}_{\ast,i}\|_2  \alpha^{1.5} \tau L_{\max}^2  \label{suuum} \\
    &\leq 8\alpha^{1.5} \tau L_{\max}^3  \dist_t \label{laast}
\end{align}
where \eqref{eq} follows since $\mathbf{B}_{t,i,s-1}^\top(\mathbf{I}_d - \mathbf{\hat{B}}_{t,i,s-1}\mathbf{\hat{B}}_{t,i,s-1}^\top) = \mathbf{0}$, \eqref{suum} follows using the Cauchy-Schwarz and triangle inequalities, \eqref{suuum} follows using that $(\mathbf{I}_d - \mathbf{\hat{B}}_{t,i,s-1}\mathbf{\hat{B}}_{t,i,s-1}^\top)\mathbf{B}_{t,i,s-1} = \mathbf{0}$ and applying $A_{2,t,i}(\tau)$ and $A_{3,t,i}(\tau)$, \eqref{laast} follows by the fact that $\|(\mathbf{I}_d - \mathbf{\hat{B}}_{t,i,s-1}\mathbf{\hat{B}}_{t,i,s-1}^\top)\mathbf{B}_\ast\|= \dist(\mathbf{B}_{t,i,s}, \mathbf{B}_\ast) \leq 1.1\dist_t$ by $A_{4,t,i}(\tau)$.
For the second term in \eqref{3terms}, note that
\begin{align}
    |\omega_{t,i,s-1}| &\leq \alpha|\mathbf{w}_{t,i,s-1}^\top \del_{t,i,s-1}\mathbf{w}_{t,i,s-1}| + \alpha^2|\mathbf{w}_{t,i,s-1}^\top \mathbf{B}_{t,i,s-1}^\top\mathbf{B}_\ast\mathbf{w}_{\ast,i}|\nonumber \\
    &\leq  8 c_3  \alpha^{4} \tau L_{\max}^4 \kappa_{\max}^2 E_0^{-1} + 2.2 \alpha^{1.5} L_{\max}^2 \nonumber\\
    &\leq 3\alpha^{2} L_{\max}^2
\end{align}
As a result, we have
\begin{align}
   \| \mathbf{B}_t^\top\big(\tilde{\del}_{t,i,s-1}&  - \omega_{t,i,s-1}\mathbf{I}_d\big) (\mathbf{I}_d - \mathbf{\hat{B}}_{t,i,s-2}\mathbf{\hat{B}}_{t,i,s-2}^\top)\mathbf{e}_{t,i,s-2}\|_2 \nonumber \\
   &\leq  \| \mathbf{B}_t^\top \tilde{\del}_{t,i,s-1}  (\mathbf{I}_d - \mathbf{\hat{B}}_{t,i,s-2}\mathbf{\hat{B}}_{t,i,s-2}^\top)\mathbf{e}_{t,i,s-2}\|_2   \nonumber \\
    &\quad + |\omega_{t,i,s-1} |\| \mathbf{B}_t^\top (\mathbf{I}_d - \mathbf{\hat{B}}_{t,i,s-2}\mathbf{\hat{B}}_{t,i,s-2}^\top)\mathbf{e}_{t,i,s-2}\|_2    \nonumber \\
    &\leq  \| \mathbf{B}_t^\top \tilde{\del}_{t,i,s-1}  (\mathbf{I}_d - \mathbf{\hat{B}}_{t,i,s-2}\mathbf{\hat{B}}_{t,i,s-2}^\top)\mathbf{e}_{t,i,s-2}\|_2   \nonumber \\
    &\quad + 3.3 \alpha^{1.5} L_{\max}^2\| (\mathbf{I}_d - \mathbf{\hat{B}}_{t,i,s-2}\mathbf{\hat{B}}_{t,i,s-2}^\top)\mathbf{B}_{\ast}\mathbf{w}_{\ast,i}\|_2    \nonumber \\
    &\leq  \| \mathbf{B}_t^\top \tilde{\del}_{t,i,s-1}  (\mathbf{I}_d - \mathbf{\hat{B}}_{t,i,s-2}\mathbf{\hat{B}}_{t,i,s-2}^\top)\mathbf{e}_{t,i,s-2}\|_2   \nonumber \\
    &\quad + 3.7 \alpha^{1.5} L_{\max}^3 \dist_t  \label{42} 
\end{align}
where \eqref{42} follows by $A_{4,t,i}(\tau)$, and
\begin{align}
    \| &\mathbf{B}_t^\top \tilde{\del}_{t,i,s-1}  (\mathbf{I}_d - \mathbf{\hat{B}}_{t,i,s-2}\mathbf{\hat{B}}_{t,i,s-2}^\top)\mathbf{e}_{t,i,s-2}\|_2 \nonumber\\
    &\leq \| \mathbf{B}_t^\top \mathbf{\hat{B}}_{t,i,s-1}   \mathbf{\hat{B}}_{t,i,s-1}^\top(\mathbf{I}_d - \mathbf{\hat{B}}_{t,i,s-2}\mathbf{\hat{B}}_{t,i,s-2}^\top)\mathbf{e}_{t,i,s-2}\|_2  \nonumber \\
    &\quad + \alpha\| \mathbf{B}_t^\top \mathbf{{B}}_{t,i,s-1}   \mathbf{{B}}_{t,i,s-1}^\top(\mathbf{I}_d - \mathbf{\hat{B}}_{t,i,s-2}\mathbf{\hat{B}}_{t,i,s-2}^\top)\mathbf{e}_{t,i,s-2}\|_2  \nonumber \\
    &= \| \mathbf{B}_t^\top \mathbf{\hat{B}}_{t,i,s-1} (\mathbf{R}^{-1}_{t,i,s-1})^\top  \mathbf{{B}}_{t,i,s-1}^\top(\mathbf{I}_d - \mathbf{\hat{B}}_{t,i,s-2}\mathbf{\hat{B}}_{t,i,s-2}^\top)\mathbf{e}_{t,i,s-2}\|_2  \nonumber \\
    &\quad + \alpha\| \mathbf{B}_t^\top \mathbf{{B}}_{t,i,s-1}   \mathbf{{B}}_{t,i,s-1}^\top(\mathbf{I}_d - \mathbf{\hat{B}}_{t,i,s-2}\mathbf{\hat{B}}_{t,i,s-2}^\top)\mathbf{e}_{t,i,s-2}\|_2  \nonumber \\
    &=\alpha \| \mathbf{B}_t^\top \mathbf{\hat{B}}_{t,i,s-1} (\mathbf{R}^{-1}_{t,i,s-1})^\top  \mathbf{{w}}_{t,i,s-2} \mathbf{e}_{t,i,s-2}^\top(\mathbf{I}_d - \mathbf{\hat{B}}_{t,i,s-2}\mathbf{\hat{B}}_{t,i,s-2}^\top)\mathbf{e}_{t,i,s-2}\|_2  \nonumber\\
    &\quad + \alpha^2\| \mathbf{B}_t^\top \mathbf{{B}}_{t,i,s-1} \mathbf{{w}}_{t,i,s-2}  \mathbf{{e}}_{t,i,s-2}^\top(\mathbf{I}_d - \mathbf{\hat{B}}_{t,i,s-2}\mathbf{\hat{B}}_{t,i,s-2}^\top)\mathbf{e}_{t,i,s-2}\|_2 \label{hgh} \\
    &\leq 44 \alpha^{1.5} L_{\max}^3 \dist_t \label{44}
\end{align}
where \eqref{hgh} follows since  $\mathbf{B}_{t,i,s-2}(\mathbf{I}_d - \mathbf{\hat{B}}_{t,i,s-2}\mathbf{\hat{B}}_{t,i,s-2}^\top) = \mathbf{0}$ and \eqref{44} follows using the Cauchy-Schwarz inequality, $A_3(t)$, $A_{2,t,i}(\tau)$, $A_{3,t,i}(\tau)$, and $A_{4,t,i}(\tau)$. Next, recalling that $\mathbf{\hat{B}}_{t,i,s}\mathbf{R}_{t,i,s}$ is the QR-factorization of $\mathbf{{B}}_{t,i,s}$, we have, for any $s$,
\begin{align}
    \|\tilde{\del}_{t,i,s-1}\| &= \|\mathbf{\hat{B}}_{t,i,s-1} \mathbf{\hat{B}}_{t,i,s-1}^\top - \alpha \mathbf{{B}}_{t,i,s-1} \mathbf{{B}}_{t,i,s-1}^\top\| \nonumber \\
    &\leq \|\mathbf{\hat{B}}_{t,i,s-1}(\mathbf{I}_k - \alpha \mathbf{R}_{t,i,s-1} \mathbf{R}_{t,i,s-1}^\top)\mathbf{\hat{B}}_{t,i,s-1}^\top \| \nonumber \\
    &\leq \|\mathbf{I}_k - \alpha \mathbf{R}_{t,i,s-1} \mathbf{R}_{t,i,s-1}^\top \| \nonumber \\
    &\leq \max(|1 - \alpha \sigma_{\min}^2(\mathbf{B}_{t,i,s-1})|, |1 - \alpha \sigma_{\max}^2(\mathbf{B}_{t,i,s-1})| ) \label{ggg} \\
    &\leq \max(|1 - \alpha \tfrac{1 - \|\del_{t,i,s-1}\|}{\alpha}|, |1 - \alpha  \tfrac{1 + \|\del_{t,i,s-1}\|}{\alpha})| )  \nonumber \\
    &= \|\del_{t,i,s-1}\|  \nonumber \\ 
    &\leq 2c_3\alpha^2 \tau L_{\max}^2 \kappa_{\max}^2 E_0^{-1} \label{gg}
\end{align}
where \eqref{ggg} follows by Weyl's inequality and \eqref{gg} follows 
by $A_{3,t,i}(s-1)$. Furthermore,
$
  \|\tilde{\del}_{t,i,s-1} +\omega_{t,i,s-1}\mathbf{I}_d\| \leq 2c_3\alpha^2 \tau L_{\max}^2 \kappa_{\max}^2 E_0^{-1} + 3 \alpha^2L_{\max}^2 \leq 3 c_3\alpha^2 \tau L_{\max}^2 \kappa_{\max}^2 E_0^{-1}
$
for any $s$. 
Thus, for the third term in \eqref{3terms}, for any $s>2$,
\begin{align}
    \| &\mathbf{B}_t^\top\big(\tilde{\del}_{t,i,s-1} - \omega_{t,i,s-1}\mathbf{I}_d\big) \big(\tilde{\del}_{t,i,s-2}  - \omega_{t,i,s-2}\mathbf{I}_d\big)\mathbf{e}_{t,i,s-2}  \|_2 \nonumber \\
    &=  \| \mathbf{B}_t^\top\big(\tilde{\del}_{t,i,s-1} - \omega_{t,i,s-1}\mathbf{I}_d\big) \big(\tilde{\del}_{t,i,s-2}  - \omega_{t,i,s-2}\mathbf{I}_d\big)( \tilde{\del}_{t,i,s-3}  - \omega_{t,i,s-3}\mathbf{I}_d)\mathbf{e}_{t,i,s-3} \|_2\nonumber \\
    &\quad + \|\mathbf{B}_t^\top\big(\tilde{\del}_{t,i,s-1} - \omega_{t,i,s-1}\mathbf{I}_d\big) \big(\tilde{\del}_{t,i,s-2}  - \omega_{t,i,s-2}\mathbf{I}_d\big)( \mathbf{I}_d - \mathbf{\hat{B}}_{t,i,s-3}\mathbf{\hat{B}}_{t,i,s-3}^\top)\mathbf{e}_{t,i,s-3} \|_2  \nonumber \\
    &\leq  \| \mathbf{B}_t^\top\big(\tilde{\del}_{t,i,s-1} - \omega_{t,i,s-1}\mathbf{I}_d\big) \big(\tilde{\del}_{t,i,s-2}  - \omega_{t,i,s-2}\mathbf{I}_d\big)( \tilde{\del}_{t,i,s-3}  - \omega_{t,i,s-3}\mathbf{I}_d)\mathbf{e}_{t,i,s-3} \|_2 \nonumber \\
    &\quad +  10 c_3^2\alpha^{3.5} \tau^2 L^5_{\max} \kappa_{\max}^4 E_0^{-2} \dist_t  \nonumber \\
    &\vdots \nonumber \\
    &\leq \bigg\|\mathbf{B}_t^\top\prod_{r=1}^s\big(\tilde{\del}_{t,i,s-r} - \omega_{t,i,s-r}\mathbf{I}_d\big) \mathbf{e}_{t,i,s-r}\bigg\|_2 \nonumber \\
    &\quad +  10 c_3^2 \alpha^{3.5} \tau^2 L^5_{\max} \kappa^4_{\max}E_0^{-2} \dist_t \sum_{r=0}^s (3c_3)^r\alpha^{2r}\tau^r L_{\max}^{2r} \kappa_{\max}^{2r} E_0^{-r}\nonumber \\
    &\leq  \bigg\|\mathbf{B}_t^\top\prod_{r=1}^s\big(\tilde{\del}_{t,i,s-r} - \omega_{t,i,s-r}\mathbf{I}_d\big) \mathbf{e}_{t,i,s-r}\bigg\|_2  + \frac{ 10c_3^2 \alpha^{3.5} \tau^2 L^5_{\max} \kappa_{\max}^4 E_{0}^{-2} \dist_t}{1 - 3c_3 \alpha^{2} \tau L^2_{\max} \kappa_{\max}^2 E_0^{-1}}  \nonumber  \\
    &\leq  3.5 \times (3 c_3)^s \alpha^{2s -0.5}\tau^s{L_{\max}^{2s+1} }\kappa_{\max}^{2s}E_0^{-s} + {  15 c_3^2 \alpha^{3.5} \tau^2 L_{\max}^5 \kappa_{\max}^4 E_0^{-2} \dist_t} \label{45}
\end{align}
and for $s=2$ we have
\begin{align}
    \| \mathbf{B}_t^\top\big(\tilde{\del}_{t,i,s-1} - \omega_{t,i,s-1}\mathbf{I}_d\big)& \big(\tilde{\del}_{t,i,s-2}  - \omega_{t,i,s-2}\mathbf{I}_d\big)\mathbf{e}_{t,i,s-2}  \|_2  \nonumber  \\
    &\leq  3.5\times (3c_3)^2 \alpha^{3.5} \tau^{2}  L_{\max}^{5}\kappa_{\max}^{4}E_0^{-2}.  \label{46}
\end{align}
Thus, using $\|\mathbf{w}_{t,i,s}\|_2\leq 2 \sqrt{\alpha}L_{\max}$ with \eqref{3terms}, \eqref{laast}, \eqref{42}, \eqref{44}, \eqref{45}, and \eqref{46}, we have
\begin{align}
  \|  \alpha^2 \mathbf{E}_4 \|_2 &\leq 2 \alpha^{2.5} L_{\max} \sum_{s=2}^{\tau-1} \bigg(3.5 \times (3c_3)^2\alpha^{2s-0.5}\tau^s  L_{\max}^{2s+1}\kappa_{\max}^{2s} E_0^{-2s}\nonumber\\
  &\quad \quad \quad \quad \quad \quad \quad\quad  + 15c_3^2\alpha^{3.5} \tau^2  L_{\max}^5 \kappa_{\max}^4 \dist_t  +  (8 \tau + 48)\alpha^{1.5} L_{\max}^3 \dist_t \bigg)\nonumber \\
  &\leq \frac{ 63 c_3^2 \alpha^{6}\tau^{2} L_{\max}^6 \kappa_{\max}^4 E_0^{-2} }{1 - 3 c_s \alpha^{2}\tau L_{\max}^2 \kappa_{\max}^2 E_0^{-1} } + (30c_3^2\alpha^{6} \tau^3 L_{\max}^6 \kappa_{\max}^4E_0^{-2}   + 66\alpha^{4} \tau^2 L_{\max}^4 )\dist_t  \nonumber \\
  &\leq 
  90 c_3^2 \alpha^{6}\tau^{2} L_{\max}^6 \kappa_{\max}^4 E_0^{-2} +  94\alpha^{4} \tau^2 L_{\max}^4 \dist_t \nonumber
\end{align}
where the last inequality follows by choice of $\alpha$.
Combining all terms, we obtain
\begin{align}
    \|\del_{t+1}\|_2 &\leq (1 - 0.15 \alpha^2 E_0 \mu^2 )\|\del_t\|_2 +  (206 + 101/c_3) \alpha^4 \tau L_{\max}^4 + 32
     \alpha^4  L_{\max}^4  \nonumber
    \\
    &\quad   + 118 c_3^2\alpha^{6}\tau^{2} L_{\max}^6 \kappa_{\max}^4 E_0^{-2}  +  94\alpha^{4} \tau^2 L_{\max}^4 \dist_t \nonumber \\
    &\leq  (1 - 0.15 \alpha^2 E_0 \mu^2 )\|\del_t\|_2 +  (340 + 101/c_3) \alpha^4 \tau L_{\max}^4   +  94\alpha^{4} \tau^2 L_{\max}^4  \dist_t \label{leq} \\
    &\quad  \vdots \nonumber \\
    &\leq (1 - 0.15 \alpha^2 E_0 \mu^2 )^t\|\del_0\|_2 \nonumber \\
    &\quad + \sum_{t' = 1}^t (1 - 0.15 \alpha^2 E_0 \mu^2)^{t-t'} \big( (340 + 101/c_3) \alpha^4 \tau L_{\max}^4 + 94\alpha^{4} \tau^2 L_{\max}^4 \dist_{t'}\big)  \nonumber \\ 
    &\leq \|\del_0\|_2 + (340 + 101/c_3) \alpha^4 \tau L_{\max}^4 \sum_{t' = 1}^t (1 - 0.15\alpha^2 E_0 \mu^2)^{t-t'}    \nonumber \\
    &\quad + 94\alpha^{4} \tau^2 L_{\max}^4\sum_{t' = 1}^t (1 - 0.04 \alpha^2 \tau E_0 \mu^2 )^{t'} \label{ms13} \\
      &\leq \|\del_0\|_2 + ({7(340 + 101/c_3)+25\times 94} ) \alpha^2 \tau   L_{\max}^2 \kappa_{\max}^2 E_0^{-1} \nonumber \\
      &\leq  \alpha^2 \tau   L_{\max}^2 \kappa_{\max}^2 + ({4730+ 101/c_3} ) \alpha^2 \tau   L_{\max}^2 \kappa_{\max}^2 E_0^{-1} \nonumber \\
      &\leq c_3\alpha^2 \tau L_{\max}^2 \kappa_{\max}^2 
\end{align}
where \eqref{leq} follows by choice of $\alpha \leq \tfrac{1 - \delta_0}{c_3\sqrt{\tau} L_{\max}^2 \kappa_{\max}^2}\leq \tfrac{E_0}{c_3\sqrt{\tau} L_{\max}^2 \kappa_{\max}^2 }$, \eqref{ms13} follows from $A_5(t)$, and the last inequality is due to $c_3 = 4800$. 
 
 
 
 
 
 
 
 
 
 
 
 \end{proof}

\begin{lemma} \label{lem:a3}
$\cap_{i\in \mathcal{I}_t}\big(A_{1,t,i}(\tau) \cap A_{2,t,i}(\tau)\cap A_{3,t,i}(\tau)\big) \cap A_2(t)  \implies A_{4}(t+1)$.
\end{lemma}

\begin{proof}
We have \begin{align}
    \mathbf{B}_{t+1} &=  \mathbf{{B}}_{t} \bigg(\frac{1}{m} \sum_{i\in \mathcal{I}_t}   \prod_{s=0}^{\tau-1}(\mathbf{I}_k - \alpha \mathbf{w}_{t,i,s}\mathbf{w}_{t,i,s}^\top)\bigg)\nonumber \\
   & \quad+ \mathbf{\hat{B}}_\ast  \bigg(\frac{\alpha}{m} \sum_{i\in \mathcal{I}_t}  \mathbf{w}_{\ast,i}  \sum_{s=0}^{\tau-1}\mathbf{w}_{t,i,s}^\top \prod_{r=s+1}^{\tau-1}(\mathbf{I}_k - \alpha \mathbf{w}_{t,i,r}\mathbf{w}_{t,i,r}^\top) \bigg) \nonumber 
    \end{align}
    which implies
    \begin{align}
 \mathbf{B}_{\ast,\perp}^\top \mathbf{B}_{t+1} &= \mathbf{B}_{\ast,\perp}^\top\mathbf{{B}}_{t} (\mathbf{I} - \alpha \mathbf{w}_t \mathbf{w}_t^\top) \left(\frac{1}{m} \sum_{i\in \mathcal{I}_t}   \prod_{s=1}^{\tau-1}(\mathbf{I}_k - \alpha \mathbf{w}_{t,i,s}\mathbf{w}_{t,i,s}^\top)\right) \label{ssslast}
\end{align}
We can expand the right product of $\mathbf{B}_t (\mathbf{I} - \alpha \mathbf{w}_t \mathbf{w}_t^\top)$ using the binomial expansion as follows:
\begin{align}
\frac{1}{m} \sum_{i\in \mathcal{I}_t}   \prod_{s=1}^{\tau-1}(\mathbf{I}_k - \alpha \mathbf{w}_{t,i,s}\mathbf{w}_{t,i,s}^\top) &=  \mathbf{I}_k - \frac{\alpha}{m} \sum_{i\in \mathcal{I}_t}   \sum_{s = 1}^{\tau-1} \mathbf{w}_{t,i,s}\mathbf{w}_{t,i,s}^\top\nonumber \\
&\quad + \frac{\alpha^2}{m} \sum_{i\in \mathcal{I}_t}  \sum_{s=1}^{\tau-1} \sum_{s^{(1)}=s+1}^{\tau-1}  \mathbf{w}_{t,i,s}\mathbf{w}_{t,i,s}^\top\mathbf{w}_{t,i,s^{(1)}}\mathbf{w}_{t,i,s^{(1)}}^\top \nonumber \\
&\quad - \dots + \sign(\tau)\frac{\alpha^{\tau}}{m} \sum_{i\in \mathcal{I}_t}  \prod_{s=1}^{\tau-1}  \mathbf{w}_{t,i,s}\mathbf{w}_{t,i,s}^\top  \nonumber
\end{align}
Recall that each $\|\mathbf{w}_{t,i,s}\|_2 \leq 2\sqrt{\alpha}L_{\max}$. Thus, after the identity, the spectral norm of the first set of summations has spectral norm at most $4\alpha^2 \tau L_{\max}^2 $, the second set has specral norm at most $16\alpha^4 \tau^2 L_{\max}^4$, and so on. 
We in fact use the first set of summations as a negative term, and bound all subsequent sets of summations as errors, exploiting the fact that their norms are geometrically decaying. In particular, we have:
\begin{align}
    \left\|\frac{1}{m} \sum_{i\in \mathcal{I}_t}   \prod_{s=1}^{\tau-1}(\mathbf{I}_k - \alpha \mathbf{w}_{t,i,s}\mathbf{w}_{t,i,s}^\top)  \right\|_2 &\leq \left\| \mathbf{I}_k - \frac{\alpha}{m} \sum_{i\in \mathcal{I}_t}    \sum_{s = 1}^{\tau-1} \mathbf{w}_{t,i,s}\mathbf{w}_{t,i,s}^\top \right\|_2 + \sum_{z=2}^{\tau-1}(4\alpha^2 \tau L_{\max}^2)^z \nonumber \\ 
    &\leq  \left\| \mathbf{I}_k - \frac{\alpha}{m} \sum_{i\in \mathcal{I}_t}    \sum_{s = 1}^{\tau-1} \mathbf{w}_{t,i,s}\mathbf{w}_{t,i,s}^\top \right\|_2 +  \frac{4\alpha^4 \tau^2 L_{\max}^4}{1 - 4\alpha^2 \tau L_{\max}^2}\nonumber \\
    &\leq  \left\| \mathbf{I}_k - \frac{\alpha}{m} \sum_{i\in \mathcal{I}_t}    \sum_{s = 1}^{\tau-1} \mathbf{w}_{t,i,s}\mathbf{w}_{t,i,s}^\top \right\|_2 +  {5\alpha^4 \tau^2 L_{\max}^4} \label{sslast}
    \end{align}
Next, we use $\|\mathbf{w}_{t,i,s} - \alpha \mathbf{B}_{t,i,s-1}^\top \mathbf{B}_\ast \mathbf{w}_{\ast,i}\| \leq  4 c_3 \alpha^{2.5} \tau L_{\max}^3 \kappa_{\max}^2 E_0^{-1} $ for all $s\geq 1$
to obtain 
\begin{align}
   \left\| \mathbf{I}_k - \frac{\alpha}{m} \sum_{i\in \mathcal{I}_t}    \sum_{s = 1}^{\tau-1} \mathbf{w}_{t,i,s}\mathbf{w}_{t,i,s}^\top \right\|_2 &\leq  \left\| \mathbf{I}_k - \frac{\alpha^3}{m} \sum_{i\in \mathcal{I}_t}    \sum_{s = 1}^{\tau-1} \mathbf{B}_{t,i,s-1}^\top\mathbf{B}_\ast\mathbf{w}_{\ast,i}\mathbf{w}_{\ast,i}^\top\mathbf{B}_\ast^\top\mathbf{B}_{t,i,s-1} \right\|_2 \nonumber \\
   &\quad  + 13 c_3 \alpha^4 \tau^2 L_{\max}^4 \kappa_{\max}^2 E_0^{-1} \label{slast}
   \end{align}
 Next, we have for any $s-1 \in \{1,\dots,\tau-1\}$, 
 \begin{align}
    \| \mathbf{B}_{t,i,s-1} - \mathbf{B}_t \|_2 = \| \mathbf{B}_{t,i,s-1} - \mathbf{B}_{t,i,0} \|_2  &\leq \sum_{r = 1}^{s-1} \| \mathbf{B}_{t,i,r} - \mathbf{B}_{t,i,r-1} \|_2 \nonumber \\
    &\leq \alpha \sum_{r = 1}^{s-1} \|\mathbf{e}_{t,i,r-1}\|_2 \|\mathbf{w}_{t,i,r-1}\|_2 \nonumber \\
    &\leq 7\alpha^{1.5} (s-1) L_{\max}^2 \nonumber
 \end{align}
 Thus,
   \begin{align}
 \bigg\| \mathbf{I}_k - \frac{\alpha^3}{m} \sum_{i\in \mathcal{I}_t}    &\sum_{s = 1}^{\tau-1} \mathbf{B}_{t,i,s-1}^\top\mathbf{B}_\ast\mathbf{w}_{\ast,i}\mathbf{w}_{\ast,i}^\top\mathbf{B}_\ast^\top\mathbf{B}_{t,i,s-1} \bigg\|_2  \nonumber \\
 &\leq  \left\| \mathbf{I}_k - \frac{\alpha^3}{m} \sum_{i\in \mathcal{I}_t}    \sum_{s = 1}^{\tau-1} \mathbf{B}_{t}^\top\mathbf{B}_\ast\mathbf{w}_{\ast,i}\mathbf{w}_{\ast,i}^\top\mathbf{B}_\ast^\top\mathbf{B}_{t} \right\|_2 \nonumber \\
 &\quad  + \left\|\frac{\alpha^3}{m} \sum_{i\in \mathcal{I}_t}    \sum_{s = 1}^{\tau-1} (\mathbf{B}_{t} - \mathbf{B}_{t,i,s-1} )^\top\mathbf{B}_\ast\mathbf{w}_{\ast,i}\mathbf{w}_{\ast,i}^\top\mathbf{B}_\ast^\top\mathbf{B}_{t,i,s-1} \right\|_2  \nonumber \\
 &\quad + \left\|\frac{\alpha^3}{m} \sum_{i\in \mathcal{I}_t}    \sum_{s = 1}^{\tau-1} \mathbf{B}_{t}^\top\mathbf{B}_\ast\mathbf{w}_{\ast,i}\mathbf{w}_{\ast,i}^\top\mathbf{B}_\ast^\top(\mathbf{B}_t - \mathbf{B}_{t,i,s-1}) \right\|_2  \nonumber\\
 &\leq  \left\| \mathbf{I}_k - \frac{\alpha^3}{m} \sum_{i\in \mathcal{I}_t}    \sum_{s = 1}^{\tau-1} \mathbf{B}_{t}^\top\mathbf{B}_\ast\mathbf{w}_{\ast,i}\mathbf{w}_{\ast,i}^\top\mathbf{B}_\ast^\top\mathbf{B}_{t} \right\|_2 + 16 \alpha^4 (\tau-1)^2  L_{\max}^4. \label{comb}
\end{align}
Furtheromre,
\begin{align}
    &\left\| \mathbf{I}_k - \tfrac{\alpha^3}{m} \sum_{i\in \mathcal{I}_t}    \sum_{s = 1}^{\tau-1} \mathbf{B}_{t}^\top\mathbf{B}_\ast\mathbf{w}_{\ast,i}\mathbf{w}_{\ast,i}^\top\mathbf{B}_\ast^\top\mathbf{B}_{t} \right\|_2 \nonumber \\
    &= \left\| \mathbf{I}_k - \tfrac{\alpha^3(\tau-1)}{m} \sum_{i\in \mathcal{I}_t}    \mathbf{B}_{t}^\top\mathbf{B}_\ast\mathbf{w}_{\ast,i}\mathbf{w}_{\ast,i}^\top\mathbf{B}_\ast^\top\mathbf{B}_{t} \right\|_2  \nonumber \\
    &\leq \left\| \mathbf{I}_k - \tfrac{\alpha^3(\tau-1)}{M} \sum_{i=1}^M    \mathbf{B}_{t}^\top\mathbf{B}_\ast\mathbf{w}_{\ast,i}\mathbf{w}_{\ast,i}^\top\mathbf{B}_\ast^\top\mathbf{B}_{t} \right\|_2 \nonumber \\
    &\quad + {\alpha^3(\tau-1)} \|\mathbf{B}_{t}^\top\mathbf{B}_\ast\|^2_2 \left\| \tfrac{1}{m}\sum_{i\in\mathcal{I}_t} \left(\mathbf{w}_{\ast,i}\mathbf{w}_{\ast,i}^\top - \tfrac{1}{M}\sum_{i=1}^M \mathbf{w}_{\ast,i}\mathbf{w}_{\ast,i}^\top\right)\right\|_2\nonumber \\
    &\leq \left\| \mathbf{I}_k - \tfrac{\alpha^3(\tau-1)}{M} \sum_{i=1}^M    \mathbf{B}_{t}^\top\mathbf{B}_\ast\mathbf{w}_{\ast,i}\mathbf{w}_{\ast,i}^\top\mathbf{B}_\ast^\top\mathbf{B}_{t} \right\|_2 + 6{\alpha^4(\tau-1)} L_{\max}^4,\label{cond} 
    \end{align}
noting that \eqref{cond} follows since we are conditioning on the event $A_0$.   Finally, 
    \begin{align}
   &\left\| \mathbf{I}_k - \tfrac{\alpha^3(\tau-1)}{M} \sum_{i=1}^M    \mathbf{B}_{t}^\top\mathbf{B}_\ast\mathbf{w}_{\ast,i}\mathbf{w}_{\ast,i}^\top\mathbf{B}_\ast^\top\mathbf{B}_{t} \right\|_2\nonumber \\ &\leq 1 -  \alpha^3 (\tau-1) \sigma_{\min}^2(\mathbf{B}_t^\top\mathbf{B}_\ast)\sigma_{\min}\left(\tfrac{1}{M} \sum_{i=1}^M    \mathbf{w}_{\ast,i}\mathbf{w}_{\ast,i}^\top\right) \nonumber \\
    &\leq 1 -  \alpha^3 (\tau-1) \sigma_{\min}^2(\mathbf{B}_t^\top\mathbf{B}_\ast)\sigma_{\min}\left(\tfrac{1}{M} \sum_{i=1}^M    \mathbf{w}_{\ast,i}\mathbf{w}_{\ast,i}^\top - \mathbf{\bar{w}}_{\ast}\mathbf{\bar{w}}_{\ast}^\top\right) \nonumber \\
    &= 1 -  \alpha^3 (\tau-1) \sigma_{\min}(\mathbf{B}_t^\top\mathbf{B}_\ast)^2\sigma_{\min}\left(\tfrac{1}{M} \sum_{i=1}^M    (\mathbf{w}_{\ast,i}-\mathbf{\bar{w}}_{\ast} )(\mathbf{w}_{\ast,i}-\mathbf{\bar{w}}_{\ast})^\top  \right) \nonumber \\
    &\leq 1 -  0.1\alpha^2 (\tau-1)E_0\mu \label{nnnnn} \\
    &\leq 1 -  0.05\alpha^2 \tau E_0\mu  \label{nn}
\end{align}
where \eqref{nnnnn} follows by Lemma \ref{lem:distlb} and Assumption \ref{assump:td}, and \eqref{nn} follows since $\tau \geq 2$.
Combining \eqref{nn}, \eqref{cond}, \eqref{comb}, \eqref{slast}, \eqref{sslast}, and \eqref{ssslast}, we obtain
\begin{align}
    \| \mathbf{B}_{\ast,\perp}^\top \mathbf{B}_{t+1} \|_2 &\leq \bigg(1 - 0.05\alpha^2 \tau E_0 \mu^2 + 6\alpha^4 (\tau-1)L_{\max}^4 + 16 \alpha^4 (\tau-1)^2  L_{\max}^4   \nonumber\\
    &\quad \quad \quad \quad + 13 c_3 \alpha^4 \tau^2 L_{\max}^4 \kappa_{\max}^2 E_0^{-1} +  {5\alpha^4 \tau^2 L_{\max}^4} \bigg) \|\mathbf{B}_{\ast,\perp}^\top \mathbf{B}_{t} \|_2 \nonumber \\
    &\leq \big(1 - 0.05\alpha^2 \tau E_0 \mu^2 + (24/c_3^2) \alpha^2 \tau E_0 \mu^2  + (13/c_3)  \alpha^2 \tau E_0 \mu^2 \big) \|\mathbf{B}_{\ast,\perp}^\top \mathbf{B}_{t} \|_2 \nonumber \\
    &\leq (1 - 0.04 \alpha^2 \tau E_0 \mu^2 ) \|\mathbf{B}_{\ast,\perp}^\top \mathbf{B}_{t} \|_2
\end{align}
using $\alpha^2 \tau \leq \tfrac{(1-\delta_0)^2}{c_3^2\tau\kappa_{\max}^4 L_{\max}^2}\leq \tfrac{E_0^2}{c_3^3\tau\kappa_{\max}^4 L_{\max}^2 }$ and $c_3> 1305$.
\end{proof}









\begin{lemma} \label{lem:a4}
$A_3(t+1) \cap A_4(t+1) \cap A_5(t)  \implies A_{5}(t+1)$
\end{lemma}

\begin{proof}
We use the contraction of $\|\mathbf{B}_{\ast,\perp}^\top \mathbf{B}_{s}\|_2$ ($A_4(t)$) and the fact that $\|\del_{s}\|_2$ is small for all $s\in [t]$ ($A_5(t)$), as in Lemma \ref{lem:distlb}, to obtain
\begin{align}
    \dist_{t+1}&= \|\mathbf{B}_{\ast,\perp}^\top \mathbf{\hat{B}}_{t+1}\|_2 \nonumber \\
    &\leq\tfrac{1}{ \sigma_{\min}(\mathbf{B}_{t+1})}\|\mathbf{B}_{\ast,\perp}^\top \mathbf{{B}}_{t+1}\|_2\nonumber \\
    &\leq\tfrac{1}{ \sigma_{\min}(\mathbf{{B}}_{t+1})}(1-0.04\alpha^2 \tau E_0 \mu^2 )\|\mathbf{B}_{\ast,\perp}^\top \mathbf{{B}}_{t}\|_2\nonumber \\
    &\quad \vdots \nonumber \\
    &\leq \tfrac{1}{ \sigma_{\min}(\mathbf{{B}}_{t+1})}(1-0.04\alpha^2 \tau E_0 \mu^2 )^t\|\mathbf{B}_{\ast,\perp}^\top \mathbf{{B}}_{0}\|_2\nonumber \\
    &\leq  \tfrac{\sigma_{\max}(\mathbf{B}_0)}{ \sigma_{\min}(\mathbf{{B}}_{t+1})}(1-0.04\alpha^2 \tau E_0 \mu^2 )^t\delta_0 \nonumber \\
    &\leq \tfrac{\sqrt{1+\|\del_{0}\|_2}/\sqrt{\alpha}}{ \sqrt{1-\|\del_{t+1}\|_2}/\sqrt{\alpha}}(1-0.04\alpha^2 \tau E_0 \mu^2 )^t\delta_0 \nonumber 
\end{align}
Now, we argue as in \eqref{uhu} (with $\|\del_t\|$ replaced by $\|\del_{t+1}$ without anything changing in the analysis) to find
$\tfrac{\sqrt{1+\|\del_{0}\|_2}/\sqrt{\alpha}}{ \sqrt{1-\|\del_{t+1}\|_2}/\sqrt{\alpha}}\delta_0\leq \tfrac{2+\delta_0}{3}\leq 1 $. Thus $ \dist_{t+1}\leq (1-0.04\alpha^2 \tau E_0 \mu^2 )^t$ as desired.
\end{proof}




\subsection{Proof of Proposition \ref{prop:dgd}} \label{app:prooflb}

\begin{proposition}[Distributed GD lower bound]
Suppose we are in the setting described in Section \ref{sec:linear} and $k > 1$. Then for any set of ground-truth heads $\{\mathbf{w}_{\ast,i}\}_{i=1}^M$, full-rank initialization $\mathbf{B}_0\in\mathbb{R}^{d\times k}$, initial distance $\delta_0 \in (0,1/2]$, step size $\alpha > 0$, and number of rounds $T$, there exists $\mathbf{B}_\ast\in \mathcal{O}^{d \times k}$ such that $\dist(\mathbf{B}_0, \mathbf{B}_\ast) = \delta_0$ and $\dist(\mathbf{B}_T^{\text{D-GD}}, \mathbf{B}_\ast) \geq \delta_0$, where $\mathbf{B}_T^{\text{D-GD}}\equiv \mathbf{B}_T^{\text{D-GD}}(\mathbf{B}_0,\mathbf{B}_\ast,\{\mathbf{w}_{\ast,i}\}_{i=1}^M, \alpha)$ is the result of D-GD with step size $\alpha$ and initialization $\mathbf{B}_0$ on the system with ground-truth representation $\mathbf{B}_\ast$ and ground-truth heads $\{\mathbf{w}_{\ast,i}\}_{i=1}^M$.
\end{proposition}

\begin{proof}
Recall that $\mathbf{B}_T^{\text{D-GD}}(\mathbf{B}_0,\mathbf{B}_\ast,\{\mathbf{w}_{\ast,i}\}_{i=1}^M, \alpha)$ is the result of D-GD with step size $\alpha$ and initialization $\mathbf{B}_0$ on the system with ground-truth representation $\mathbf{B}_\ast$ and ground-truth heads $\{\mathbf{w}_{\ast,i}\}_{i=1}^M$.

There are two disjoint cases: (1) for all $\mathbf{B}_\ast \in \mathcal{B} \coloneqq \{\mathbf{B}\in \mathcal{O}^{d\times k}:\dist(\mathbf{B}_0,\mathbf{B}) = \delta_0, \mathbf{B}\mathbf{\bar{w}}_{\ast}\in\col(\mathbf{B}_0)\}$, $\dist(\mathbf{B}_T^{\text{D-GD}}(\mathbf{B}_0,\mathbf{B}_\ast,\{\mathbf{w}_{\ast,i}\}_{i=1}^M, \alpha), \mathbf{B}_\ast)\geq 0.7 \delta_0$, or (2) there exists some $\mathbf{B}_\ast \in \mathcal{B}$ such that \\ $\dist(\mathbf{B}_T^{\text{D-GD}}(\mathbf{B}_0,\mathbf{B}_\ast,\{\mathbf{w}_{\ast,i}\}_{i=1}^M, \alpha), \mathbf{B}_\ast)< 0.7 \delta_0$. If case (1) holds then the proof is complete. Otherwise, let 
$\mathbf{B}_\ast \in \mathcal{B}$ such that  $\dist(\mathbf{B}_T^{\text{D-GD}}(\mathbf{B}_0,\mathbf{B}_\ast,\{\mathbf{w}_{\ast,i}\}_{i=1}^M, \alpha), \mathbf{B}_\ast)< 0.7 \delta_0$. We will show that there exists another $\mathbf{B}_{\ast'}\in \mathcal{B}$ such that  $\dist(\mathbf{B}_T^{\text{D-GD}}(\mathbf{B}_0,\mathbf{B}_{\ast'},\{\mathbf{w}_{\ast,i}\}_{i=1}^M, \alpha), \mathbf{B}_{\ast'})\geq  \delta_0$, so D-GD cannot guarantee to recover the ground-truth representation, completing the proof.

Consider case (2). Without loss of generality we can write $\mathbf{B}_0 =\tfrac{1}{\|\mathbf{\bar{w}}_{\ast}\|} \mathbf{B}_\ast \mathbf{\bar{w}}_\ast \mathbf{v}_0^\top + \mathbf{\tilde{B}}_0\mathbf{\tilde{V}}_0^\top$ for some $\mathbf{v}_0\in \mathbb{R}^k: \|\mathbf{v}_0\|=1$, $\mathbf{\tilde{B}}_0\in \mathcal{O}^{d \times k-1}: \mathbf{\tilde{B}}_0^\top \mathbf{B}_\ast \mathbf{\bar{w}}_\ast=\mathbf{0}$, $\mathbf{\tilde{V}}_0\in \mathcal{O}^{k \times k-1}:\mathbf{\tilde{V}}_0^\top \mathbf{v}_0=\mathbf{0}$ using the SVD,  since $\mathbf{B}_\ast \mathbf{\bar{w}}_\ast \in \col(\mathbf{B}_0)$. Likewise, we can write $\mathbf{B}_\ast = \tfrac{1}{\|\mathbf{\bar{w}}_{\ast}\|^2} \mathbf{B}_\ast \mathbf{\bar{w}}_\ast \mathbf{\bar{w}}_\ast^\top + \mathbf{\tilde{B}}_\ast\mathbf{\tilde{V}}_\ast^\top$ for some $\mathbf{\tilde{B}}_\ast\in \mathcal{O}^{d \times k-1}: \mathbf{\tilde{B}}_\ast^\top \mathbf{{B}}_\ast\mathbf{\bar{w}}_{\ast}=\mathbf{0}$, $\mathbf{\tilde{V}}_\ast\in \mathcal{O}^{k \times k-1}: \mathbf{\tilde{V}}_\ast^\top \mathbf{\bar{w}}_\ast=\mathbf{0}$.
Using these decompositions, we can see that
\begin{align}
    \mathbf{B}_\ast \mathbf{B}_\ast^\top &= \big(\tfrac{1}{\|\mathbf{\bar{w}}_{\ast}\|^2} \mathbf{B}_\ast \mathbf{\bar{w}}_\ast \mathbf{\bar{w}}_\ast^\top\! +\! \mathbf{\tilde{B}}_\ast\mathbf{\tilde{V}}_\ast^\top\big) \big(\tfrac{1}{\|\mathbf{\bar{w}}_{\ast}\|^2} \mathbf{B}_\ast \mathbf{\bar{w}}_\ast \mathbf{\bar{w}}_\ast^\top\! +\! \mathbf{\tilde{B}}_\ast\mathbf{\tilde{V}}_\ast^\top)^\top \nonumber \\
  &= \tfrac{1}{\|\mathbf{\bar{w}}_{\ast}\|^4} \mathbf{B}_\ast \mathbf{\bar{w}}_\ast \mathbf{\bar{w}}_\ast^\top \mathbf{\bar{w}}_\ast \mathbf{\bar{w}}_\ast^\top \mathbf{B}_\ast^\top    + \tfrac{1}{\|\mathbf{\bar{w}}_{\ast}\|^2} \mathbf{B}_\ast \mathbf{\bar{w}}_\ast \mathbf{\bar{w}}_\ast^\top \mathbf{\tilde{V}}_\ast \mathbf{\tilde{B}}_\ast^\top
  + \tfrac{1}{\|\mathbf{\bar{w}}_{\ast}\|^2} \mathbf{\tilde{B}}_\ast\mathbf{\tilde{V}}_\ast^\top \mathbf{\bar{w}}_\ast \mathbf{\bar{w}}_\ast^\top \mathbf{B}_\ast^\top \nonumber \\
  &\quad + \mathbf{\tilde{B}}_\ast\mathbf{\tilde{V}}_\ast^\top \mathbf{\tilde{V}}_\ast \mathbf{\tilde{B}}_\ast^\top \nonumber \\
  &= \tfrac{1}{\|\mathbf{\bar{w}}_{\ast}\|^2} \mathbf{B}_\ast \mathbf{\bar{w}}_\ast  \mathbf{\bar{w}}_\ast^\top \mathbf{B}_\ast^\top  
 + \mathbf{\tilde{B}}_\ast \mathbf{\tilde{B}}_\ast^\top \nonumber 
\end{align}
and
\begin{align}
  \delta_0&\coloneqq  \dist(\mathbf{B}_{0},\mathbf{B}_\ast) \nonumber \\
  &\coloneqq \|(\mathbf{I}_d - \mathbf{B}_\ast \mathbf{B}_\ast^\top)\mathbf{B}_0\|\nonumber \\
  &= \big\|\big(\mathbf{I}_d - \tfrac{1}{\|\mathbf{\bar{w}}_{\ast}\|^2} \mathbf{B}_\ast \mathbf{\bar{w}}_\ast  \mathbf{\bar{w}}_\ast^\top \mathbf{B}_\ast^\top  
 - \mathbf{\tilde{B}}_\ast \mathbf{\tilde{B}}_\ast^\top
  \big)\big(\tfrac{1}{\|\mathbf{\bar{w}}_{\ast}\|} \mathbf{B}_\ast \mathbf{\bar{w}}_\ast \mathbf{v}_0^\top \!+\! \mathbf{\tilde{B}}_0\mathbf{\tilde{V}}_0^\top  \big)\big\|\nonumber \\
  &=  \big\|\big(\mathbf{I}_d - \tfrac{1}{\|\mathbf{\bar{w}}_{\ast}\|^2} \mathbf{B}_\ast \mathbf{\bar{w}}_\ast  \mathbf{\bar{w}}_\ast^\top \mathbf{B}_\ast^\top  
 - \mathbf{\tilde{B}}_\ast \mathbf{\tilde{B}}_\ast^\top
  \big)\tfrac{1}{\|\mathbf{\bar{w}}_{\ast}\|} \mathbf{B}_\ast \mathbf{\bar{w}}_\ast \mathbf{v}_0^\top  \nonumber\\
  &\quad + \big(\mathbf{I}_d - \tfrac{1}{\|\mathbf{\bar{w}}_{\ast}\|^2} \mathbf{B}_\ast \mathbf{\bar{w}}_\ast  \mathbf{\bar{w}}_\ast^\top \mathbf{B}_\ast^\top  
 - \mathbf{\tilde{B}}_\ast \mathbf{\tilde{B}}_\ast^\top
  \big)\mathbf{\tilde{B}}_0\mathbf{\tilde{V}}_0^\top  \big\|\nonumber \\
  &=  \big\| \big(\mathbf{I}_d - \tfrac{1}{\|\mathbf{\bar{w}}_{\ast}\|^2} \mathbf{B}_\ast \mathbf{\bar{w}}_\ast  \mathbf{\bar{w}}_\ast^\top \mathbf{B}_\ast^\top  
 - \mathbf{\tilde{B}}_\ast \mathbf{\tilde{B}}_\ast^\top
  \big)\mathbf{\tilde{B}}_0\mathbf{\tilde{V}}_0^\top  \big\|\nonumber \\
  &= \big\| \big(\mathbf{I}_d   
 - \mathbf{\tilde{B}}_\ast \mathbf{\tilde{B}}_\ast^\top
  \big)\mathbf{\tilde{B}}_0\mathbf{\tilde{V}}_0^\top  \big\|\nonumber \\
  &= \big\| \big(\mathbf{I}_d   
 - \mathbf{\tilde{B}}_\ast \mathbf{\tilde{B}}_\ast^\top
  \big)\mathbf{\tilde{B}}_0 \big\|\nonumber \\
  &= \dist(\mathbf{\tilde{B}}_0,  \mathbf{\tilde{B}}_\ast). \label{dist}
\end{align}
Next, let $\mathbf{B}_{\ast'} = \tfrac{1}{\|\mathbf{\bar{w}}_{\ast}\|^2} \mathbf{B}_\ast \mathbf{\bar{w}}_\ast \mathbf{\bar{w}}_\ast^\top +  (2 \mathbf{\tilde{B}}_0 \mathbf{\tilde{B}}_0^\top \mathbf{\tilde{B}}_\ast -  \mathbf{\tilde{B}}_\ast)\mathbf{\tilde{V}}_\ast^\top$. We first check that $\mathbf{B}_{\ast'}\in \mathcal{O}^{d\times k}$:
\begin{align}
    \mathbf{B}_{\ast'}^\top \mathbf{B}_{\ast'}&= \big(\tfrac{1}{\|\mathbf{\bar{w}}_{\ast}\|^2} \mathbf{B}_\ast \mathbf{\bar{w}}_\ast \mathbf{\bar{w}}_\ast^\top +  (2 \mathbf{\tilde{B}}_0 \mathbf{\tilde{B}}_0^\top \mathbf{\tilde{B}}_\ast -  \mathbf{\tilde{B}}_\ast)\mathbf{\tilde{V}}_\ast^\top\big)^\top \big(\tfrac{1}{\|\mathbf{\bar{w}}_{\ast}\|^2} \mathbf{B}_\ast \mathbf{\bar{w}}_\ast \mathbf{\bar{w}}_\ast^\top +  (2 \mathbf{\tilde{B}}_0 \mathbf{\tilde{B}}_0^\top \mathbf{\tilde{B}}_\ast -  \mathbf{\tilde{B}}_\ast)\mathbf{\tilde{V}}_\ast^\top\big)\nonumber \\
    &= \tfrac{1}{\|\mathbf{\bar{w}}_{\ast}\|^2} \mathbf{\bar{w}}_\ast \mathbf{\bar{w}}_\ast^\top  + \mathbf{\tilde{V}}_\ast (2 \mathbf{\tilde{B}}_0 \mathbf{\tilde{B}}_0^\top \mathbf{\tilde{B}}_\ast -  \mathbf{\tilde{B}}_\ast)^\top (2 \mathbf{\tilde{B}}_0 \mathbf{\tilde{B}}_0^\top \mathbf{\tilde{B}}_\ast -  \mathbf{\tilde{B}}_\ast) \mathbf{\tilde{V}}_\ast^\top \label{zero}\\
    &= \tfrac{1}{\|\mathbf{\bar{w}}_{\ast}\|^2} \mathbf{\bar{w}}_\ast \mathbf{\bar{w}}_\ast^\top + \mathbf{\tilde{V}}_\ast (4 \mathbf{\tilde{B}}_\ast^\top\mathbf{\tilde{B}}_0\mathbf{\tilde{B}}_0^\top\mathbf{\tilde{B}}_0 \mathbf{\tilde{B}}_0^\top \mathbf{\tilde{B}}_\ast - 4\mathbf{\tilde{B}}_\ast^\top\mathbf{\tilde{B}}_0 \mathbf{\tilde{B}}_0^\top \mathbf{\tilde{B}}_\ast +  \mathbf{\tilde{B}}_\ast^\top\mathbf{\tilde{B}}_\ast) \mathbf{\tilde{V}}_\ast^\top  \nonumber \\
    &= \tfrac{1}{\|\mathbf{\bar{w}}_{\ast}\|^2} \mathbf{\bar{w}}_\ast \mathbf{\bar{w}}_\ast^\top  + \mathbf{\tilde{V}}_\ast \mathbf{\tilde{V}}_\ast^\top\nonumber \\
    &= [\mathbf{\tilde{V}}_\ast, \tfrac{1}{\|\mathbf{\bar{w}}_{\ast}\|} \mathbf{\bar{w}}_\ast ]  [\mathbf{\tilde{V}}_\ast, \tfrac{1}{\|\mathbf{\bar{w}}_{\ast}\|} \mathbf{\bar{w}}_\ast ]^\top\nonumber \\
    &= \mathbf{I}_k \label{bg}
\end{align}
as desired, where \eqref{zero} follows since $\mathbf{\tilde{B}}_\ast^\top \mathbf{B}_\ast \mathbf{\bar{w}}_\ast= \mathbf{\tilde{B}}_0^\top \mathbf{B}_\ast \mathbf{\bar{w}}_\ast=\mathbf{0}$, and \eqref{bg} follows since $ [\mathbf{\tilde{V}}_\ast, \tfrac{1}{\|\mathbf{\bar{w}}_{\ast}\|} \mathbf{\bar{w}}_\ast ] \in \mathcal{O}^{k \times k}$ by the definition of the SVD.
Furthermore, 
\begin{align}
    \dist(\mathbf{B}_0, \mathbf{B}_{\ast'}) &= \|(\mathbf{I}_d - \mathbf{B}_0\mathbf{B}_0^\top)\mathbf{B}_{\ast'} \|\nonumber \\
    &= \|(\mathbf{I}_d \!-\!  \tfrac{1}{\|\mathbf{\bar{w}}_{\ast}\|^2} \mathbf{B}_\ast \mathbf{\bar{w}}_\ast  \mathbf{\bar{w}}_\ast^\top \mathbf{B}_\ast^\top \! -\! \mathbf{\tilde{B}}_0\mathbf{\tilde{B}}_0^\top)\big(\tfrac{1}{\|\mathbf{\bar{w}}_{\ast}\|^2} \mathbf{B}_\ast \mathbf{\bar{w}}_\ast \mathbf{\bar{w}}_\ast^\top +  (2 \mathbf{\tilde{B}}_0 \mathbf{\tilde{B}}_0^\top \mathbf{\tilde{B}}_\ast -  \mathbf{\tilde{B}}_\ast)\mathbf{\tilde{V}}_\ast^\top\big)\|\nonumber \\
    &= \|(\mathbf{I}_d - \tfrac{1}{\|\mathbf{\bar{w}}_{\ast}\|^2} \mathbf{B}_\ast \mathbf{\bar{w}}_\ast  \mathbf{\bar{w}}_\ast^\top \mathbf{B}_\ast^\top  - \mathbf{\tilde{B}}_0\mathbf{\tilde{B}}_0^\top)\tfrac{1}{\|\mathbf{\bar{w}}_{\ast}\|^2} \mathbf{B}_\ast \mathbf{\bar{w}}_\ast \mathbf{\bar{w}}_\ast^\top \nonumber \\
    &\quad + (\mathbf{I}_d - \tfrac{1}{\|\mathbf{\bar{w}}_{\ast}\|^2} \mathbf{B}_\ast \mathbf{\bar{w}}_\ast  \mathbf{\bar{w}}_\ast^\top \mathbf{B}_\ast^\top  - \mathbf{\tilde{B}}_0\mathbf{\tilde{B}}_0^\top) (2 \mathbf{\tilde{B}}_0 \mathbf{\tilde{B}}_0^\top \mathbf{\tilde{B}}_\ast -  \mathbf{\tilde{B}}_\ast)\mathbf{\tilde{V}}_\ast^\top\|\nonumber \\
    &= \| (\mathbf{I}_d - \tfrac{1}{\|\mathbf{\bar{w}}_{\ast}\|^2} \mathbf{B}_\ast \mathbf{\bar{w}}_\ast  \mathbf{\bar{w}}_\ast^\top \mathbf{B}_\ast^\top  - \mathbf{\tilde{B}}_0\mathbf{\tilde{B}}_0^\top) (2 \mathbf{\tilde{B}}_0 \mathbf{\tilde{B}}_0^\top \mathbf{\tilde{B}}_\ast -  \mathbf{\tilde{B}}_\ast)\mathbf{\tilde{V}}_\ast^\top\|\nonumber \\
    &= \| (\mathbf{I}_d - \mathbf{\tilde{B}}_0\mathbf{\tilde{B}}_0^\top) (2 \mathbf{\tilde{B}}_0 \mathbf{\tilde{B}}_0^\top \mathbf{\tilde{B}}_\ast -  \mathbf{\tilde{B}}_\ast)\mathbf{\tilde{V}}_\ast^\top\|\nonumber \\
    &= \| (\mathbf{I}_d - \mathbf{\tilde{B}}_0\mathbf{\tilde{B}}_0^\top)  \mathbf{\tilde{B}}_\ast \mathbf{\tilde{V}}_\ast^\top\|\nonumber \\
    &= \| (\mathbf{I}_d - \mathbf{\tilde{B}}_0\mathbf{\tilde{B}}_0^\top)  \mathbf{\tilde{B}}_\ast\|\nonumber \\
    &= \dist(\mathbf{\tilde{B}}_\ast, \mathbf{\tilde{B}}_0) \nonumber \\
    &=\delta_0 \label{fhf}
\end{align}
where \eqref{fhf} follows from \eqref{dist}. Moreover, $\mathbf{B}_{\ast'}\mathbf{\bar{w}}_{\ast} = \mathbf{B}_{\ast}\mathbf{\bar{w}}_{\ast} \in \col(\mathbf{B}_0)$, thus $\mathbf{B}_{\ast'}\in \mathcal{B}$. Next, 
\begin{align}
    \dist(\mathbf{B}_\ast, \mathbf{B}_{\ast'}) &= \|(\mathbf{I}_d - \mathbf{B}_\ast\mathbf{B}_\ast^\top)\mathbf{B}_{\ast'} \|\nonumber \\
    &= \|(\mathbf{I}_d \!-\!  \tfrac{1}{\|\mathbf{\bar{w}}_{\ast}\|^2} \mathbf{B}_\ast \mathbf{\bar{w}}_\ast  \mathbf{\bar{w}}_\ast^\top \mathbf{B}_\ast^\top \! -\! \mathbf{\tilde{B}}_\ast\mathbf{\tilde{B}}_\ast^\top)\big(\tfrac{1}{\|\mathbf{\bar{w}}_{\ast}\|^2} \mathbf{B}_\ast \mathbf{\bar{w}}_\ast \mathbf{\bar{w}}_\ast^\top +  (2 \mathbf{\tilde{B}}_0 \mathbf{\tilde{B}}_0^\top \mathbf{\tilde{B}}_\ast -  \mathbf{\tilde{B}}_\ast)\mathbf{\tilde{V}}_\ast^\top\big)\|\nonumber \\ 
    &= \|(\mathbf{I}_d \! -\! \mathbf{\tilde{B}}_\ast\mathbf{\tilde{B}}_\ast^\top)  (2 \mathbf{\tilde{B}}_0 \mathbf{\tilde{B}}_0^\top \mathbf{\tilde{B}}_\ast \!-\!  \mathbf{\tilde{B}}_\ast)\| \nonumber \\
    &= 2\|(\mathbf{I}_d \! -\! \mathbf{\tilde{B}}_\ast\mathbf{\tilde{B}}_\ast^\top)  \mathbf{\tilde{B}}_0 \mathbf{\tilde{B}}_0^\top \mathbf{\tilde{B}}_\ast\| \nonumber \\
    &\geq 2 \|(\mathbf{I}_d \! -\! \mathbf{\tilde{B}}_\ast\mathbf{\tilde{B}}_\ast^\top)  \mathbf{\tilde{B}}_0\|\sigma_{\min}( \mathbf{\tilde{B}}_0^\top \mathbf{\tilde{B}}_\ast) \nonumber \\
    &= 2 \dist(\mathbf{\tilde{B}}_0,\mathbf{\tilde{B}}_\ast)\sqrt{1 - \dist^2( \mathbf{\tilde{B}}_0,\mathbf{\tilde{B}}_\ast ) } \label{swrt} \\
    &= 2 \delta_0\sqrt{1-\delta_0} \nonumber
\end{align}
where \eqref{swrt} follows since $\sigma_{\min}^2(\mathbf{\tilde{B}}_1^\top \mathbf{\tilde{B}}_2 ) + \sigma_{\max}^2( (\mathbf{I}_d \! -\! \mathbf{\tilde{B}}_1\mathbf{\tilde{B}}_1^\top)  \mathbf{\tilde{B}}_2) = 1$ for any $\mathbf{\tilde{B}}_1,\mathbf{\tilde{B}}_2\in\mathcal{O}^{d,k-1}$.

Note that for D-GD the global update for the representation is 
\begin{align}
    \mathbf{B}_{t+1} = \mathbf{B}_t - \frac{\alpha}{M}\sum_{i=1}^M \nabla_{\mathbf{B}} f_i(\mathbf{B}_t, \mathbf{w}_t) = \mathbf{B}_t - \alpha (\mathbf{B}_t \mathbf{w}_t -\mathbf{B}_\ast \mathbf{\bar{w}}_{\ast}) \mathbf{w}_t^\top,
\end{align}
and similarly, the update for the head is $\mathbf{w}_{t+1} = \mathbf{w}_t - \alpha\mathbf{B}_t^\top(\mathbf{B}_t\mathbf{w}_t - \mathbf{B}_\ast \mathbf{\bar{w}}_{\ast})$. Thus, the behavior of D-GD is indistinguishable in the settings with ground-truth representations $\mathbf{B}_{\ast'}, \mathbf{B}_\ast$ since   $\mathbf{B}_{\ast'} \mathbf{\bar{w}}_\ast = \mathbf{B}_\ast \mathbf{\bar{w}}_\ast$.
In particular, $\mathbf{B}_T^{\text{D-GD}}(\mathbf{B}_0,\mathbf{B}_{\ast},\{\mathbf{w}_{\ast,i}\}_{i=1}^M, \alpha) = \mathbf{B}_T^{\text{D-GD}}(\mathbf{B}_0,\mathbf{B}_{\ast'},\{\mathbf{w}_{\ast,i}\}_{i=1}^M, \alpha)$
 Using this equality along with the triangle inequality yields
\begin{align}
    \dist(\mathbf{B}_T^{\text{D-GD}}(\mathbf{B}_0,\mathbf{B}_{\ast'},\{\mathbf{w}_{\ast,i}\}_{i=1}^M, \alpha), \mathbf{B}_{\ast'}  ) &= \dist(\mathbf{B}_T^{\text{D-GD}}(\mathbf{B}_0,\mathbf{B}_{\ast},\{\mathbf{w}_{\ast,i}\}_{i=1}^M, \alpha), \mathbf{B}_{\ast'}  ) \nonumber \\
    &\geq \dist( \mathbf{B}_\ast,  \mathbf{B}_{\ast'}) - \dist(\dist(\mathbf{B}_T^{\text{D-GD}}(\mathbf{B}_0,\mathbf{B}_{\ast},\{\mathbf{w}_{\ast,i}\}_{i=1}^M, \alpha), \mathbf{B}_\ast  ) \nonumber \\
    &\geq 2 \delta_0 \sqrt{1-\delta_0^2} - 0.7\delta_0 \label{lk} \\
    &\geq (\sqrt{3}-0.7)\delta_0 \label{klk} \\
    &\geq \delta_0 \nonumber
\end{align}
as desired, where \eqref{lk} follows by the definition of case (2) and \eqref{swrt}, and \eqref{klk} follows by $\delta_0 \in (0,1/2]$.
\end{proof}



\newpage
\section{Experimental Details} \label{app:experiments}

\subsection{Multi-task linear regression}

The multi-task linear regression experiments consist of two stages: training and fine-tuning. During training, we track $\dist(\mathbf{B}_t, \mathbf{B}_\ast)$ in Figure \ref{fig:1} and the Frobenius norm of the gradient of \ref{glob_pop}, i.e. $(\|\mathbf{B}_t^\top(\mathbf{B}_t\mathbf{w}_t - \mathbf{B}_\ast \mathbf{\bar{w}}_{\ast})\|^2_2 + \|(\mathbf{B}_t\mathbf{w}_t - \mathbf{B}_\ast \mathbf{\bar{w}}_{\ast})\mathbf{w}_t^\top\|_F^2)^{1/2}$, in Figure \ref{fig:linear}(left). For fine-tuning, we track the squared Euclidean distance of the post-fine-tuned model from the ground-truth in Figure \ref{fig:linear}(right) for various numbers of fine-tuning samples $n$.

Each training trial consists of first sampling $M$ ground truth heads $\mathbf{w}_{\ast,i} \sim \mathcal{N}(\mathbf{0},\mathbf{I}_k)$ and a ground truth representation $\mathbf{\check{B}}_\ast \in \mathbb{R}^{d \times k}$ such that each element is i.i.d. sampled from a standard Gaussian distribution. Then, $\mathbf{B}_\ast$ is formed by computing the QR factorization of $\mathbf{\check{B}}_\ast$, i.e. $\mathbf{B}_\ast\mathbf{R}_\ast = \mathbf{\check{B}}$, where $\mathbf{R}_\ast \in \mathbb{R}^{k\times k}$ is upper triangular and $\mathbf{B}_\ast\in \mathcal{O}^{d \times k}$ has orthonormal columns. To initialize the model we set $\mathbf{w}_0= \mathbf{0}\in \mathbb{R}^k$ and sample $\mathbf{\check{B}}_0 \in \mathbb{R}^{d \times k}$ such that each element is i.i.d. sampled from a standard Gaussian distribution, then compute  $\mathbf{{B}}_0 = \tfrac{1}{\sqrt{\alpha}}\mathbf{\hat{B}}_0$ where $\mathbf{\hat{B}}_0 \in \mathcal{O}^{d\times k}$ is the matrix with orthonormal columns resulting from the QR factorization of $\mathbf{\check{B}}_0$. Then we run FedAvg with $\tau=2$ and D-SGD on the population objective \ref{glob_pop}, with both sampling $m=M$ clients per round and using step size $\alpha = 0.4$. The training plots show quantities averaged over 10 independent trials.

For each fine-tuning trial, we similarly draw a new head $\mathbf{w}_{\ast,M+1}\sim \mathcal{N}(\mathbf{0},\mathbf{I}_k)$, and data $\mathbf{x}_{M+1,j}\sim \mathcal{N}(\mathbf{0},\mathbf{I}_k)$, $y_{M+1,j} = \langle \mathbf{B}_\ast \mathbf{w}_{\ast,i}, \mathbf{x}_{M+1,j}\rangle+\zeta_{M+1,j}$ where $\zeta_{M+1,j} \sim \mathcal{N}(0,0.01)$.
Then we run GD on the empirical loss $\frac{1}{2n}\sum_{j=1}^n (\langle \mathbf{Bw}, \mathbf{x}_{M+1,j}\rangle  - \mathbf{y}_{M+1,j})^2$ for $\tau'=200$ iterations with step size $\alpha = 0.01$. The fine-tuning plot shows average results over 10 independent, end-to-end trials (starting with training), and the error bars give standard deviations.

\subsection{Image classification}

The CNN used in the image classification experiments has six convolutional layers with ReLU activations and max pooling after every other layer. On top of the six convolutional layers  is a 3-layer MLP with ReLU activations.



All models are trained with a step size of $\alpha = 0.1$ after tuning in $\{ 0.5,0.1, 0.05, 0.01,0.005\}$ and selecting the best $\alpha$ that yields the smallest training loss. In all cases, $m = 0.1M$. We use the SGD optimizer with weight decay $10^{-4}$ and momentum 0.5. We train all models such that $T\tau = 125000$. Thus, D-SGD trains for $T=125000$ rounds, since $\tau=1$ in this case. Likewise, for FedAvg with $\tau=50$, $T=2500$ training rounds are executed. The batch size is 10 in all cases. We also experimented with larger batch sizes for D-SGD but they did not improve performance.

Each client has  500 training samples in all cases. Thus, FedAvg with $\tau=50$ is equivalent to FedAvg with one local epoch. For the experiments with $C$ classes per client for all clients, each client has the same number of images from each class. 
For the experiment testing fine-tuning performance on new classes from the same dataset, the first 80 classes for CIFAR100 are used for training, while classes 80-99 are reserved for new clients. For fine-tuning, 10 epochs of SGD are executed on the training data for the new client. For fine-tuning on CIFAR10, each new client has images from all 10 classes, and equal numbers of samples per class for training. For fine-tuning on CIFAR100, each new client has images from all 20 classes, with equal numbers of samples from each class for training when possible, and either 2 or 3 samples from each class otherwise (when the number of fine-tuning samples equals 50). Accuracies are top 1 accuracies evaluated on 2000 test samples per client for CIFAR10 and 400 test samples per client for the last 20 classes of CIFAR100.

To compute the layer-wise similarities in {Figure \ref{fig:2}},  we use the Centered Kernel Alignment (CKA)  similarity metric, which is  the most common metric used to measure the similarity between neural networks \cite{kornblith2019similarity}. CKA similarity between  model layers is evaluated by feeding the same input through both networks and computing the similarity between the outputs of  the layers. The similarity metric is invariant to rotations and isotropic scaling of the layer outputs \cite{kornblith2019similarity}. We use the code from \cite{subramanian2021torch_cka} to compute CKA similarity.

For {Figure \ref{fig:sim}}, the cosine similarity is evaluated by first feeding $n$ images from class $c$ through the trained network and storing the output of the network layer before the final linear layer to obtain the features $\{ \mathbf{f}_1^c,\dots,\mathbf{f}_n^c \}$, where each $\mathbf{f}_i^c\in\mathbb{R}^{512}$. Then, to obtain the average cosine similarity between features from classes $c$ and $c'$, we compute $\frac{1}{n}\sum_{i=1}^n\frac{|\langle \mathbf{f}_{i}^c , \mathbf{f}_{i}^{c'} \rangle |}{\|\mathbf{f}_{i}^{c} \|\|\mathbf{f}_{i}^{c'} \|}$. We use $n=25$.



All experiments were performed on two 8GB NVIDIA GeForce RTX 2070 GPUs.

\end{document}